\documentclass[twoside,11pt]{article}

\usepackage{jmlr2e}
\usepackage[margin=.95in]{geometry}

\usepackage[hidelinks]{hyperref}

\usepackage{graphicx}
\usepackage{url}
\usepackage{amsmath}
\usepackage{amsfonts}
\usepackage{caption}
\usepackage{amsfonts}
\usepackage{subcaption}
\usepackage{stackrel}
\usepackage{bibunits}
\usepackage{sidecap}
\usepackage{color}\definecolor{lightred}{rgb}{1,.5,.3}
\usepackage{xspace}
\usepackage{comment}
\usepackage{algorithm}
\usepackage{array}
\usepackage{enumitem}
\usepackage{natbib}

\newcommand*{\eqnameformat}[1]{%
  \textsf{#1}%
}

\makeatletter
\@ifdefinable{\org@maketag@@@}{%
  \let\org@maketag@@@\maketag@@@
  \renewcommand*{\maketag@@@}[1]{%
    \org@maketag@@@{%
      \@ifundefined{eq@name}{#1}{%
        \begin{tabular}[t]{@{}r@{}}%
          #1\tabularnewline
          \eqnameformat{\@nameuse{eq@name}}%
        \end{tabular}%
      }%
    }%
  }%
}
\newif\ifeqname@star
\newcommand*{\eqname}{%
  \@ifstar{\eqname@startrue\eqname@}{\eqname@starfalse\eqname@}%
}
\newcommand*{\eqname@}[2][]{%
  \gdef\eq@name{#1}%
  \ifx\eq@name\@empty
  \else
    \begingroup
      \@ifundefined{GetTitleString}{%
        \gdef\@currenteqlabelname{#2}%
      }{%
        \GetTitleString{#2}%
        \global\let\@currenteqlabelname\GetTitleStringResult
      }%
      \let\@currentlabelname\@currenteqlabelname
      \label{#1}%
    \endgroup
  \fi
  \gdef\eq@name{#2}%
  \ifx\eq@name\@empty
    \global\let\eq@name\relax
  \else
    \ifeqname@star
      \gdef\eq@name{\llap{#2}}%
    \fi
  \fi
}
\@ifdefinable{\org@make@display@tag}{%
  \let\org@make@display@tag\make@display@tag
  \def\make@display@tag{%
    \@ifundefined{@currenteqlabelname}{}{%
      \let\@currentlabelname\@currenteqlabelname
    }%
    \org@make@display@tag
  }%
}
\let\eq@name\relax
\let\@currenteqlabelname\relax
\g@addto@macro\displ@y@{%
  \global\let\eq@name\relax
  \global\let\@currenteqlabelname\relax
}
\@ifdefinable{\org@math@cr@@}{%
  \let\org@math@cr@@\math@cr@@
  \def\math@cr@@[#1]{%
    \org@math@cr@@[{#1}]%
    \noalign{%
      \global\let\eq@name\relax
    }%
  }%
}
\@ifdefinable{\org@eqref}{%
  \let\org@eqref\eqref
  \renewcommand*{\eqref}[1]{%
    \begingroup
      \let\eq@name\relax
      \org@eqref{#1}%
    \endgroup
  }%
}
\g@addto@macro\equation{%
  \eqname{}%
}
\makeatother

\newcommand{\beq}{\begin{equation}}
\newcommand{\eeq}{\end{equation}}
\newcommand{\bea}{\begin{eqnarray}}
\newcommand{\eea}{\end{eqnarray}}
\newcommand{\bean}{\begin{eqnarray*}}
\newcommand{\eean}{\end{eqnarray*}}
\newcommand{\bit}{\begin{itemize}}
\newcommand{\eit}{\end{itemize}}
\newcommand{\ben}{\begin{enumerate}}
\newcommand{\een}{\end{enumerate}}
\newcommand{\blem}{\begin{lemma}}
\newcommand{\elem}{\end{lemma}}
\newcommand{\bthm}{\begin{thm}}
\newcommand{\ethm}{\end{thm}}

\newcommand{\sign}{\textrm{sign}}

\newcommand{\Ind}{\mathbf{1}}

\newcommand{\numobs}{n}
\newcommand{\numitems}{d}
\newcommand{\wt}{w}
\newcommand{\lemmawt}{z}
\newcommand{\unitvec}{e}

\newcommand{\noisestd}{\sigma}

\newcommand{\noisestdc}{\noisestd_c}
\newcommand{\eps}{\epsilon}
\newcommand{\reals}{\mathbb{R}}
\newcommand{\obs}{y}
\newcommand{\packdmin}{\alpha}
\newcommand{\diff}{x} 
\newcommand{\diffdirect}{u}
\newcommand{\diffmx}{X} 
\newcommand{\Lnormsqr}[3]{(#1-#2)^T #3 (#1-#2)}
\newcommand{\Lnormd}[3]{\|#1-#2\|_{#3}}
\newcommand{\Lnorm}[2]{\|#1\|_{#2}}
\newcommand{\loss}{\ell}

\newcommand{\tstone}{{\sc Thurstone}\xspace}
\newcommand{\dir}{{\sc Cardinal}\xspace}
\newcommand{\pair}{{\sc Paired Cardinal}\xspace}
\newcommand{\btl}{{\sc BTL}\xspace}
\newcommand{\ord}{{\sc Ordinal}\xspace}
\newcommand{\mwise}{{\sc $\numchoices$-wise}\xspace}

\newcommand{\norm}[1]{\lVert#1\rVert_2}
\newcommand{\inprod}[2]{\ensuremath{\langle #1 , \, #2 \rangle}}
\newcommand{\trace}[1]{\mathrm{tr}(#1)}
\newcommand{\kl}[2]{\ensuremath{D_{\mathrm{KL}}(#1\|#2)}}

\newcommand{\glmcdf}{F}

\newcommand{\wmax}{B}
\newcommand{\cdfparam}{\zeta}

\newcommand{\qed}{\hfill$\blacksquare$}

\newcommand{\numchoices}{\ensuremath{m}} 
\newcommand{\hessgencdf}{\ensuremath{H}} 
\newcommand{\genlapmx}{\ensuremath{L}} 
\newcommand{\genlapmxinv}{\ensuremath{L^\dagger}}
\newcommand{\rankminusonemx}{M}

\newlength{\widebarargwidth}
\newlength{\widebarargheight}
\newlength{\widebarargdepth}

\makeatletter
\long\def\@makecaption#1#2{
        \vskip 0.8ex
        \setbox\@tempboxa\hbox{\small {\bf #1:} #2}
        \parindent 1.5em  
        \dimen0=\hsize
        \advance\dimen0 by -3em
        \ifdim \wd\@tempboxa >\dimen0
                \hbox to \hsize{
                        \parindent 0em
                        \hfil 
                        \parbox{\dimen0}{\def\baselinestretch{0.96}\small
                                {\bf #1.} #2
                                } 
                        \hfil}
        \else \hbox to \hsize{\hfil \box\@tempboxa \hfil}
        \fi
        }
\makeatother

\newcommand{\widgraph}[2]{\includegraphics[keepaspectratio,width=#1]{#2}}

\newcommand{\order}{\ensuremath{\mathcal{O}}}

\newcommand{\mprob}{\ensuremath{\mathbb{P}}}

\newcommand{\real}{\ensuremath{\mathbb{R}}}

\newcommand{\defn}{\ensuremath{: \, = }}

\newcommand{\Exs}{\ensuremath{\mathbb{E}}}

\newcommand{\wthat}{\ensuremath{\widehat{\scriptsize \wt}}}
\newcommand{\wthatML}{\widehat{\wt}_{\mbox{\scalebox{.6}{ML}}}}
\newcommand{\wtstar}{\ensuremath{\wt^*}}

\newcommand{\ones}{\ensuremath{1}}
\newcommand{\zeros}{\ensuremath{0}}

\newcommand{\kull}[2]{\ensuremath{D_{\mathrm{KL}}(#1\|#2)}}

\newcommand{\mjwcomment}[1]{}
\newcommand{\mjwcommenta}[1]{}

\newcommand{\sbcomment}[1]{}
\newcommand{\sbcommenta}[1]{}

\newcommand{\nscomment}[1]{}
\newcommand{\nscommenta}[1]{}

\newcommand{\packvec}[1]{\ensuremath{{\wt^{#1}}}}

\newcommand{\LamTil}{\ensuremath{\Lambda^{\dagger}}}

\newcommand{\Pclass}{\ensuremath{\mathcal{P}}}
\newcommand{\genpar}{\ensuremath{\wt}}
\newcommand{\packnum}{\ensuremath{M}}
\newcommand{\SEMINORM}{\ensuremath{\rho}}
\newcommand{\MiniMax}{\ensuremath{\mathfrak{M}}}
\newcommand{\LAPNORM}[1]{\ensuremath{\|#1\|_{\Lap}}}
\newcommand{\LAPNORMINV}[1]{\ensuremath{\|#1\|_{\LapInv}}}
\newcommand{\LambdaInv}{\ensuremath{\Lambda^{\dagger}}}

\newcommand{\Wspace}{\ensuremath{\mathcal{W}}}

\newcommand{\Wclass}{\Wspace}

\newcommand{\Xmat}{\ensuremath{X}}

\newcommand{\regnoise}{\ensuremath{\varepsilon}}

\newcommand{\strongcon}{\ensuremath{\gamma}}

\newcommand{\Lap}{\ensuremath{L}}
\newcommand{\LapInv}{\ensuremath{\Lap^{\dagger}}}

\newcommand{\vtil}{\ensuremath{\widetilde{v}}}
\newcommand{\util}{\ensuremath{\widetilde{u}}}

\newcommand{\nullspace}{\ensuremath{\operatorname{nullspace}}}

\newcommand{\matsnorm}[2]{|\!|\!| #1 | \! | \!|_{{#2}}}

\newcommand{\fronorm}[1]{\ensuremath{\matsnorm{#1}{\footnotesize{\tiny{\mbox{fro}}}}}}
\newcommand{\opnorm}[1]{\ensuremath{\matsnorm{#1}{\tiny{\mbox{op}}}}}

\newcommand{\plaincon}{\ensuremath{c}}
\newcommand{\cvoconst}{\ensuremath{b}}

\newcommand{\wttil}{\ensuremath{\widetilde{\wt}}}

\long\def\comment#1{}

\newcommand{\gendiff}{E}

\newcommand{\revmx}{R} 
\newcommand{\eigenvalue}[2]{\lambda_{#1}(#2)}
\newcommand{\spanvectors}[1]{{\rm span}(#1)}
\newcommand{\identity}{I}
\newcommand{\diag}[1]{{\rm diag}(#1)}

\newcommand{\graph}{G}
\newcommand{\vertices}{V}
\newcommand{\edges}{E}

\newcommand{\SPECSETgeneral}{\ensuremath{\mathcal{W}}}
\newcommand{\SPECSET}{\ensuremath{{\SPECSETgeneral}_{\wmax}}}
\newcommand{\SPECSETinf}{\ensuremath{{\SPECSETgeneral}_{\infty}}}

\newcommand{\myglmpdf}{\ensuremath{\glmcdf'}}

\newcommand{\numitem}{\numitems}

\newcommand{\wtMLE}{\wthatML}

\newcommand{\lemstrongcon}{\ensuremath{\kappa}}

\newcommand{\Vvar}{\ensuremath{V}}
\newcommand{\Vtil}{\ensuremath{\widetilde{\Vvar}}}

\begin{document}
\title{Estimation from Pairwise Comparisons: \\ Sharp Minimax Bounds
with Topology Dependence}

\author{\name Nihar B. Shah \email nihar@eecs.berkeley.edu
\AND
\name Sivaraman Balakrishnan \email sbalakri@berkeley.edu
\AND
\name Joseph Bradley \email joseph.kurata.bradley@gmail.com
\AND 
\name Abhay Parekh \email parekh@berkeley.edu
\AND
\name Kannan Ramchandran \email kannanr@eecs.berkeley.edu
\AND 
\name Martin J. Wainwright \email wainwrig@berkeley.edu\\
\addr UC Berkeley
}

\editor{}

\maketitle


\begin{abstract}
Data in the form of pairwise comparisons arises in many domains,
including preference elicitation, sporting competitions, and peer
grading among others. We consider parametric 
ordinal models for such pairwise
comparison data involving a latent vector $w^* \in \mathbb{R}^d$ that
represents the ``qualities'' of the $d$ items being compared; this
class of models includes the two most widely used parametric
models--the Bradley-Terry-Luce (BTL) and the Thurstone models. Working
within a standard minimax framework, we provide tight upper and lower
bounds on the optimal error in estimating the quality score vector
$w^*$ under this class of models.  The bounds depend on the topology
of the comparison graph induced by the subset of pairs being compared
via its Laplacian spectrum.  Thus, in settings where the subset of
pairs may be chosen, our results provide principled guidelines for
making this choice.  Finally, we compare these error rates to those
under cardinal measurement models and show that the error rates in the
ordinal and cardinal settings have identical scalings apart from
constant pre-factors.
\end{abstract}


\section{Introduction}
In an increasing range of applications, it is of interest to elicit
judgments from non-expert humans.  For instance, in marketing,
elicitation of preferences of consumers about products, either directly
or indirectly, is a common practice~\citep{green1981estimating}.  The
gathering of this and related data types has been greatly facilitated
by the emergence of ``crowdsourcing'' platforms such as Amazon
Mechanical Turk: they have become powerful, low-cost tools for
collecting human
judgments~\citep{khatib2011crystal,lang2011using,von2008recaptcha}.
Crowdsourcing is employed not only for collection of consumer
preferences, but also for other types of data, including counting the
number of malaria parasites in an image of a blood
smear~\citep{luengo2012crowdsourcing}; rating responses of an online
search engine to search queries~\citep{kazai2011search}; or for
labeling data for training machine learning
algorithms~\citep{hinton2012deep,raykar2010learning,deng2009imagenet}.
In a different domain, competitive sports can be understood as a
mechanism for sequentially performing comparisons between individuals
or teams~\citep{chessbase2007elo,herbrich2007trueskill}.  Finally, 
peer-grading in massive open online courses
(MOOCs)~\citep{piech2013tuned} can be viewed as another form of
elicitation.

A common method of elicitation is through pairwise comparisons. For
instance, the decision of a consumer to choose one product over
another constitutes a pairwise comparison between the two
products. Workers in a crowdsourcing setup are often asked to compare
pairs of items: for instance, they might be asked to identify the
better of two possible results of a search engine, as shown in
Figure~\ref{fig:searchRelevance_ordinal}. Competitive sports such as
chess or basketball also involve sequences of pairwise comparisons.
\begin{figure*}
\centering
\begin{subfigure}{.45\textwidth}
\centering
\includegraphics[width=.9\textwidth]{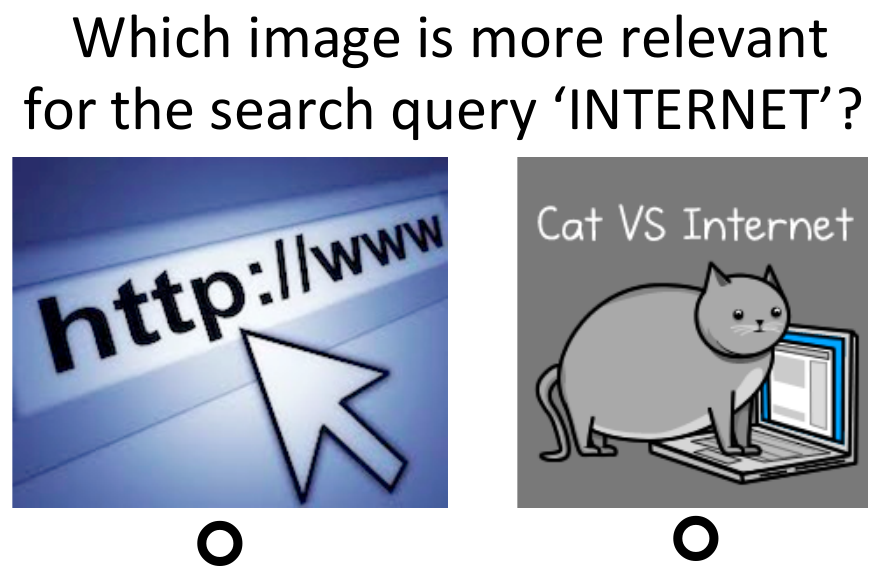}
\caption{Asking for a pairwise comparison.}
\label{fig:searchRelevance_ordinal}
\end{subfigure}
\begin{subfigure}{.45\textwidth}
\centering
\includegraphics[width=.9\textwidth]{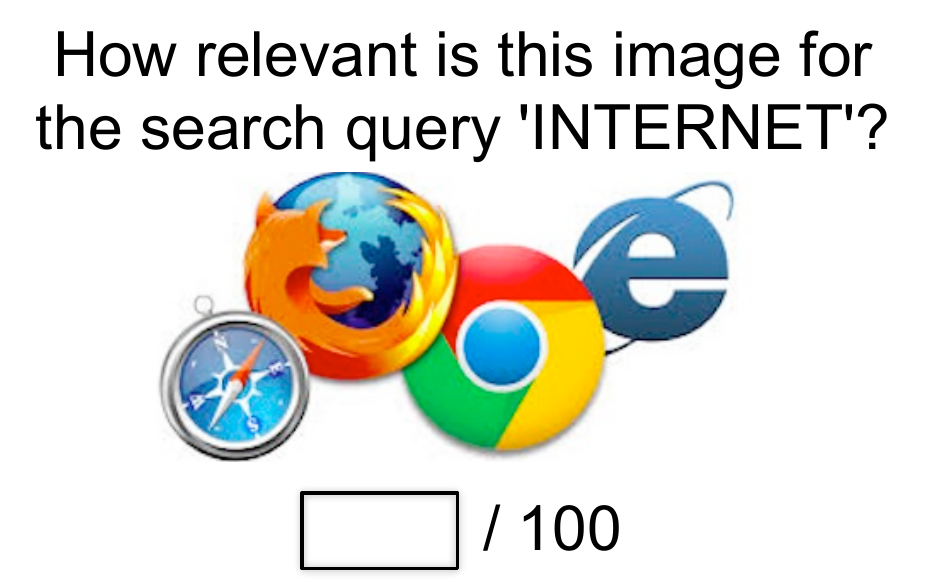}
\caption{Asking for a numeric score.}
\label{fig:searchRelevance_cardinal}
\end{subfigure}
\caption{An example of eliciting judgments from people: rating the relevance of the result of a search query.}
\label{fig:searchRelevance}
\end{figure*}
From a modeling point of view, we can think of pairwise comparisons as
a means of estimating the underlying ``qualities'' or ``weights'' of
the items being compared (e.g., skill levels of chess players,
relevance of search engine results, etc.).  Each pairwise comparison
can be viewed as a noisy sample of some function of the underlying
pair of (real-valued) weights.  Noise can arise from a variety of
sources.  When objective questions are posed to human subjects, noise
can arise from their differing levels of expertise.  In a sports
competition, many sources of randomness can influence the outcome of
any particular match between a pair of competitors.  Thus, one
important goal is to estimate the latent qualities based on noisy data
in the form of pairwise comparisons. A related problem is that of
experimental design: assuming that we can choose the subset of pairs
to be compared (e.g., in designing a chess tournament), what choice
will allow for the most accurate estimation?  Characterizing the
fundamental difficulty of estimating the weights will allow us to make
this choice judiciously. These tasks are the primary focus of this
paper.

In more detail, the focus of this paper is the aggregation from
pairwise comparisons in a fairly broad class of parametric models.
This class includes as special cases the two most popular models for
pairwise comparisons---namely, the Thurstone (Case
V)~\citep{thurstone1927law} and the Bradley-Terry-Luce
(BTL)~\citep{bradley1952rank,luce1959individual} models. The Thurstone
(Case V) model has been used in a variety of both
applied~\citep{swets1973relative,chessbase2007elo,herbrich2007trueskill}
and theoretical
papers~\citep{bramley2005rank,krabbe2008thurstone,nosofsky1985luce}.
Similarly, the BTL model has been popular in both theory and practice
(e.g.,~\citep{nosofsky1985luce,atkinson1998asian,koehler1982application,
  heldsinger2010using,loewen2012testing,green1981estimating,khairullah1987approach}).

\subsection{Some past work}
There is a vast literature on the Thurstone and BTL models, and we focus
on those most closely related to our own work. 
\citet{negahban2012iterative} provide minimax bounds for the
BTL model in the special case of comparisons 
chosen uniformly at random.  They focus on
this case in order to complement their analysis of an
algorithm based on a random walk. In their analysis, there is a gap between the achievable
rate of the MLE and the lower bound.  In contrast, our analysis
eliminates this discrepancy and shows that MLE is an optimal estimator
(up to constant factors) and achieves the minimax rate.  In
independent and concurrent work, ~\citet{hajek2014minimax}
consider the problem of estimation in the Plackett-Luce model, which
extends the BTL model to comparisons of two or more items.  They
derive bounds on the minimax error rates under this model which are
tight up to logarithmic factors. In contrast, our results are tight up
to constants and, as we emphasize in the following section, provide
deeper insights into the role of the topology of the comparison
graph. \citet{jagabathula2008inferring} design an algorithm for
aggregating ordinal data when the underlying distribution over the
permutations is assumed to be sparse. \citet{ammar2011ranking} employ
a different, maximum entropy approach towards parameterization and
inference from partially ranked data. \citet{rajkumar2014statistical} study the statistical convergence properties of several rank aggregation algorithms.

Our work assumes a fixed design setup. In this setup, the choice of which pairs
to compare and the number of times to compare them is 
chosen ahead of time in a non-adaptive fashion. There is a parallel line
of literature on ``sorting'' or ``active ranking'' from pairwise
comparisons.  For instance,~\citet{braverman2008noisy} assume a
noise model where the outcome of a pairwise comparison depends only on
the relative ranks of the items being compared, and not on their
actual ranks or values.  On the other hand, \citet{jamieson2011active}
consider the problem of ranking a set of items assuming that items can
be embedded into a smaller-dimensional Euclidean space, and that the
outcomes of the pairwise comparisons are based on the relative
distances of these items from a fixed reference point in the Euclidean
space.

A recent line of work considers a variant of the BTL and the Thurstone
models where the comparisons may depend on some auxiliary unknown
variable in addition to the items being compared; for instance, the
accuracy of the individual making the comparison in an objective
task. \citet{chen2013pairwise} consider a crowdsourcing setup where
the outcome depends on the worker's expertise. They present algorithms
for inference under such a model and present empirical
evaluations. \citet{yi2013inferring} consider a problem in the spirit
of collaborative filtering where certain unknown preferences of a
certain user must be predicted based on the preferences of other users
as well as of that user over other items. \citet{lee2011model}
consider the inverse problem of measuring the expertise of individuals
based on the rankings submitted by them, and the proposed algorithms
assume an underlying Thurstone model.


\subsection{Our contributions}

Both the Thurstone (Case V) and BTL models involve an unknown vector
$\wtstar \in \real^\numitem$ corresponding to the underlying qualities
of $\numitem$ items, and in a pairwise comparison between items $j$
and $k$, the probability of $j$ being ranked above $k$ is some
function $F$ of the difference $\wt^*_j - \wt^*_k$.  The Thurstone
(Case V) and BTL are based on different choices of $F$, and both
belong to the broader class of models analyzed in this paper, in which
$F$ is required only to be strongly log-concave.

With this context, the main contributions of this paper are to provide
some answers to the following questions:
\begin{itemize}[leftmargin=*]
\item How does the minimax error for estimating the weight vector
  $\wtstar$ in various norms scale with the problem dimension (the
  number of items) and the number of observations?
\begin{itemize}[leftmargin=*]
\item We derive upper and lower bounds on the minimax estimation rates
  under the model described above.  Our upper/lower bounds on the
  estimation error agree up to constant factors: to the
  best of our knowledge, despite the voluminous literature on these
  two models, this provides the first sharp characterization of the
  associated minimax rates.  Moreover, our error guarantees provide
  guidance to the practitioner in assessing the number of
  pairwise comparisons to be made in order to guarantee a
  pre-specified accuracy. 
\end{itemize}
\item Given a budget of $\numobs$ comparisons, which pairs of items
  should be compared?
\begin{itemize}[leftmargin=*]
\item The bounds that we derive depend on the comparison graph induced
  by the subset of pairs that are compared.  Our theoretical analysis
  reveals that the spectral gap of a certain scaled version of the
  graph Laplacian plays a fundamental role, and provides guidelines
  for the practitioner on how to choose the subset of comparisons to
  be made.
\end{itemize}
\item When is it better to elicit pairwise comparisons versus numeric
  scores?
\begin{itemize}[leftmargin=*]
\item When eliciting data, one often has the liberty to ask for either
  cardinal values (Figure~\ref{fig:searchRelevance}b) or for pairwise
  comparisons (Figure~\ref{fig:searchRelevance}a) from the human
  subjects. One would like to adopt the approach that would lead to a
  better estimate. One may be tempted to think that cardinal
  elicitation methods are superior, since each cardinal measurement
  gives a real-valued number whereas an ordinal measurement provides
  at most one bit of information.  Our bounds show, however, that the
  scaling of the error in the cardinal and ordinal settings is
  identical up to constant pre-factors. As we demonstrate, this result
  allows for a comparison of cardinal and ordinal data elicitation
  methods in terms of the per-measurement noise alone,
  \emph{independent} of the number of measurements and the number of
  items.  A priori, there is no obvious reason for the relative
  performance to be independent of the number of measurements and
  items.
\end{itemize}
\end{itemize}

\paragraph{Notation:}
For any symmetric matrix $M$ of size $(m \times m)$, we will let
$\eigenvalue{1}{M} \leq \eigenvalue{2}{M} \leq \cdots \leq
\eigenvalue{m}{M}$ denote its ordered eigenvalues. We will use the
notation $\kl{\mprob_1}{\mprob_2}$ to denote the Kullback-Leibler
divergence between the two distributions $\mprob_1$ and $\mprob_2$. For any integer $m$, we will let $[m]$ denote the set $\{1,\ldots,m\}$.


\section{Problem formulation}
\label{SecProblem}

We begin with some background followed by a precise formulation of the
problem.

\subsection{Generative models for ranking}
\label{sec::models}

Given a collection of $\numitems$ items to be evaluated, we suppose that
each item has a certain numeric \emph{quality score}, and a comparison
of any pair of items is generated via a comparison of the two quality
scores in the presence of noise. We represent the quality scores as a
vector $\wtstar \in \real^\numitems$, so item $j \in [\numitems]$ has
quality score $\wtstar_j$.  Now suppose that we make $\numobs$
pairwise comparisons: if comparison $i \in [\numobs]$ pertains to
comparing item $a_i$ with item $b_i$, then it can be described by a
differencing vector $\diff_i \in \real^\numitems$, with entry $a_i$
equal to one, entry $b_i$ equal to $-1$, and the remaining entries set
\mbox{to $0$.}

With this notation, we study the problem of estimating the weight
vector $\wtstar$ based on observing a collection of $\numobs$ independent
samples \mbox{$\obs_i \in \{-1,1\}$} drawn from the distribution
\begin{align}
\label{EqnPairwise}
\mprob \big[ \obs_i = 1 \vert \diff_i, \wtstar \big] & = \glmcdf \Big(
\frac{\inprod{\diff_i}{\wtstar} }{\noisestd} \Big) \quad \mbox{for $i
  \in [\numobs]$,} \tag{\ord}
\end{align}
where $\glmcdf$ is a known function taking values in $[0,1]$.  Since
the probability of item $a_i$ dominating $b_i$ should be independent
of the order of the two items being compared, we require throughout
that $\glmcdf(x) = 1 - \glmcdf(-x)$.  

In any model of the general form~\eqref{EqnPairwise}, the parameter
$\noisestd > 0$, assumed to be known, plays the role of a noise
parameter, with a higher value of $\noisestd$ leading to more
uncertainty in the comparisons. Moreover, we assume that $\glmcdf$ is
\emph{strongly log-concave} in a neighborhood of the origin, meaning
that there is some curvature parameter $\strongcon > 0$ such that
\begin{align} \label{EqnDefnStrongcon}
\frac{d^2}{dt^2} (-\log \glmcdf(t)) \geq \strongcon \quad \mbox{for
  all $t \in [-2\wmax/\noisestd, 2\wmax/\noisestd]$}.
\end{align}
Here the known parameter $\wmax$ denotes a bound on the $\ell_\infty$-norm
of the weight vector, namely
\begin{align*}
\Lnorm{\wtstar}{\infty} \leq \wmax.
\end{align*}
As our analysis shows, a bound of this form is fundamental: the
minimax error for estimating $\wtstar$ will diverge to infinity if we
are allowed to consider models in which $\wmax$ is arbitrarily large
(see Proposition~\ref{prop:unbounded_weights} in
Appendix~\ref{AppUnboundedWeights}).  Informally, this
behavior is related to the difficulty of estimating very small (or very
large) probabilities that can arise in the two models for large
$\|\wtstar\|_\infty$.  Note that any model of the
form~\eqref{EqnPairwise} is invariant to shifts in $\wtstar$, that is,
it does not differentiate between the vector $\wtstar$ and the shifted
vector $\wtstar + \ones$, where $\ones$ denotes the vector of all
ones. Therefore, in order to ensure identifiability of $\wtstar$, we
assume throughout that $\inprod{\ones}{\wtstar} = 0$. We will use the
notation $\SPECSET$ to denote the set of permissible quality score
vectors
\begin{align}
\SPECSET & \defn \big \{ \wt \in \real^\numitem \, \mid \,
\|\wt\|_\infty \leq \wmax, \quad \mbox{and} \quad \inprod{1}{\wt} = 0
\big \}.
\label{EqnDefnSpecset}
\end{align}

Both the Thurstone (Case V) model with Gaussian
noise~\citep{thurstone1927law} and Bradley-Terry-Luce (\btl)
models~\citep{bradley1952rank,luce1959individual} are special cases of
this general set-up, as we now describe.

\paragraph{Thurstone (Case V):}
This model is is a special case of the family~\eqref{EqnPairwise},
obtained by setting
\begin{align}
\glmcdf(t) & = \int_{-\infty}^t \frac{1}{\sqrt{2 \pi}} e^{-u^2/2} du,
\end{align}
corresponding to the CDF of the standard normal distribution.
Consequently, the Thurstone model can alternatively be written as
making $\numobs$ i.i.d. observations of the form
\begin{align}
\label{EqnThurstone}
\obs_i = \sign \biggr \{ \inprod{\diff_i}{\wtstar} + \eps_i \biggr \},
\quad \mbox{for $i \in [\numobs]$}, \tag{\tstone}
\end{align}
where $\eps_i \sim N(0, \noisestd^2)$ is observation noise.  It can be
verified that the \tstone model is strongly log-concave (e.g., see
\citep{tsukida2011analyze}).

\paragraph{Bradley-Terry-Luce:}
The Bradley-Terry-Luce (\btl)
model~\citep{bradley1952rank,luce1959individual} is another special
case in which
\begin{align*}
\glmcdf(t) = \frac{1}{1+ e^{-t}},
\end{align*}
and hence
\begin{align}
\mprob \big[ \obs_i = 1 \vert \diff_i, \wtstar \big] & = \frac{1}{1 +
  \exp\big(-\frac{\inprod{\diff_i}{\wtstar} }{\noisestd} \big)} \quad
\mbox{for $i \in [\numobs]$.} \tag{\btl}
\end{align}
It can also be verified that the \btl model is strongly log-concave.


\paragraph{Cardinal observation models:}  

While our primary focus is on the pairwise-comparison
setting, for comparison purposes we also analyze analogous cardinal
settings where each observation is real valued. In particular, we
consider the following two cardinal analogues of the~\tstone model. In
the \dir model we consider, each observation $i \in [\numobs]$
consists of a numeric evaluation $\obs_i \in \reals$ of a single item,
\begin{align}
\label{eqn::direct}
\obs_i & = \inprod{\diffdirect_i}{\wtstar} + \eps_i \qquad \mbox{for $i
  \in [\numobs]$,} \tag{\dir}
\end{align}
where $\diffdirect_i$ in this case is a coordinate vector with one of
its entries equal to $1$ and remaining entries equal to $0$, and
$\eps_i$ is independent Gaussian noise $N(0,\noisestd^2)$.  One may
alternatively elicit cardinal values of the differences between pairs
of items
\begin{align}
\label{EqnPairedLinear}
\obs_i & = \inprod{\diff_i}{\wtstar} + \eps_i \qquad \mbox{for $i \in
  [\numobs]$,} \tag{\pair}
\end{align}
where $\eps_i$ are i.i.d. $N(0,\noisestd^2)$. We term this model the
\pair model.

\subsection{Fixed design and the graph Laplacian}

We analyze the estimation error when a fixed subset of pairs is chosen
for comparison.  Of interest to us will be the \emph{comparison graph}
defined by these chosen pairs, with each pair inducing an edge in the
graph. Edge weights are determined by the fraction of times a given
pair is compared.  The analysis in the sequel reveals the central role
played by the Laplacian of this weighted graph.  Note that we are
operating in a fixed-design setup where the graph is constructed
offline and does not depend on the observations.

In the ordinal models, the $i^{\mathrm{th}}$ measurement is related to
the difference between the two items being compared, as defined by the
measurement vector $\diff_i \in \real^\numitems$.  We let $\Xmat \in
\real^{\numobs \times \numitems}$ denote the measurement matrix with
the vector $\diff_i^T$ as its $i^{\mathrm{th}}$ row.  The Laplacian
matrix $\Lap$ associated with this differencing matrix is given by
\begin{align}
\label{EqnDefnLap}
\Lap & \defn \frac{1}{\numobs} \Xmat^T \Xmat \; = \; \frac{1}{\numobs}
\sum_{i=1}^\numobs \diff_i \diff_i^T.
\end{align}
By construction, for any vector $v \in \real^\numitems$, we have $v^T
\Lap v = \sum_{j \neq k} \Lap_{jk} (v_j - v_k)^2$, where $\Lap_{jk}$
is the fraction of the measurement vectors $\{\diff_i\}_{i=1}^\numobs$
in which items $(j,k)$ are compared.

The Laplacian matrix is positive semidefinite, and has at least one
zero-eigenvalue, corresponding to the all-ones eigenvector.  The
Laplacian matrix induces a graph on the vertex set \mbox{$\{1, \ldots,
  \numitems \}$}, in which a given pair $(j,k)$ is included as an edge
if and only if $\Lap_{jk} \neq 0$, and the weight on an edge $(j,k)$
equals $\Lap_{jk}$.  We emphasize that 
throughout our analysis, we assume that the comparison
graph is \emph{connected}, since otherwise, the quality score vector
$\wtstar$ is not identifiable.  Note that the Laplacian matrix $\Lap$
induces a semi-norm\footnote{A semi-norm differs from a norm in that the
  semi-norm of a non-zero element is allowed to be zero.} on
$\real^\numitems$, given by
\begin{align}
\label{EqnDefnLapNorm}
\LAPNORM{u - v} & \defn \sqrt{(u-v)^T \Lap (u-v)}.
\end{align}
We study optimal rates of estimation in this semi-norm,
as well as the usual $\ell_2$-norm. As will be clearer in the sequel
the $\Lap$ semi-norm is a natural metric in our setup, and estimation
in this induced metric can be done at a topology independent rate.
The estimation error in the $\Lap$ semi-norm 
is closely related to the prediction risk in generalized linear
models. It arises naturally when one is interested
in predicting the probability of a certain outcome for a
new comparison.


\section{Bounds on the minimax risk}
\label{SecTheoryComparison}

In this section, we state the main results of the paper, and discuss
some of their consequences.


\subsection{Minimax rates in the squared $\Lap$ semi-norm}

Our first main result provides bounds on the minimax risk under the
squared $\Lap$ semi-norm~\eqref{EqnDefnLapNorm} in the pairwise
comparison models introduced earlier.  In all of the statements, we
use $\plaincon_1, \plaincon_2,$ etc. to denote positive numerical
constants, independent of the sample size $\numobs$, number of items
$\numitems$ and other problem-dependent parameters.

Apart from the parameter $\strongcon$, the bounds presented
subsequently will depend on $\glmcdf$ through a second parameter
$\cdfparam$, defined as
\begin{align}\label{EqnDefnCdfparam}
\cdfparam \defn \frac{\max \limits_{x \in [0,2\wmax/\noisestd]}
  \glmcdf'(x)}{\glmcdf(2\wmax/\noisestd)(1-\glmcdf(2\wmax/\noisestd))}.
\end{align}
 In the \btl and the \tstone models, we have $\cdfparam \defn \frac{
   \glmcdf'(0)}{\glmcdf(2\wmax/\noisestd)(1-\glmcdf(2\wmax/\noisestd))}$.

\begin{theorem}[Bounds on minimax rates in $\Lap$ semi-norm]
\label{ThmMinimaxL}
\begin{subequations}
\begin{enumerate}
\item[(a)] For a sample size \mbox{$\numobs \geq \frac{\plaincon_1
    \noisestd^2 \trace{\LapInv}}{\cdfparam \wmax^2}$}, any estimator
  $\wttil$ based on $\numobs$ samples from the~\ord model has
  Laplacian squared error lower bounded as
\begin{align}
\label{EqnLLowerBound}
d\sup_{\wtstar \in \SPECSET} \Exs \Big[ \LAPNORM{\wttil - \wtstar}^2
  \Big] \geq \frac{\plaincon_{1\ell}}{\cdfparam} \;\noisestd^2 \frac{
  \numitems }{\numobs}.
\end{align}
\item[(b)] For any instance of the~\ord model with $\strongcon$-strong
  log-concavity and any $\wtstar \in \SPECSET$, the maximum likelihood
  estimator satisfies the bound
\begin{align*}
\mprob \Big[ \Lnorm{\wthatML-\wtstar}{\Lap}^2 > t \frac{\plaincon
    \cdfparam^2 \noisestd^2}{\strongcon^2 } \frac{ \numitems
  }{\numobs} \Big] \leq e^{-t} \quad\mbox{for all $t \geq 1$},
\end{align*} 
and consequently
\begin{align}
\label{EqnLUpperBound}
\sup_{\wtstar \in \SPECSET} \Exs \Big[ \LAPNORM{\wthatML - \wtstar}^2
  \Big] \leq \frac{\plaincon_{1u} \cdfparam}{\strongcon}\noisestd^2
\frac{\numitems}{\numobs}.
\end{align}
\end{enumerate}
\end{subequations}
\end{theorem}

The results of Theorem~\ref{ThmMinimaxL} characterize the minimax risk
in the squared $\Lap$ semi-norm up to constant factors.  The upper
bounds follow from an analysis of the maximum likelihood estimator,
which turns out to be a convex optimization problem.  On the other
hand, the lower bounds are based on a combination of
information-theoretic techniques and carefully constructed packings of
the parameter set $\SPECSET$.  The main technical difficulty is in
constructing a packing in the semi-norm induced by the Laplacian
$\Lap$.  See Appendix~\ref{AppThmMinimaxL} for the full proof.


\subsection{Minimax rates in the squared $\ell_2$-norm}
\label{SecRates2}

Let us now turn to optimizing the minimax risk under the squared
Euclidean norm.  Theorem~\ref{ThmMinimax2} below presents upper and
lower bounds on this quantity.

\begin{theorem}[Bounds on minimax rates in $\ell_2$-norm]
\label{ThmMinimax2}
\begin{subequations}
\begin{enumerate}
\item[(a)] For a sample size \mbox{$\numobs \geq
    \frac{\plaincon_2 \noisestd^2 \trace{\LapInv}}{\cdfparam
      \wmax^2}$}, any estimator $\wttil$ based on $\numobs$ samples from the~\ord model has squared Euclidean error lower bounded as
\begin{align}
\label{Eqn2LowerBound}
\sup_{\wtstar \in \SPECSET} \Exs \Big[ \Lnorm{\wttil - \wtstar}{2}^2
  \Big] \geq \plaincon_{2 \ell} \; \frac{\noisestd^2 }{\numobs}
\max\Big\{\numitems^2, \max_{\numitems' \in \{2,\ldots,\numitems\}}
\sum_{i = \lfloor 0.99 \numitems' \rfloor}^{\numitems'}
\frac{1}{\eigenvalue{i}{\Lap}} \Big\}.
\end{align}
\item[(b)] For any instance of the~\ord model with $\strongcon$-strong
  log-concavity and any $\wtstar \in \SPECSET$, the maximum likelihood
  estimator satisfies the bound
\begin{align}
\label{Eqn2UpperBound}
\sup_{\wtstar \in \SPECSET} \Exs \Big[ \Lnorm{\wthatML - \wtstar}{2}^2
  \Big] \leq \frac{\plaincon_{2u} \cdfparam}{\strongcon}\noisestd^2
\frac{\numitems}{\eigenvalue{2}{\Lap}\numobs}.
\end{align}
\end{enumerate}
\end{subequations}
\end{theorem}

See Appendix~\ref{AppThmMinimax2} for the proof of this theorem.  As
we describe in the next section, the upper and lower bounds on minimax
risk from Theorem~\ref{ThmMinimax2} to identify the comparison
graph(s) that lead to the best possible minimax risk over all possible
graph topologies.

\begin{figure*}
\centering
\begin{subfigure}{.4\textwidth}
\centering
\includegraphics[width=\textwidth]{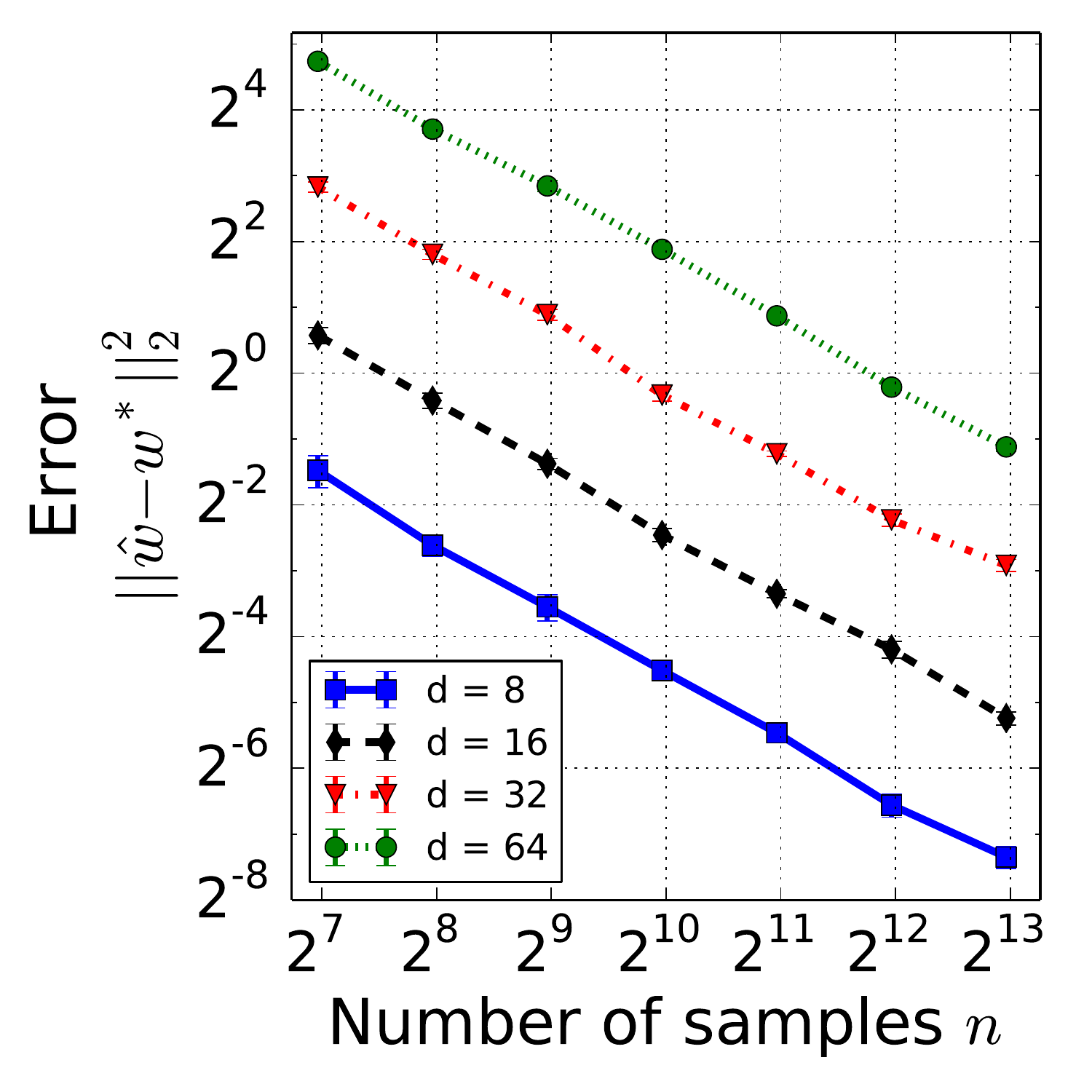}
\caption{Error}
\end{subfigure}\hspace{.16\textwidth}
\begin{subfigure}{.4\textwidth}
\centering
\includegraphics[width=\textwidth]{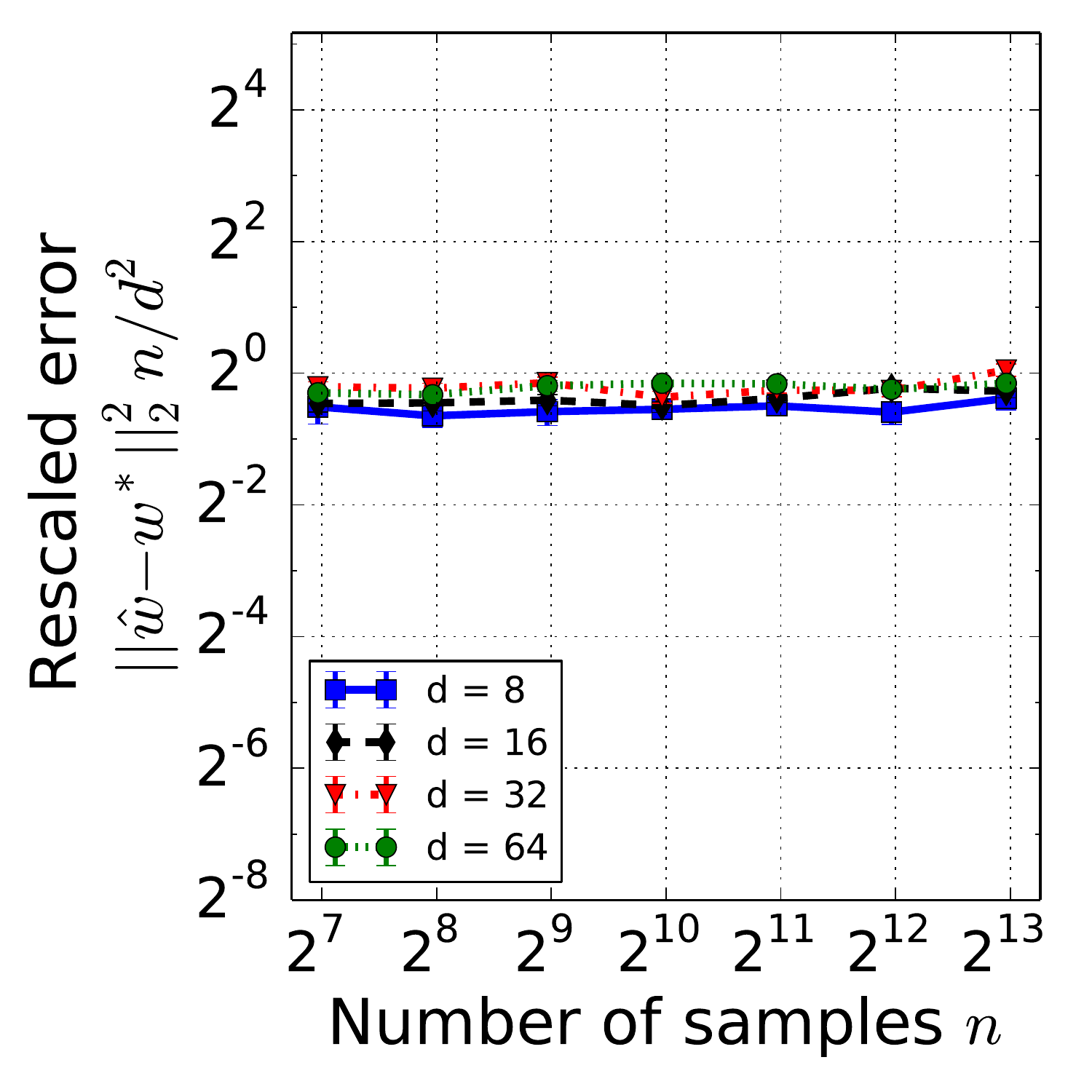}
\caption{Rescaled error}
\end{subfigure}
\caption{Simulation results under the \tstone model. The comparison
  topology chosen here is the complete graph.}
\label{fig:vary_numitems}
\end{figure*}

Figure~\ref{fig:vary_numitems} depicts results from simulations under
the \tstone model, depicting the squared $\ell_2$ error for the
maximum likelihood estimator for various values of $\numobs$ and
$\numitems$. In the simulations, the true vector $\wtstar$ is generated by first
drawing a $\numitems$-length vector uniformly at random from $[-1,1]^\numitems$, followed by a scale and shift
to ensure $\wtstar \in \SPECSET$. The $\numobs$ pairs are chosen
uniformly (with replacement) at random from the set of ${\numitems
  \choose 2}$ possible pairs of items.  The value of $\noisestd$ and $\wmax$ are both
fixed to be $1$. Given the $\numobs$ samples, inference is performed
via the maximum likelihood estimator for the \tstone model. Each point in the plots is an
average of $20$ such trials. 

The error in Figure~\ref{fig:vary_numitems}
reduces linearly with $\numobs$, exactly as predicted by our
Theorem~\ref{ThmMinimax2}. For the complete graph,
$\frac{1}{\eigenvalue{2}{\Lap}} =
\frac{\numitems-1}{2}$. Theorem~\ref{ThmMinimax2} thus predicts a
quadratic increase in the error with $\numitems$. As predicted, the
error when normalized by $\frac{1}{\numitems^2}$ in
Figure~\ref{fig:vary_numitems} converges to the same curve for all
values of $\numitems$.

Before concluding this section, we also look at the~\pair model
(Section~\ref{sec::models}), the cardinal analogue of the~\tstone
model.

\begin{theorem}[Bounds on minimax rates in $\ell_2$-norm]
\label{ThmMinimaxPair}
For the~\pair model, the minimax risk is sandwiched as
\begin{align}
\label{Eqn2LowerBoundPair}
 \plaincon_{3\ell} \;\noisestd^2 \frac{\trace{\LapInv} }{\numobs} \leq
 \inf_{\wthat} \sup_{\wtstar \in \SPECSETinf} \Exs \Big[ \Lnorm{\wthat - \wtstar}{2}^2
  \Big] \leq \plaincon_{3u} \;\noisestd^2
 \frac{\trace{\LapInv} }{\numobs}.
\end{align}
\end{theorem}
The proof of Theorem~\ref{ThmMinimaxPair} is available in
Appendix~\ref{AppThmMinimaxPairedLinear}.

We conjecture that the dependence of the squared $\ell_2$ minimax risk
under the \ord models on the problem parameters $\numobs$, $\numitems$
and the graph topology is identical to that derived in
Theorem~\ref{ThmMinimaxPair} for the \pair model, i.e., is
proportional to $\frac{\trace{\LapInv}}{\numobs}$.

\subsection{Extension to $\numchoices$-ary comparisons}
\label{SecMWise}

Suppose instead of eliciting pairwise comparisons, one can instead ask
the workers to make comparisons between more than two options. In
particular, we assume that each sample is a selection of the item with
the largest perceived quality among some $\numchoices$ presented
items. The setting of pairwise comparisons is a special case with
$\numchoices=2$. Recall from Theorem~\ref{ThmMinimax2} that the
minimum squared $\ell_2$ minimax risk in the pairwise comparison
setting is of the order $\frac{\numitems^2}{\numobs}$. Our goal in
this section is to bring the concept of multiple-item comparison under
the same framework as the pairwise case, and via a generalization of
our earlier theoretical analysis, understand how the error exponent
depends on $\numchoices$. 

Consider $\numitems$ items, where every item $j \in [\numitems]$ has a
certain underlying quality score $\wtstar_j \in [-\wmax,\wmax]$. You
obtain $\numobs$ samples, with each sample being a selection of the
item with the largest perceived value among some $\numchoices$
presented items.

Consider $(\numitems \times \numchoices)$ matrices
$\gendiff_1,\ldots,\gendiff_\numobs$ such that for each $i \in
[\numobs]$, the $\numchoices$ columns of $\gendiff_i$ are distinct
unit vectors. The positions of the non-zero elements in the
$\numchoices$ columns of $\gendiff_i$ represent the identities of the
$\numchoices$ items compared in the $i^{\rm th}$ sample. One can
visualize the choices of the items compared as a hyper-graph, with
$\numitems$ vertices representing the $\numitems$ items and hyper-edge
$i \in [\numobs]$ containing the $\numchoices$ items compared in
observation $i$.

Let $\revmx_1,\ldots,\revmx_\numchoices$ be $(\numchoices \times
\numchoices)$ permutation matrices representing $\numchoices$ cyclic
shifts in an arbitrary (but fixed) direction. Consider the observation
model
\begin{align*}
\mprob(\obs_i=j \vert \wtstar, \gendiff_i) = \glmcdf((\wtstar)^T \gendiff_i \revmx_j)
\end{align*}
for all $j \in [\numchoices]$, where $\glmcdf:
[-\wmax,\wmax]^\numchoices \rightarrow [0,1]$ represents the
probability of choosing the \emph{first} among the $\numchoices$ items
presented. For every $x \in [-\wmax,\wmax]^\numchoices$, $\glmcdf(x)$
is assumed to satisfy:
\begin{itemize}
\item Shift-invariance: the probabilities depend only on the
  \emph{differences} in the weights of the items presented, i.e,
  $\glmcdf(x)$ depends only on $\{x_i - x_j\}_{i,j \in
    [\numchoices]}$.
\item Strong log-concavity: $\nabla^2 (-\log \glmcdf(x)) \succeq
  \hessgencdf$ for some $(\numchoices \times \numchoices)$ symmetric
  matrix $\hessgencdf$ with $\eigenvalue{2}{\hessgencdf} >0$.
\end{itemize}

Note that the shift-invariance assumption implies $\ones
\in \nullspace(\nabla^2 (-\log \glmcdf(x)))$, thereby necessitating
$\nullspace(\hessgencdf)=\spanvectors{\ones}$ and
$\eigenvalue{1}{\hessgencdf}=0$. One can also verify that the model
proposed here reduces to the \ord model of Section~\ref{sec::models}
when $\numchoices=2$.

For any hope of inferring the true weights $\wtstar$, we must ensure
that the comparison hyper-graph is ``connected'', i.e., for every pair
of items $i,j \in [\numitems]$, there must exist a path connecting
item $i$ and item $j$ in the comparison hyper-graph. We assume this
condition is satisfied. We also continue to assume that $\wtstar \in \SPECSET \defn \{ \wt \in \reals^\numitems \mid \Lnorm{\wt}{\infty} \leq \wmax,~\inprod{\wt}{\ones}=0\}$.

The popular Plackett-Luce model falls in this class, as illustrated below.

\begin{example}[Plackett-Luce model~\citep{plackett1975analysis,luce1959individual}]
The Plackett-Luce model concerns the process of choosing an item from a given set. Specifically, given $\numchoices$ items with quality scores $\wtstar_1,\ldots,\wtstar_\numchoices$ respectively, the likelihood of choosing item $i \in [\numchoices]$ under this model is given by 
\begin{align*}
\frac{e^{\wtstar_i}}{\sum_{j=1}^{\numchoices} e^{\wtstar_j}} =: \glmcdf([\wtstar_1,\ldots,\wtstar_\numchoices]). 
\end{align*}
Every choice is made independent of all other choices.

It is easy to verify that the Plackett-Luce model satisfies shift invariance. We now show that it also satisfies strong log-concavity. A little algebra gives
\begin{align*}
\nabla^2 (-\log \glmcdf(x)) = \frac{e^{x_1}}{(\inprod{e^x}{\ones})^4} \big( \inprod{e^x}{\ones} \diag{e^x}   - e^x (e^x)^T \big),
\end{align*}
where $e^x \defn [e^{x_1} \cdots e^{x_\numchoices}]^T$. We will now derive a lower bound for the expression above. An application of the Cauchy-Schwarz inequality yields that for any vector $v \in \reals^\numchoices$,
\begin{align*}
v^T(e^x (e^x)^T ) v 
\leq  v^T  \diag{e^x} \inprod{e^x}{\ones} v,
\end{align*}
with equality if and only if $v \in {\rm span}(\ones)$. It follows that $\eigenvalue{2}{\nabla^2 (-\log \glmcdf(x))} >0$ for all $x \in [-\wmax,\wmax]^\numchoices$. Defining the scalar $\beta \defn \min_{x \in [-\wmax,\wmax]^\numchoices} \eigenvalue{2}{\frac{e^{x_1}}{(\inprod{e^x}{\ones})^4} \big( \inprod{e^x}{\ones} \diag{e^x}   - e^x (e^x)^T \big)}$, on can see that setting $\hessgencdf = \sigma  (\identity - \ones \ones^T)$ satisfies the strong log-concavity conditions.
\end{example}

Our goal is to capture the scaling of the minimax error with respect
to the number of observations $\numobs$, the dimension $\numitems$ of
the problem, and the choice of the subsets compared $\{\gendiff_i\}_{i
  \in [\numobs]}$.  It is well
understood~\citep{miller1956magical,kiger1984depth, shiffrin1994seven,
  saaty2003magic} that humans have a limited information storage and
processing capacity, which makes it difficult to compare more than a
small number of items. For instance,~\citet{saaty2003magic} recommend
eliciting preferences over \emph{no more than seven options}. Thus in
this work we will restrict our attention to $\numchoices =
\order(1)$. Moreover, the amount of noise in the selection process
also depends on the number of items $\numchoices$ presented at a time:
the higher the number, the greater the noise. We will thus not use a
`noise parameter $\noisestd$' in this setting, and assume the noise to
be incorporated in the function $\glmcdf$ which itself is a function
of $\numchoices$.

Our results involve the Laplacian of the comparison graph, defined for
the $\numchoices$-wise comparison setting as follows. Let $\genlapmx$
be an $(\numitems \times \numitems)$ matrix that depends on the choice
of the comparison topology as
\begin{align}
\genlapmx \defn \frac{1}{\numobs} \sum_{i=1}^{\numobs} \gendiff_i
(\numchoices \identity - \ones \ones^T) \gendiff_i^T.
\label{eq:defn_genlapmx}
\end{align}
We will call $\genlapmx$ the Laplacian of the comparison hyper-graph.
One can verify that when applied to the special case of
$\numchoices=2$, the matrix $\genlapmx$ defined
in~\eqref{eq:defn_genlapmx} reduces to the Laplacian of the
pairwise-comparison graph defined earlier in~\eqref{EqnDefnLap}.

The following theorem presents our main results for the
$\numchoices$-wise comparison setting.
\begin{theorem}
For the~\mwise model, the minimax risk is sandwiched as
\begin{align*}
\plaincon_{3\ell} \frac{\inf_z \glmcdf(z)}{\numchoices^2
  \eigenvalue{\numchoices}{\hessgencdf} \sup_{z} \Lnorm{\nabla
    \glmcdf(z)}{\hessgencdf^\dagger}^2} \; \frac{\numitems }{\numobs}
\leq  \inf_{\wthat} \sup_{\wtstar \in \SPECSET} \Exs \Big[ \Lnorm{\wthat - \wtstar}{\Lap}^2
  \Big] \leq \plaincon_{3u}
\frac{\numchoices^2 \sup_z \Lnorm{\nabla \log \glmcdf
    (z)}{2}^2}{\eigenvalue{2}{\hessgencdf}^2}
\frac{\numitems}{\numobs},
\end{align*}
in the squared $\Lap$ semi-norm and as
\begin{align*}
\plaincon_{4\ell} \frac{\inf_z \glmcdf(z)}{\numchoices^2
  \eigenvalue{\numchoices}{\hessgencdf} \sup_{z} \Lnorm{\nabla
    \glmcdf(z)}{\hessgencdf^\dagger}^2} \; \frac{\numitems^2
}{\numobs} \leq \inf_{\wthat} \sup_{\wtstar \in \SPECSET} \Exs \Big[ \Lnorm{\wthat - \wtstar}{2}^2
  \Big]\leq \plaincon_{4u}
\frac{\numchoices^2 \sup_z \Lnorm{\nabla \log \glmcdf
    (z)}{2}^2}{\eigenvalue{2}{\hessgencdf}^2 }
\frac{\numitems^2}{\eigenvalue{2}{\Lap} \numobs},
\end{align*}
in the squared $\ell_2$ norm. Here we assume $\numobs \geq \plaincon_5 \frac{\trace{\genlapmxinv}
  \inf_z \glmcdf(z)}{\wmax^2 \eigenvalue{\numchoices}{\hessgencdf}
  \sup_{z} \Lnorm{\nabla \glmcdf(z)}{\hessgencdf^\dagger}^2} $ for
both the lower bounds, and where the suprema and infima with respect to the parameter $z$ are taken over
the set $[-\wmax,\wmax]^\numchoices$.
\label{thm:mwise}
\end{theorem}

The proof of Theorem~\ref{thm:mwise} is provided in
Appendix~\ref{AppThmMwise}. Our results establish that the dependence
of the squared $\Lap$ semi-norm and squared Euclidean minimax error on
$\numchoices$ occurs only as multiplicative pre-factors, and the error
exponent is independent of $\numchoices$. Thus, if one follows the
standard recommendation in the psychology
literature~\citet{miller1956magical,kiger1984depth, shiffrin1994seven,
  saaty2003magic}---namely to choose $\numchoices = \order(1)$---then
the best possible scaling of the squared \mbox{$\Lap$ semi-norm}
minimax risk with respect to $\numitems$ and $\numobs$ is always
$\frac{\numitems}{\numobs}$, that of the squared Euclidean minimax
risk is always $\frac{\numitems^2}{\numobs}$, and evenly spreading the
samples across all possible choices of $\numchoices$ items is optimal. Nevertheless, a more refined modeling and analysis is required to understand the precise tradeoffs governing the choice of the number $\numchoices$ of items presented to the user.


\section{Role of graph topology}
\label{SecTopology}
We now return to the setting of pairwise comparisons. 
In certain applications, one may have the liberty to decide which
pairs are compared.  The results of the previous section demonstrated
the role played by the Laplacian of the comparison graph in the
estimation error. We now employ these results to derive guidelines
towards designing the comparison graph. Let us focus on the estimation
error in the squared $\ell_2$ norm in the ordinal setting. As
discussed earlier, we assume that the graph induced by the comparisons
is connected. An application of Theorem~\ref{ThmMinimax2} lets us
identify good topologies for pairwise comparisons in the fixed-design
setup.

A popular class of comparison topologies is that of evenly distributed
samples on an unweighted graph (e.g.,
\citep{negahban2012iterative}). Consider any fixed, unweighted graph
$\graph = (\vertices, \edges)$. We assume that the samples are
distributed evenly along the edges $\edges$ of $\graph$, and that the
sample size $\numobs$ is sufficiently large. Using standard matrix
concentration inequalities, it is straightforward to extend our
analysis to the setting of random chosen comparisons from a fixed
graph (see, for instance, \citet{mchammer}). Let $\Lap'$ denote the
Laplacian of $\graph$. We define the \emph{scaled Laplacian} of
$\graph$ as
\begin{align*}
\Lap \defn \frac{1}{\mid \edges \mid} \Lap'.
\end{align*}
One can verify that the matrix $\Lap$ defined here is identical to
what was defined in~\eqref{EqnDefnLap} in a more general context. In
order to differentiate from $\Lap$, we will term $\Lap'$ as the
\emph{regular Laplacian} of the graph $\graph$.


\subsection{Analytical results}

Consider the \ord model and the squared $\ell_2$-norm as the metric of
interest. We claim that in order to determine whether a given
comparison graph achieves minimax risk (up to a constant pre-factor),
it suffices to examine the eigen-spectrum of the scaled Laplacian
matrix.  In particular, we claim that:
\begin{itemize}
\item If the scaled Laplacian has a second smallest eigenvalue that
  scales as $\frac{1}{\eigenvalue{2}{\Lap}} = \Theta(\numitems)$, then
  the comparison graph is optimal, and leads to the smallest possible
  minimax risk, in particular one that scales as
  $\frac{\numitems^2}{\numobs}$.
\item Conversely, if the scaled Laplacian matrix has an eigen-spectrum
  satisfying 
\begin{align}
\numitems^2 = o \left(\max_{\numitems' \in \{2,\ldots,\numitems\}}
\sum_{i = \lfloor 0.99 \numitems' \rfloor}^{\numitems'}
\frac{1}{\eigenvalue{i}{\Lap}} \right),
\end{align}
then the associated estimation error is strictly larger than the
minimax risk.  In particular, this sub-optimality holds whenever
$\numitems^2 = o(\frac{1}{\eigenvalue{2}{\Lap}})$.
\end{itemize}
In order to verify these claims, we note that by
definition~\eqref{EqnDefnLap} of the Laplacian matrix, we have
\begin{align*}
\trace{\Lap} = \frac{1}{\numobs} \sum_{i=1}^{\numobs} \trace{ \diff_i
  \diff_i^T} = 2.
\end{align*}
It follows that $\eigenvalue{2}{\Lap} \leq \frac{2}{\numitems-1}$,
i.e., that $\frac{1}{\eigenvalue{2}{\Lap}} = \Omega(\numitems)$. As we
will see shortly, several classes of graphs satisfy
$\frac{1}{\eigenvalue{2}{\Lap}} = \Theta(\numitems)$. Comparing the
lower bound of $\Omega(\frac{\numitems^2}{\numobs})$ on the minimax
risk~\eqref{Eqn2LowerBound} with the upper
bound~\eqref{Eqn2UpperBound} gives the sufficient condition of
$\frac{1}{\eigenvalue{2}{\Lap}} = \Theta(\numitems)$ for optimality,
and the smallest minimax risk as
$\Theta(\frac{\numitems^2}{\numobs})$. The lower
bound~\eqref{Eqn2LowerBound} now also gives the claimed condition for
strict sub-optimality.

In order to illustrate these claims, let us consider a few canonical
classes of graphs, and study how the estimation error under the
squared Euclidean norm scales in the \ord model.  The spectra of the
regular Laplacian matrices of these graphs can be found in various
standard texts on spectral graph theory (e.g.,
\cite{brouwer2011spectra}).
\begin{itemize}[leftmargin=*]
\item {\bf Complete graph.} A complete graph has one edge between
  every pair of nodes. The spectrum of the regular Laplacian of the
  complete graph is $0,\numitems,\ldots,\numitems$, and hence the
  spectrum of the scaled Laplacian $\Lap$ is
  $0,\frac{2}{\numitems-1},\ldots,\frac{2}{\numitems-1}$. Substituting
  $\eigenvalue{2}{\Lap} = \frac{2}{\numitems-1}$ in
  Theorem~\ref{ThmMinimax2}b gives an upper bound of
  $\Theta(\frac{\numitems^2}{\numobs})$ on the minimax risk, and
  Theorem~\ref{ThmMinimax2} gives a matching lower bound. The
  sufficiency condition discussed above proves optimality.
\item {\bf Constant-degree expander.} The spectrum of the regular
  Laplacian is $0, \Theta(\numitems),\Omega(\numitems), \ldots,
  \Omega(\numitems)$. Since the number of edges is
  $\Theta(\numitems)$, the spectrum of the scaled Laplacian equals
  $0,\Theta(\frac{1}{\numitems}),\Omega(\frac{1}{\numitems}),\ldots,\Omega(\frac{1}{\numitems})$. The
  evaluation of this class of graphs with respect to the minimax risk
  is identical to that of complete graphs, giving a lower and upper
  bound of $\Theta(\frac{\numitems^2}{\numobs})$ on the minimax risk,
  and guaranteeing optimality.
\item {\bf Complete bipartite.} The $\numitems$ nodes are partitioned
  into two sets comprising, say, $m_1$ and $m_2$ nodes. There is an
  edge between every pair of nodes in different sets, and there are no
  edges between any two nodes in the same set. The eigenvalues of the
  regular Laplacian of this graph are $0,\underbrace{
    m_2,\ldots,m_2}_{\scriptscriptstyle m_1-1},\underbrace{
    m_1,\ldots,m_1}_{\scriptscriptstyle m_2-1},m_1+m_2$. Since the
  total number of edges is $m_1m_2$, the scaled Laplacian $\Lap$ has a
  spectrum $0,\underbrace{\scriptstyle
    \frac{1}{m_1},\ldots,\frac{1}{m_1}}_{\scriptscriptstyle
    m_1-1},\underbrace{\scriptstyle
    \frac{1}{m_2},\ldots,\frac{1}{m_2}}_{\scriptscriptstyle
    m_2-1},\frac{1}{m_1}+\frac{1}{m_2}$. Suppose without loss of
  generality that $m_1 \geq m_2$. Also suppose that $m_2>1$ (the case
  of $m_2=1$ is the star graph discussed below). Then we have
  $\frac{1}{m_1} \leq \frac{1}{m_2} \leq \frac{1}{m_1}+\frac{1}{m_2}$
  and that $\numitems > m_1 \geq \frac{\numitems}{2}$. Furthermore
  since $m_2>1$, the multiplicity of $\frac{1}{m_1}$ in the spectrum
  of the scaled Laplacian is at least $1$. Thus we have
  $\eigenvalue{2}{\Lap} =
  \Theta(\frac{1}{\numitems})$. Theorem~\ref{ThmMinimax2} then gives
  lower and upper bounds on the minimax risk as
  $\Theta(\frac{\numitems^2}{\numobs})$ and the sufficiency condition discussed above  guarantees its optimality.
\item {\bf Star.} A star graph has one central node with edges to
  every other node. It is a special case of the complete bipartite
  graph with $m_1 = \numitems-1$ and $m_2=1$. The spectrum of the
  regular Laplacian is $0,1,\ldots,1,\numitems$. Since there are
  $(\numitems-1)$ edges, the spectrum of the scaled Laplacian is $0,
  \frac{1}{\numitems-1}, \ldots, \frac{1}{\numitems-1},
  \frac{\numitems}{\numitems-1}$. Theorem~\ref{ThmMinimax2} and the
  sufficiency condition discussed above imply that this
  class of graphs is optimal and is associated to a minimax risk of
  $\Theta(\frac{\numitems^2}{\numobs})$.
\item {\bf Path.} A path graph is associated to an arbitrary ordering
  of the $\numitems$ nodes with edges between pairs $j$ and $(j+1)$
  for every $j \in \{1,\ldots,\numitems-1\}$. The spectrum of the
  regular Laplacian is given by $2 \big(1-\cos \big(\frac{\pi
    i}{\numitems}\big) \big),~i \in \{0,\ldots,\numitems-1\}$, and
  that of the scaled Laplacian is thus $\frac{2}{\numitems-1}
  \big(1-\cos \big(\frac{\pi i}{\numitems}\big) \big),~i \in
  \{0,\ldots,\numitems-1\}$. The relation $(1-\cos x) = \sin^2
  \frac{x}{2}$ and the approximation $\sin x \approx x$ for values of
  $x$ close to zero gives $\eigenvalue{2}{\Lap} =
  \Theta(\frac{1}{\numitems^3})$. The minimax risk is thus upper
  bounded as $\order(\frac{\numitems^4}{\numobs})$ and lower bounded
  as $\Omega(\frac{\numitems^3}{\numobs})$. This class of graphs is
  thus strictly suboptimal.
\item {\bf Cycle.} A cycle is identical to a path except for an
  additional edge between node $\numitems$ and node $1$. The spectrum
  of the regular Laplacian is given by $2 \big(1-\cos \big(\frac{2\pi
    i}{\numitems}\big) \big),~i \in \{0,\ldots,\numitems-1\}$, and
  that of the scaled Laplacian is thus $\frac{2}{\numitems}
  \big(1-\cos \big(\frac{2 \pi i}{\numitems}\big) \big),~i \in
  \{0,\ldots,\numitems-1\}$. The relation $(1-\cos x) =
  \sin^2\frac{x}{2}$ and the approximation $\sin x \approx x$ for
  values of $x$ close to zero gives $\eigenvalue{2}{\Lap} =
  \Theta(\frac{1}{\numitems^3})$. The minimax risk is thus upper
  bounded as $O(\frac{\numitems^4}{\numobs})$ and lower bounded as
  $\Omega(\frac{\numitems^3}{\numobs})$. This class of graphs is thus
  strictly suboptimal.
\item {\bf Barbell.} The nodes are partitioned into two sets of
  $\frac{\numitems}{2}$ nodes each, and there is an edge between every
  pair of nodes within each set. In addition, there is exactly one
  edge across the sets. The spectrum of the regular Laplacian can be
  computed as $0,\Theta(\frac{1}{\numitems}), \Theta(\numitems),
  \ldots, \Theta(\numitems)$. Since there are $\Theta(\numitems^2)$
  edges, the spectrum of the scaled Laplacian turns out to become $0,
  \Theta(\frac{1}{\numitems^3}),
  \Theta(\frac{1}{\numitems}),\ldots,\Theta(\frac{1}{\numitems}),
  \Omega(\frac{1}{\numitems}) $. Applying the results derived earlier
  in the paper, we get that a lower bound of $\Omega(\frac{\numitems^3}{\numobs})$
  and an upper bound of $O(\frac{\numitems^4}{\numobs})$ on the minimax risk,
  thereby also establishing the sub-optimality of this class of graphs.
\item {\bf 2D Lattice.} An $(m_1 \times m_2)$ lattice has $\numitems =
  m_1 m_2$ vertices arranged as a $(m_1 \times m_2)$ grid. Assume $m_1
  = \Theta(\numitems)$ and $m_2 = \Theta(\numitems)$. This class of
  graphs can be written as a Cartesian product of a path graph of
  length $m_1$ and a second path graph of length $m_2$. As a result,
  the spectrum of the scaled Laplacian is $\frac{2}{\numitems} \big(2-
  \cos \big(\frac{\pi i}{m_1}\big) - \cos \big(\frac{\pi j}{m_2}\big)
  \big)$,~~~{\color{white}...} $\scriptstyle i \in
  \{0,\ldots,m_1-1\},j\in \{0,\ldots,m_2-1\}$. Again, using the small
  angle approximation of the sinusoid, one can compute an upper bound
  on the minimax risk as $\order(\frac{\numitems^{3}}{\numobs})$ and a
  lower bound of $\Omega(\frac{\numitems^{2}}{\numobs})$. We do not know at
  this point whether the 2D lattice minimizes the minimax risk.
\item {\bf Hypercube.} Assume $\numitems = 2^m$ for some integer
  $m$. Representing each node as a distinct $m$-length binary vector,
  an edge exists between the nodes corresponding to any pair of
  vectors within a Hamming distance of one. The hypercube is an
  $m$-fold Cartesian product of a path with two nodes, and hence the
  regular Laplacian has an eigenvalue of $2i$ with multiplicity ${m
    \choose i}$, for $i \in \{0,\ldots,m\}$. The scaled Laplacian has
  an eigenvalue of $\frac{2i}{\numitems \log \numitems}$ with
  multiplicity ${m \choose i}$, for $i \in \{0,\ldots,m\}$. A lower
  bound on the minimax risk is $\Omega(\frac{\numitems^2}{\numobs})$
  and an upper bound is $O(\frac{\numitems^2 \log
    \numitems}{\numobs})$. We do not know if the hypercube is optimal,
  our bounds do tell us that any sub-optimality is bounded by at most a
  logarithmic factor.
\end{itemize}

Observe that the degree-$k$ expander requires $n \geq k \numitems$
samples while the complete graph requires $\numobs \geq {d \choose 2}$
samples, so in practical applications at least for small sample sizes
we should prefer a low-degree expander.

Finally, if the conjecture in Section~\ref{SecRates2} were true, namely
that the $\ell_2$ minimax risk scales as $\sigma^2 \mathrm{tr}(\LapInv)/n$,
then the condition $\trace{\LapInv} = \Theta(\numitems^2)$ would be
necessary and sufficient for optimality of a comparison graph with the
scaled Laplacian $\Lap$. Observe that the graphs designated as `optimal' in the discussion above
indeed satisfy this condition. On the other hand, the graphs
established as strictly suboptimal have $\trace{\LapInv} =
\Omega(\numitems^3)$.


\subsection{Experiments and simulations}

This section evaluates the dependence of the squared $\ell_2$-error on
the topology of the comparison graph. We consider the following five
topologies: path, barbell, complete, expander and 2D-lattice. In order
to form an expander graph, we used the Gabber-Galil
construction~\citep{gabber1981explicit}. For any chosen graph
topology, the $\numobs$ difference vectors are selected as one edge
each chosen uniformly at random (with replacement) from the comparison
graph. Recall that our theory predicts that the complete and expander
graphs will perform the best, and that the line and dumbbell graphs will fare
the worst. Also recall that our theory predicts the error will scale as
$\Lnormd{\wtstar}{\wthat}{2}^2$ scales with $n$ as $1/\numobs$ in the
complete and expander topologies.


\subsubsection{Experiments on synthetic data}
\label{sec:experiments_topology_synthetic}

This section describes simulations using data generated synthetically from
the \tstone model. In the simulations, we first generate a quality score vector $\wtstar \in \SPECSET$ using one of the procedures described below. Once $\wtstar$ is chosen, the $\numobs$ pairwise comparisons for any given topology are generated as follows. An edge is selected uniformly (with replacement) at random from the underlying graph, and the chosen edge determines the pair of items compared. The outcome of the comparison is generated as per the \tstone model with the chosen $\wtstar$ as the underlying quality score. Finally, the maximum likelihood estimator for the \tstone model is employed to estimate $\wtstar$. Every point in the plots is an average across $40$ trials.

The following six procedures are employed to generated the true quality score vector $\wtstar$ in the six respective subfigures of Figure~\ref{fig:varyTopology_synthetic}.
\begin{enumerate}[label=(\alph*)]
\item Gaussian: $\wtstar$ is drawn from the standard normal distribution $\mathcal{N}(0,I)$.
\item Uniform: $\wtstar$ is drawn uniformly at random from the set $[-1,1]^\numitems$.
\item Packing set for the path graph: We first choose a vector $z$ as by setting a value of $0$ in the first coordinate, a value $-1$ in $\frac{\numitems}{2}$ of the other coordinates chosen uniformly at random, and a value $1$ in the remaining coordinates. Letting $\Lap = U^T \Lambda U$ denote the eigen-decomposition of the Laplacian matrix of the path graph, $\wtstar$ is set as $U^T \Lambda^\dagger z$, where $\Lambda^\dagger$ is the Moore-Penrose pseudoinverse of $\Lambda$. This generation process mimics a construction used to prove the lower bound in Theorem~\ref{ThmMinimax2}, and tailors the construction for the path graph. 
\item Packing set for the barbell graph: The procedure is identical to that in (c), except that the Laplacian matrix used is that of the barbell graph.
\item Packing set for the complete graph: The procedure is identical to that in (c), except that the Laplacian matrix used is that of the complete graph.
\item Packing set for the star graph: The procedure is identical to that in (c), except that the Laplacian matrix used is that of the star graph.
\end{enumerate}
The vector $\wtstar$ generated in this procedure is then scaled and shifted to ensure $\wtstar \in \SPECSET$. The value of $\wmax$ and $\noisestd$ are set as $1$.

\begin{figure*}
\centering
\begin{subfigure}{.49\textwidth}
\includegraphics[width=\textwidth]{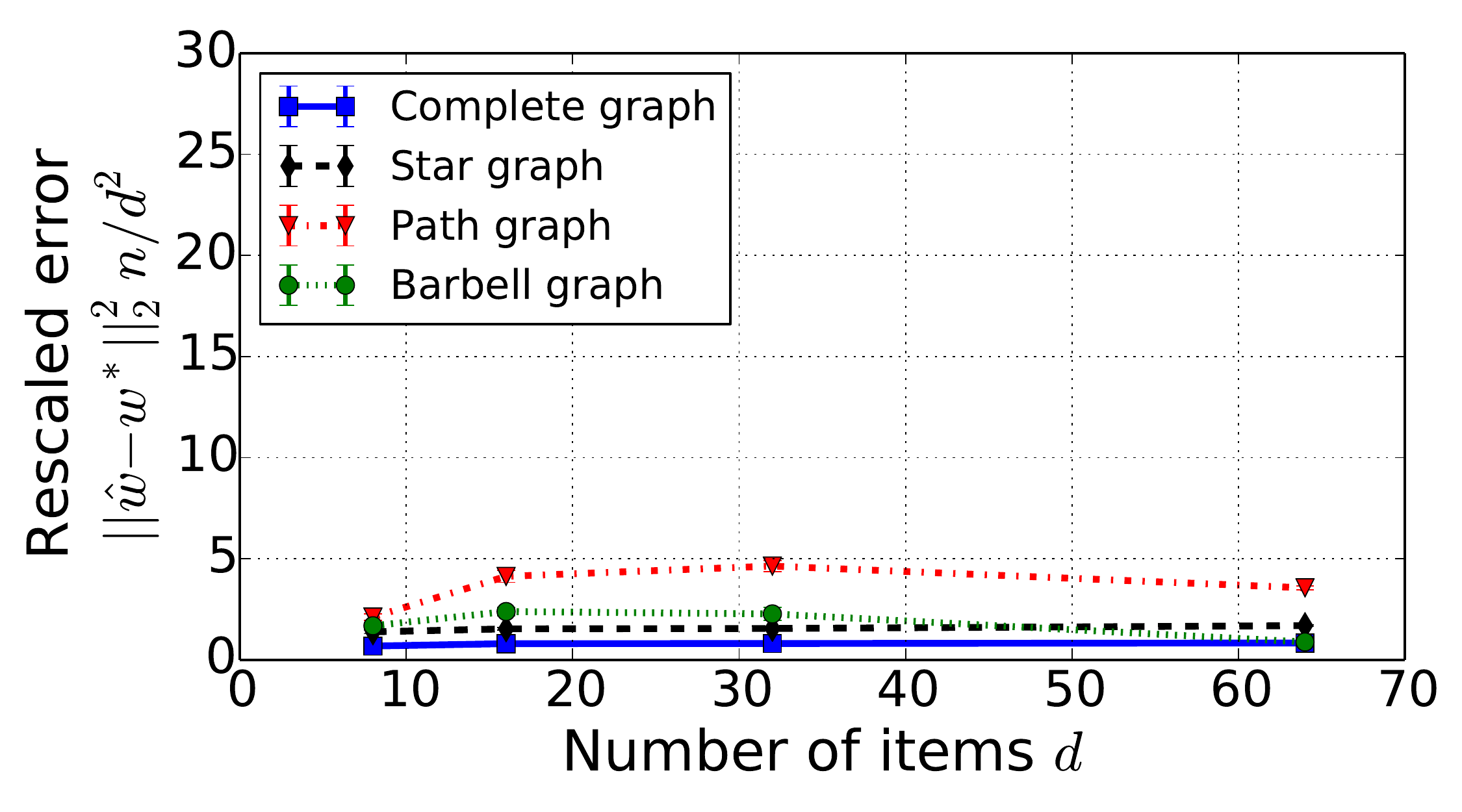}
\caption{Gaussian}
\end{subfigure}
\begin{subfigure}{.49\textwidth}
\includegraphics[width=\textwidth]{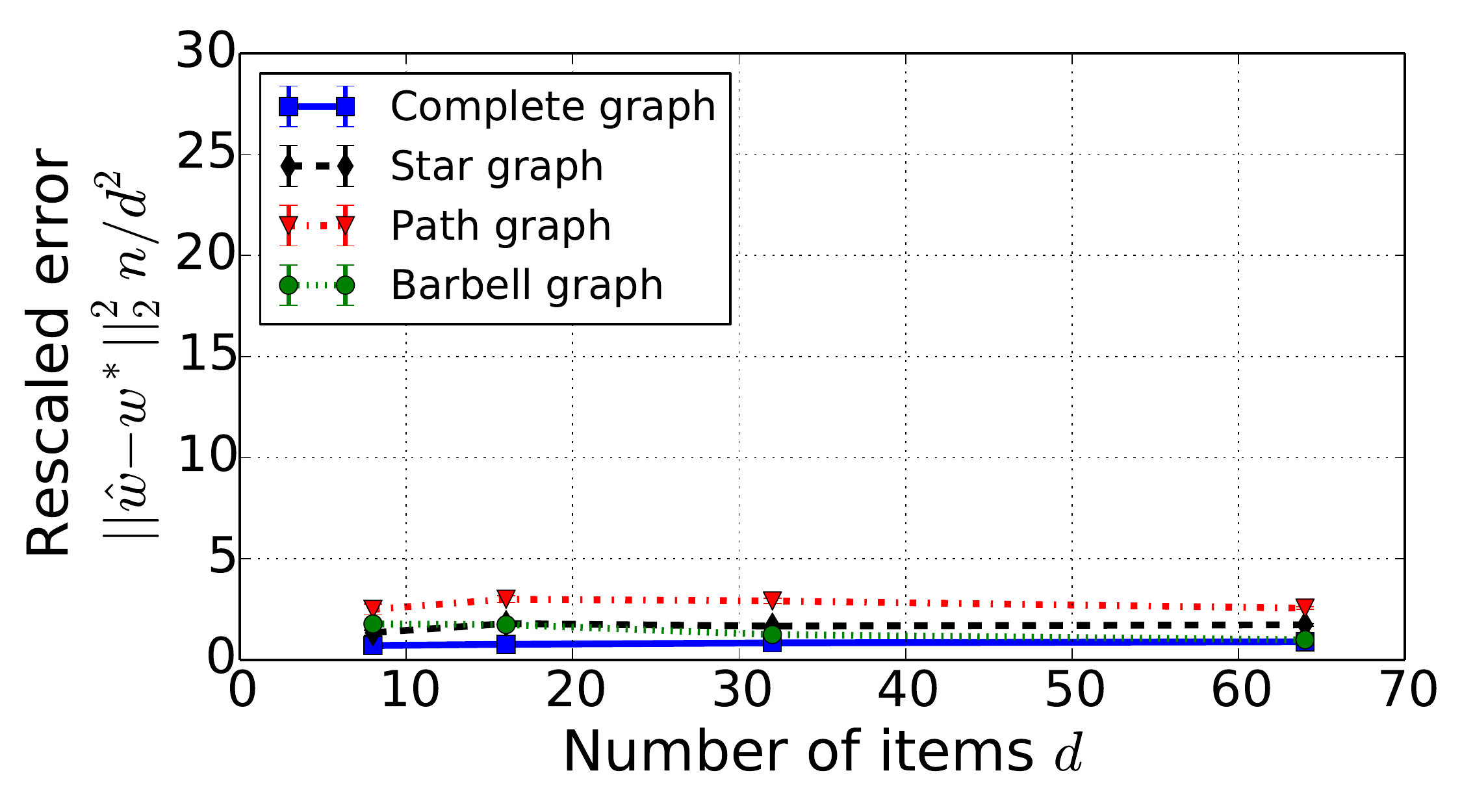}
\caption{Uniform}
\end{subfigure}\\
\begin{subfigure}{.49\textwidth}
\includegraphics[width=\textwidth]{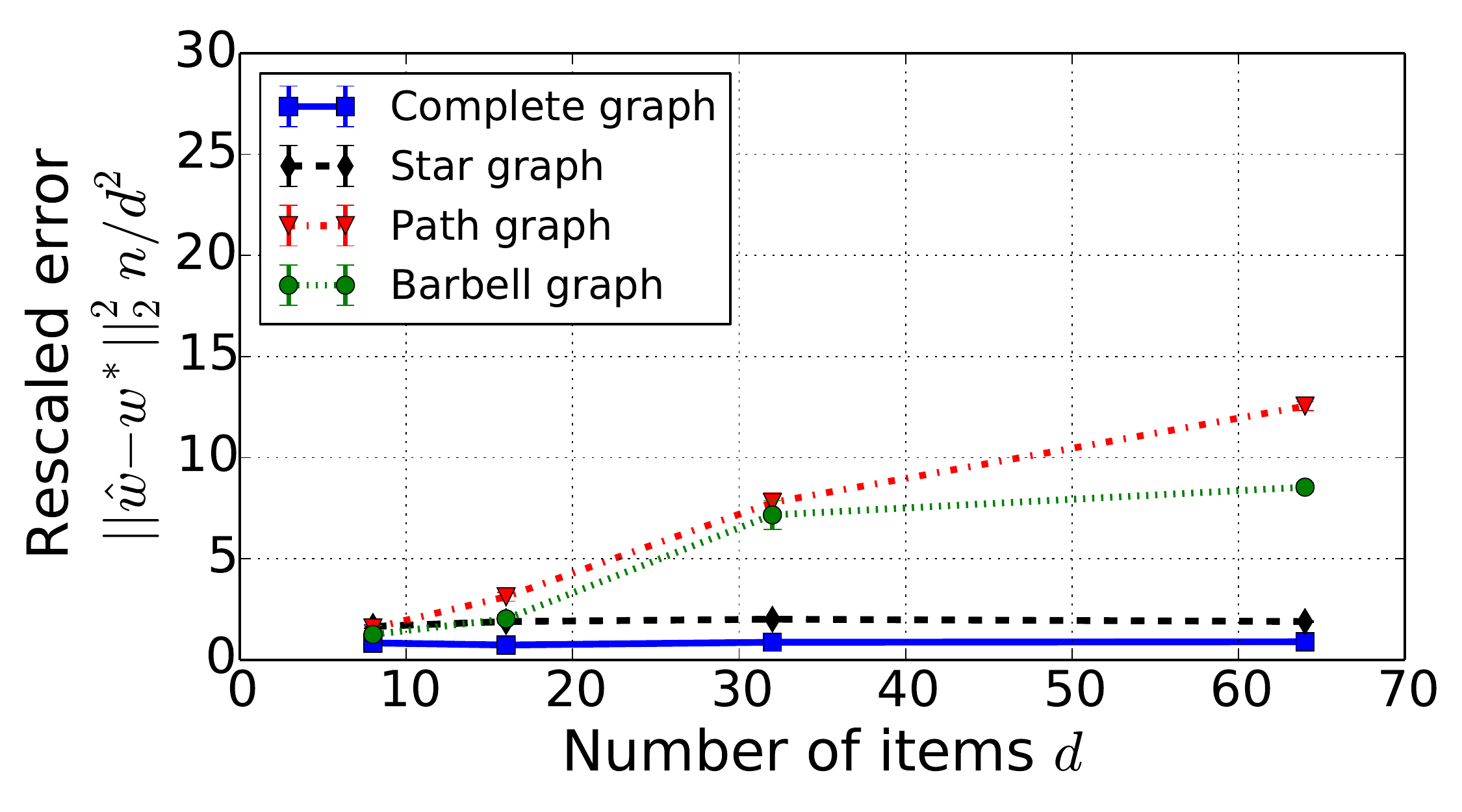}
\caption{Packing set for the path graph}
\end{subfigure}
\begin{subfigure}{.49\textwidth}
\includegraphics[width=\textwidth]{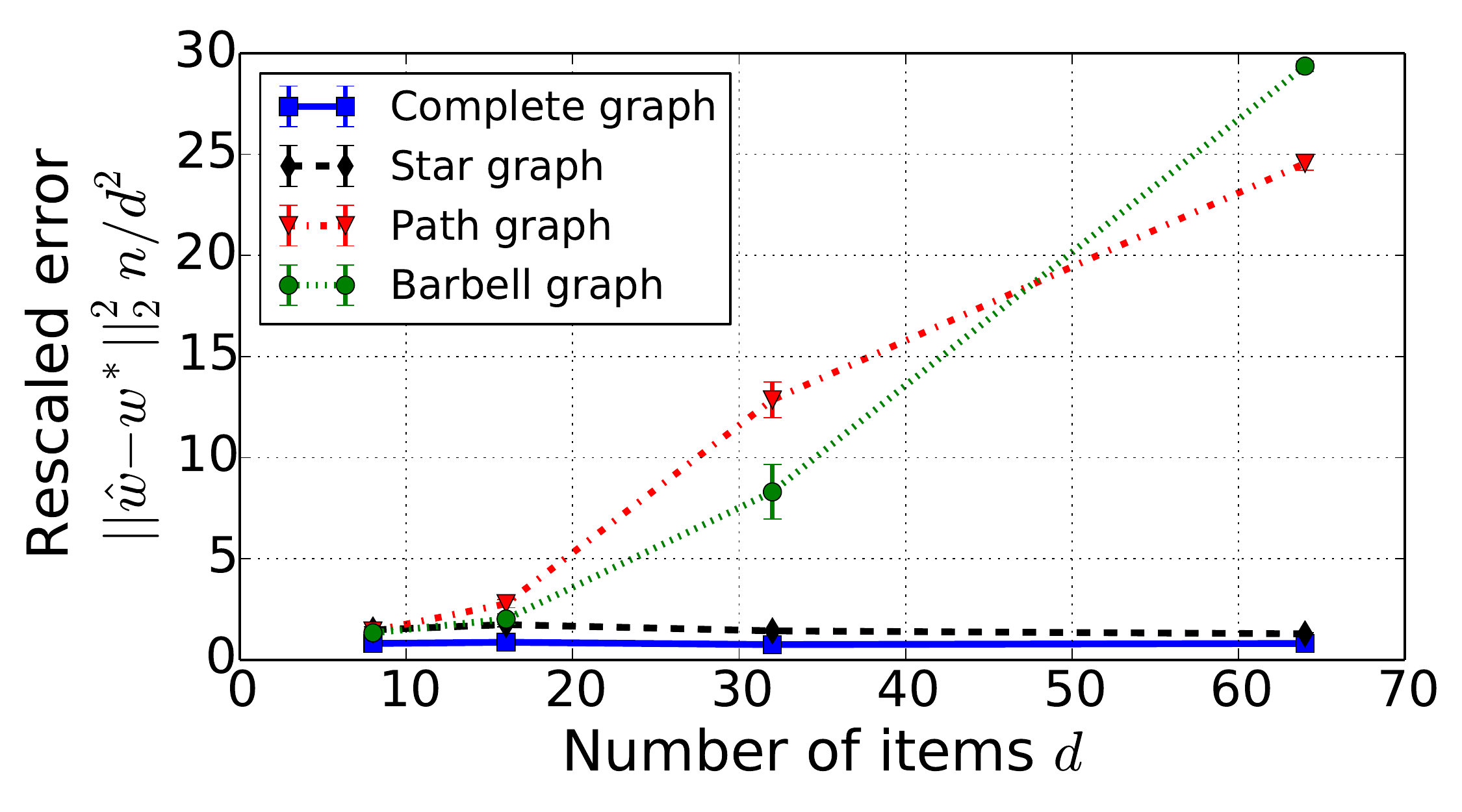}
\caption{Packing set for the barbell graph}
\end{subfigure}
\\
\begin{subfigure}{.49\textwidth}
\includegraphics[width=\textwidth]{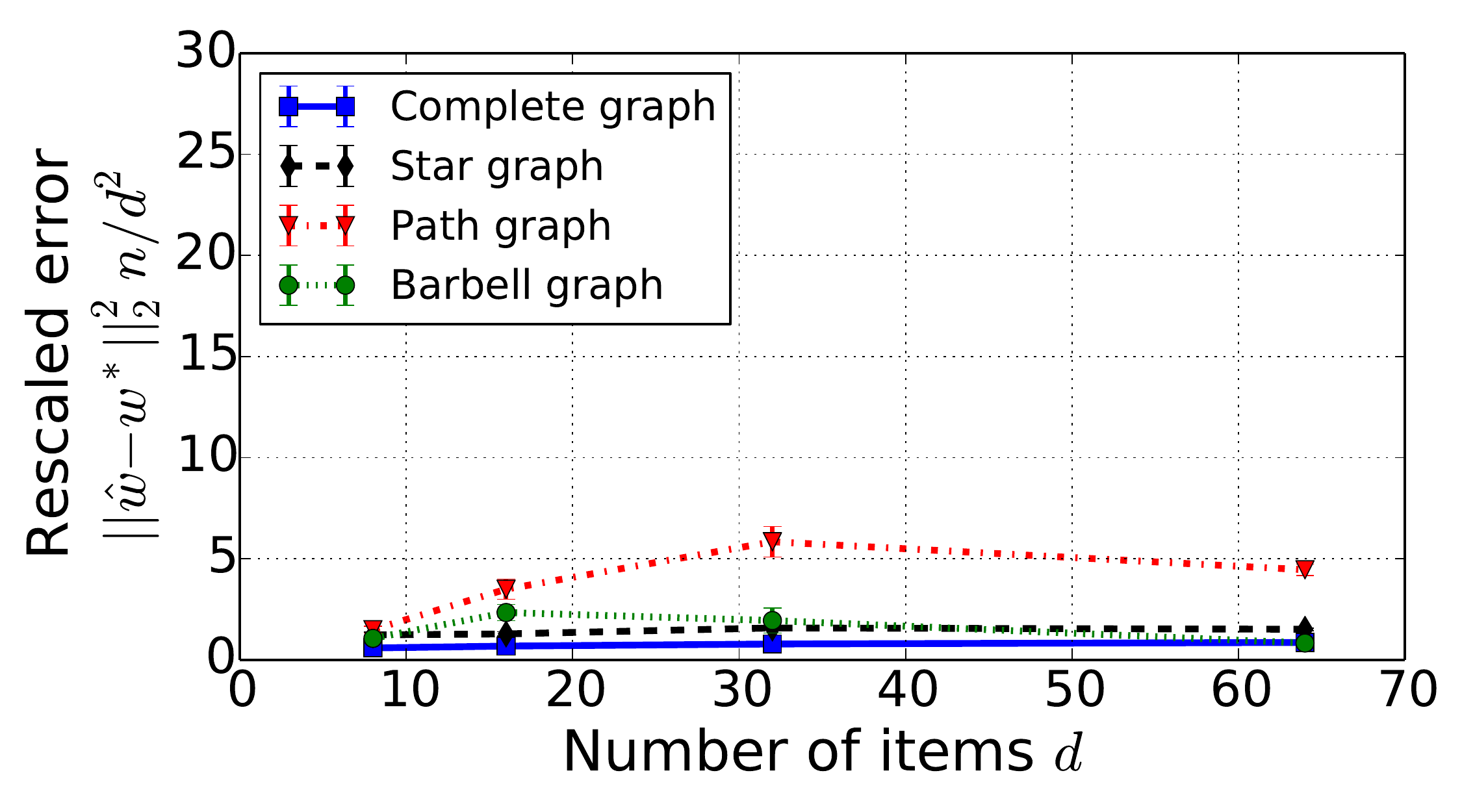}
\caption{Packing set for the complete graph}
\end{subfigure}
\begin{subfigure}{.49\textwidth}
\includegraphics[width=\textwidth]{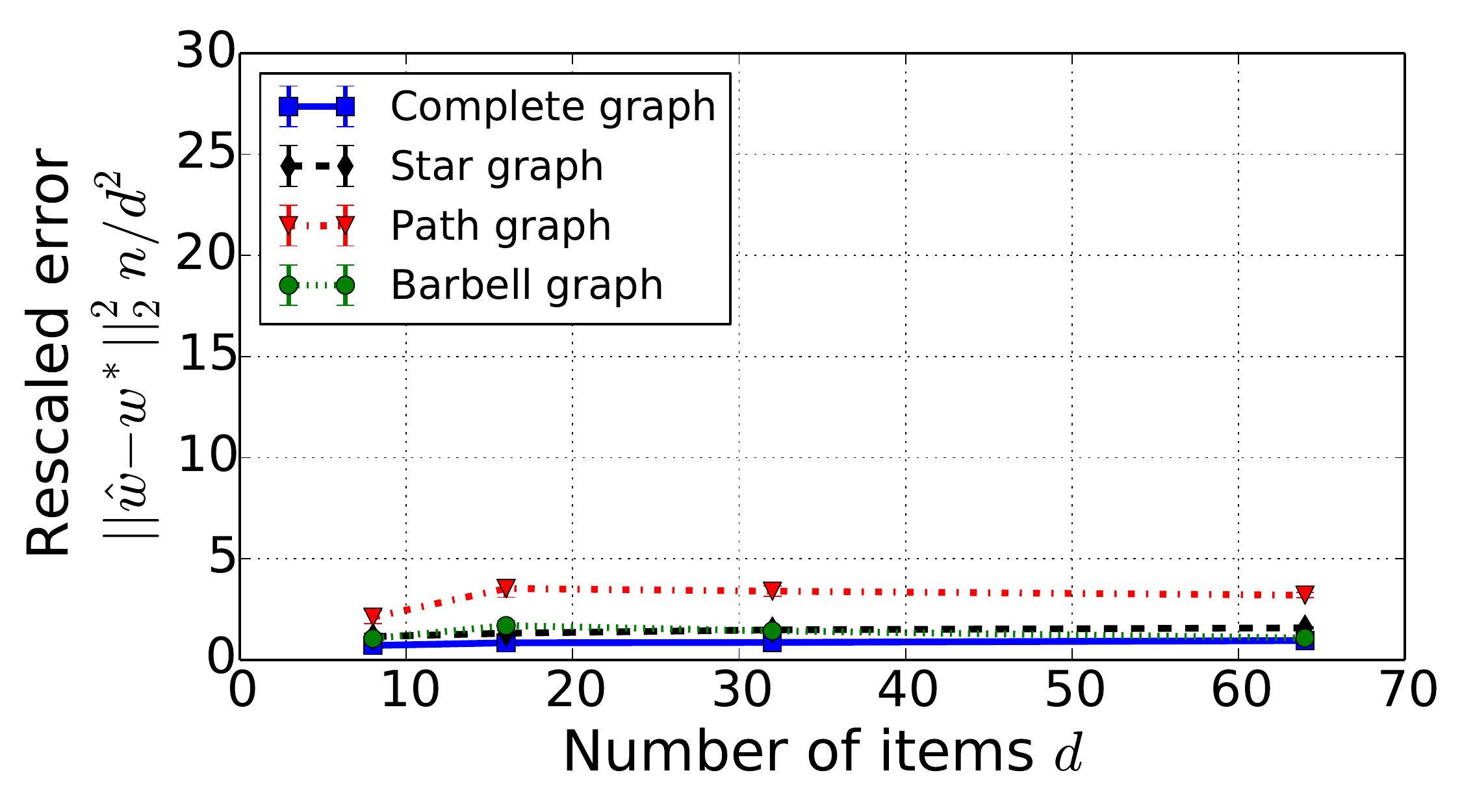}
\caption{Packing set for the star graph}
\end{subfigure}
\caption{Estimation error under different topologies for different generative processes in the synthetic simulations.}
\label{fig:varyTopology_synthetic}
\end{figure*}

Figure~\ref{fig:varyTopology_synthetic} plots the
estimation error under various topologies of the comparison
graph. Observe in the figure that the error is the lowest under the
complete and the star graphs, and the highest under the barbell and the path graphs. In particular, the error consistently varies as $\Theta(\numitems^2/\numobs)$ for the complete and star graphs -- this phenomenon holds even in plots (e) and (f) where the procedure to choose $\wtstar$ forms the worst case for the complete and star graphs respectively according to the proof of Theorem~\ref{ThmMinimax2}. On the other hand, the minimax error varies as $\Omega(\numitems^3/\numobs)$ in the worst case for the path and the barbell graphs. Finally, observe that in the simulations, the (constant) multiplicative factors to the term $\frac{\numitems^2}{\numobs}$ in the error turn out to be rather small, in the range of $0$ to $9$.


\subsubsection{Experiments on MTurk}
\label{sec:experiments_topology_mturk}

In this section, we describe the results of experiments conducted on
the popular Amazon Mechanical Turk (\url{https://www.mturk.com/}; henceforth referred to
as ``MTurk'') commercial crowdsourcing platform, evaluating the
effects of the choice of the topology. MTurk is an online
 platform where individuals or businesses can put up
a task, and any individual can log in and complete the tasks in
exchange for a payment that is specified along with the task. In our
experiments, each worker was offered $20$ cents per completed task. A
worker was allowed to do no more than one task in an
experiment. Workers were required to answer all the questions in a
task. Only those workers who had $100$ or more prior approved works
and an approval rate of $95\%$ or higher were allowed. Workers from
any country were allowed to participate, except for the task of
estimating distances between cities (for which only USA-based workers
were permitted since all questions involved American cities).

We conducted three experiments that required the workers to make ordinal choices.
\begin{enumerate}
\item[(a)] \emph{Estimating areas of circles:} In each question, the
  worker was shown a circle in a bounding box
  (Figure~\ref{fig:areaCircle_ordinal}), and the worker was required
  to identify the fraction of the box's area that the circle occupied.
\item[(b)] \emph{Estimating age of people from photographs:} The
  worker was shown photographs of people
  (Figure~\ref{fig:photoAge_ordinal}) and was asked to estimate their
  ages.
\item[(c)] \emph{Estimating distances between pairs of cities:} Pairs
  of cities were listed (Figure~\ref{fig:distanceCities_cardinal}) and
  for each pair, the worker had to estimate the distance between them.
\end{enumerate}

For each experiment, we recruited 140 workers on MTurk, and assigned
them to one of the five topologies uniformly at random. In this
experiment and others involving aggregation of ordinal data from
MTurk, the aggregation procedure follows maximum likelihood estimation
under the \tstone model, and the estimator is supplied the
best-fitting value of $\noisestd$ obtained via 3-fold
cross-validation. Each run of the estimation procedure employs the
data provided by five randomly chosen workers from the pool of workers
who performed that task. The
entire data pertaining to these experiments is available on the first
author's website.

Figure~\ref{fig:varyTopology_mturk} plots the squared $\ell_2$
estimation error for the three experiments under the five topologies
considered. We see that the relative errors are generally consistent
with our theory, with the complete graph exhibiting the best
performance and the path graph faring the worst. On real datasets,
model misspecification can in some cases
cause the outcomes to differ from our theoretical predictions. 
Understanding 
the effect of model misspecification, especially on topology
considerations, is an important question we hope to address in future work.

\begin{figure*}
\centering
\begin{subfigure}{.32\textwidth}
\includegraphics[width=\textwidth]{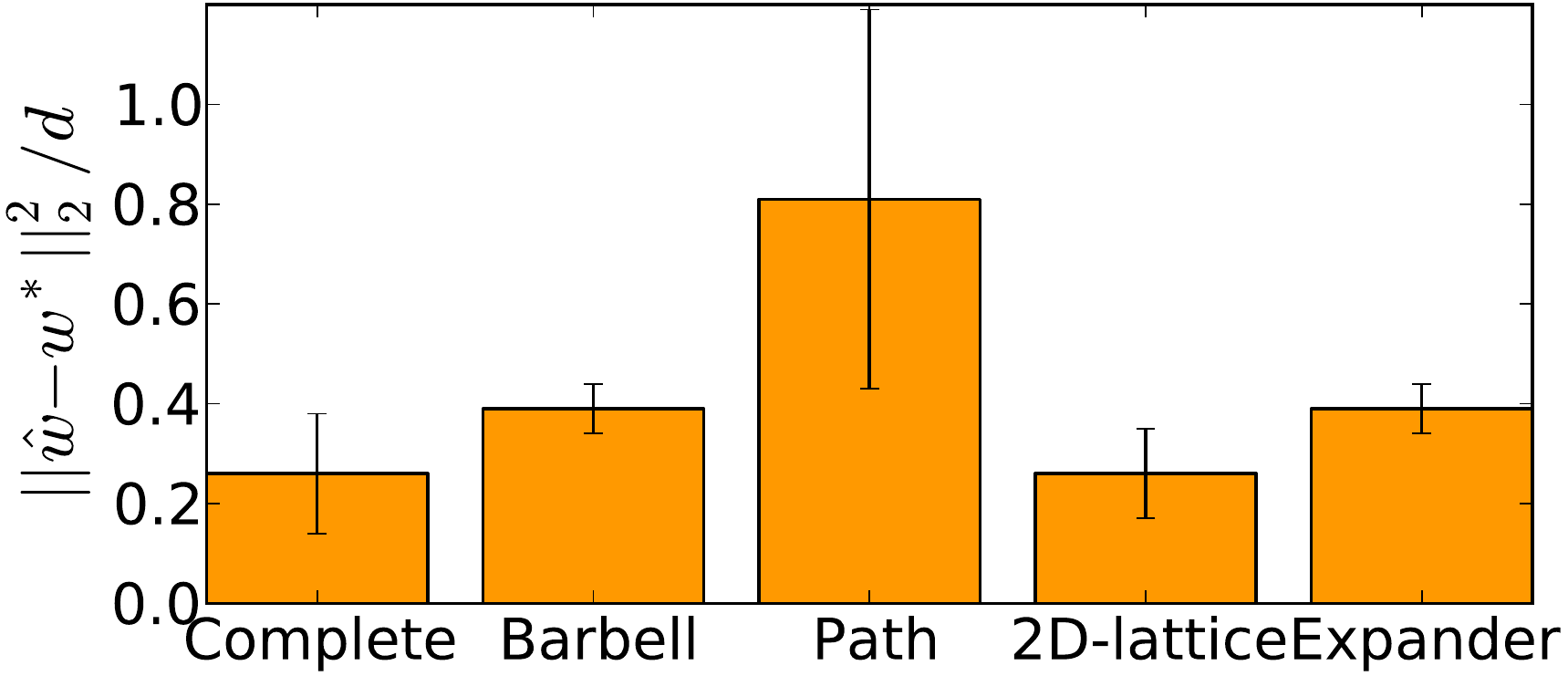}
\caption{Area of circle}
\end{subfigure}
\begin{subfigure}{.32\textwidth}
\includegraphics[width=\textwidth]{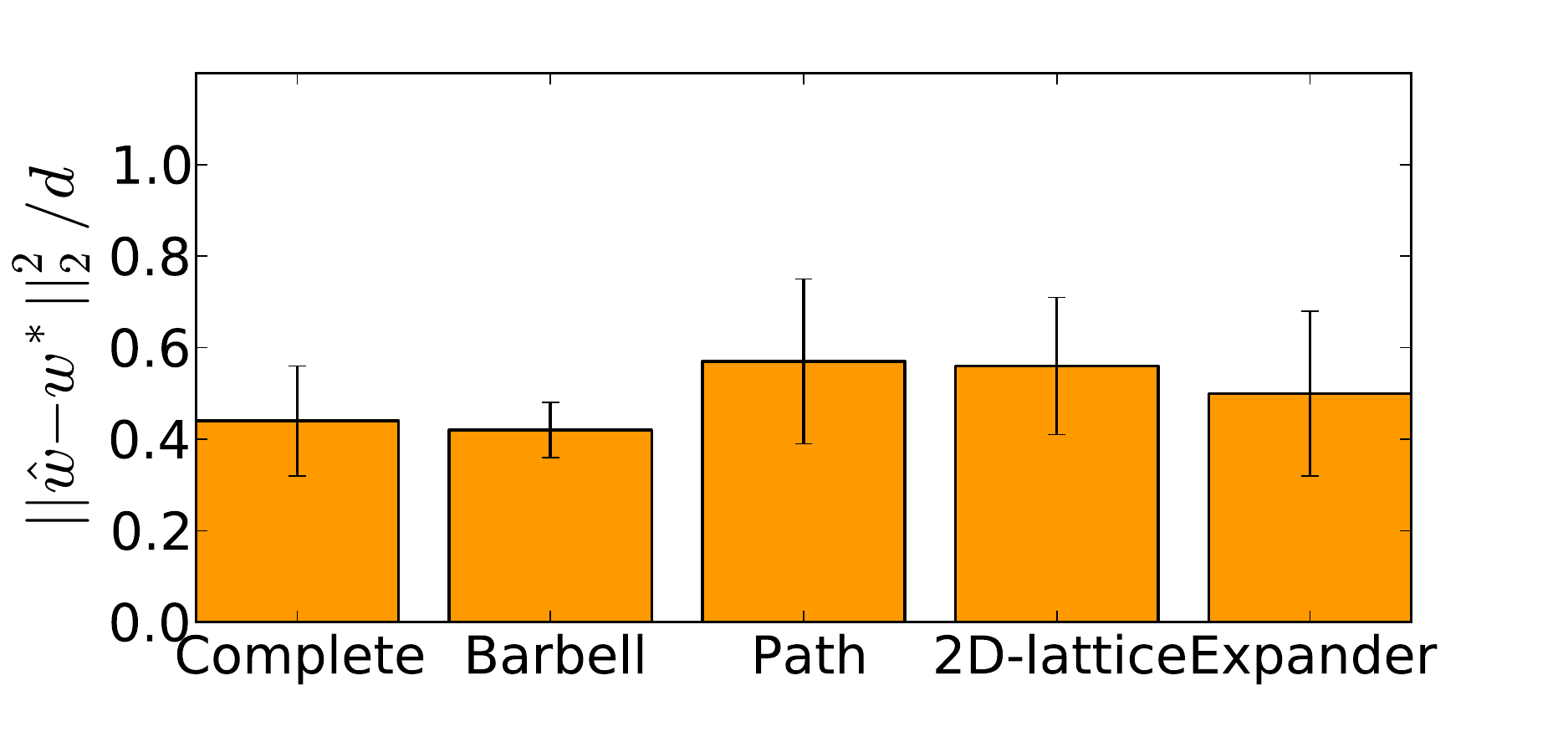}
\caption{Age from photograph}
\end{subfigure}
\begin{subfigure}{.32\textwidth}
\includegraphics[width=\textwidth]{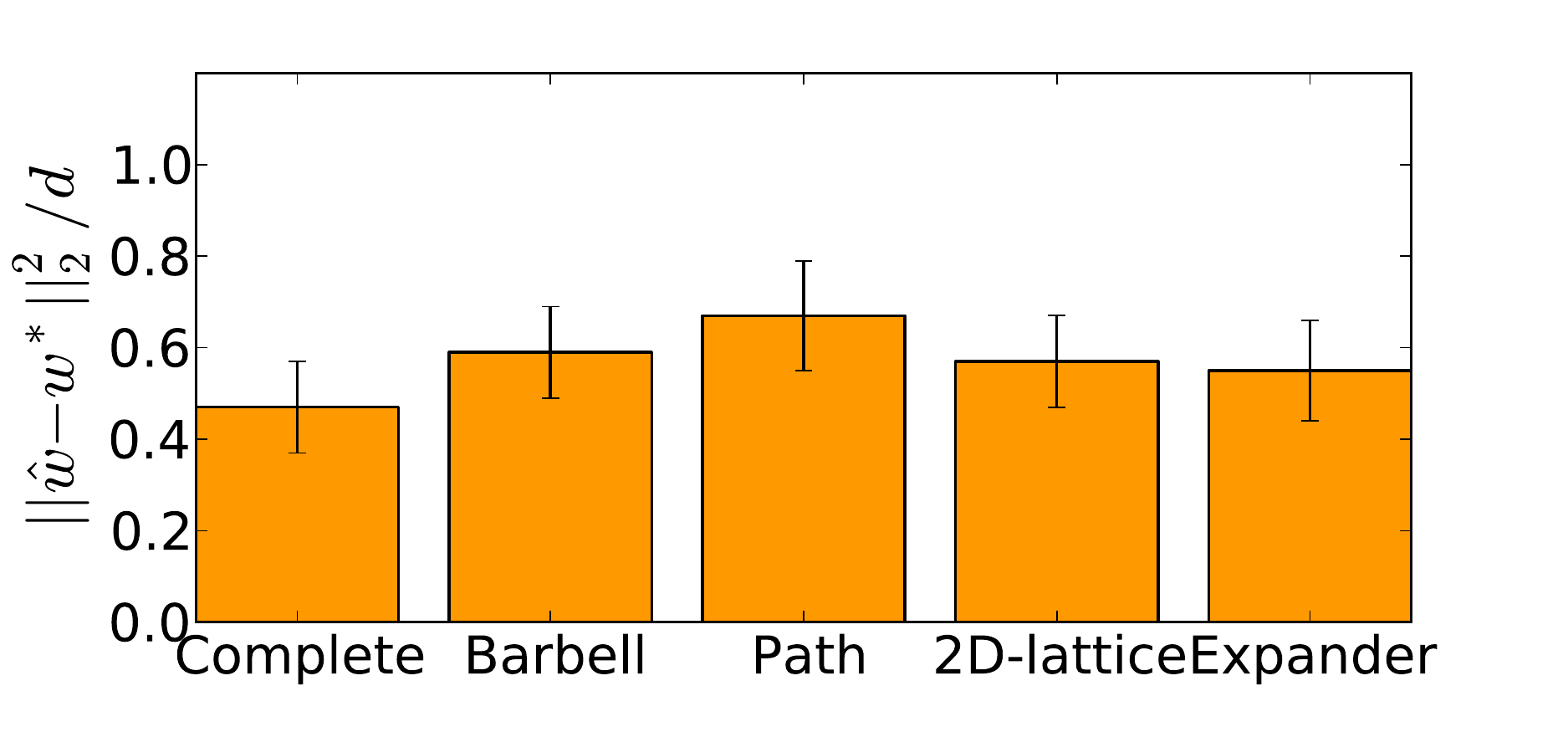}
\caption{City distances}
\end{subfigure}
\caption{Estimation error under different topologies in the experiments conducted on MTurk.}
\label{fig:varyTopology_mturk}
\end{figure*}

\section{Cardinal versus ordinal measurements}
\label{sec:cvo}

In this section, we compare two approaches towards eliciting data: a
score-based ``cardinal'' approach and a comparison-based ``ordinal''
approach. In a cardinal approach, evaluators directly enter numeric
scores as their answers (Figure~\ref{fig:searchRelevance_cardinal}),
while an ordinal approach involves comparing (pairs of) items
(Figure~\ref{fig:searchRelevance_ordinal}).

There are obvious advantages and disadvantages associated with either
approach.  On one hand, the cardinal approach allows for very fine
measurements.  For instance, the cardinal measurements in
Figure~\ref{fig:searchRelevance} can take any value between $0$ and
$100$, whereas an ordinal measurement is binary. One might be tempted
to go even further and argue that ordinal measurements necessarily
give less information, for one can always convert a set of cardinal
measurements into ordinal, simply by ordering the measurements by
value.  If this conversion were valid, the data processing
inequality~\citep{cover2012elements}, would then guarantee that
estimators based on ordinal data can never outperform estimators based
on cardinal data.  However, this conversion assumes that cardinal and
ordinal measurements suffer from the same type of statistical
fluctuation. The following set of experiments show this assumption is
false.


\subsection{Raw data from MTurk}
\label{sec:experiments_mturk_dataprocessing} 

We conducted seven different experiments on MTurk to investigate the
possibility of a ``data-processing inequality'' between the elicited
cardinal and ordinal responses: Are responses elicited in ordinal form
equivalent to data obtained by first eliciting cardinal responses and
then subtracting pairs of items? Our experiments lead us to conclude
that this is generally not the case: converting cardinally collected
data into ordinal (by subtracting pairs of responses) often leads to a
higher amount of noise as compared to that in data that is elicited
directly in ordinal form.

The tasks were selected to have a broad coverage of several important
subjective judgment paradigms such as preference elicitation,
knowledge elicitation, audio and visual perception and skill
utilization.

In addition to the three experiments described in
Section~\ref{sec:experiments_topology_mturk}, we conducted the
following four experiments.
\begin{enumerate}
\item[(d)] \emph{Finding spelling mistakes in text:} The worker had to
  identify the number of words that were misspelled in each paragraph
  shown (Figure~\ref{fig:spelling_cardinal}).
\item[(e)] \emph{Identifying sounds:} The worker was presented with
  audio clips, each of which was the sound of a single key on a
  piano (which corresponds to a single frequency). The worker had to
  estimate the frequency of the sound in each audio clip
  (Figure~\ref{fig:tones_ordinal}).
\item[(f)] \emph{Rating tag-lines for a product:} A product was
  described and tag-lines for this product were shown
  (Figure~\ref{fig:tagline_cardinal}). The worker had to rate each of
  these tag-lines in terms of its originality, clarity and relevance to
  this product.
\item[(g)] \emph{Rating relevance of the results of a search query:}
  Results for the query `Internet' for an image search were
  shown (Figure~\ref{fig:searchRelevance}) and the worker had to rate
  the relevance of these results with respect to the given query.
\end{enumerate}

Note that the data collected for (a)--(c) here was different and
independent of the data collected for these tasks in
Section~\ref{sec:experiments_topology_mturk}.

\begin{figure}[t]
\centering
\begin{subfigure}{0.26\textwidth}
\widgraph{\textwidth}{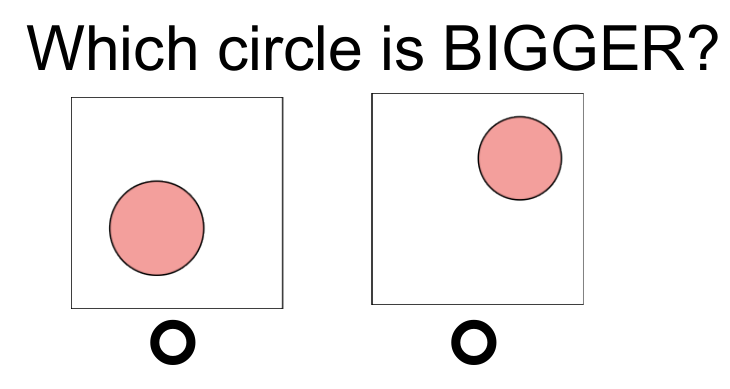}
\caption{}
\label{fig:areaCircle_ordinal}
\end{subfigure}
\quad
\begin{subfigure}{0.24\textwidth}
\widgraph{\textwidth}{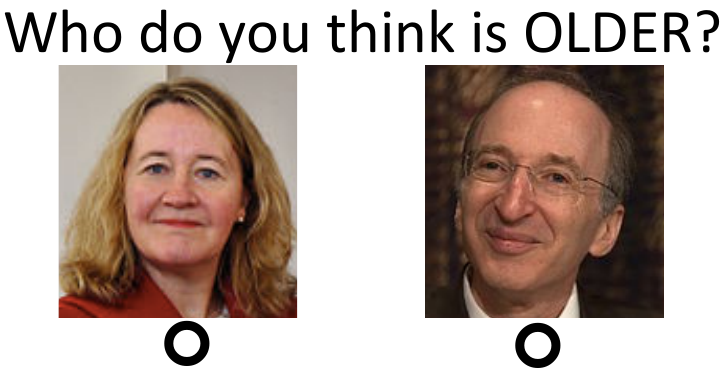}
\caption{}
\label{fig:photoAge_ordinal}
\end{subfigure}
\quad
\begin{subfigure}{0.26\textwidth}
\widgraph{\textwidth}{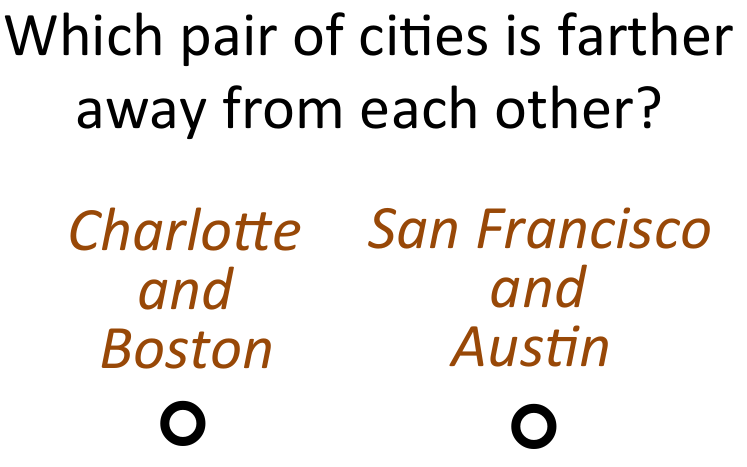}
\caption{}
\label{fig:distanceCities_cardinal}
\end{subfigure}
\\
\begin{subfigure}{0.27\textwidth}
\vspace{-.1cm} \widgraph{\textwidth}{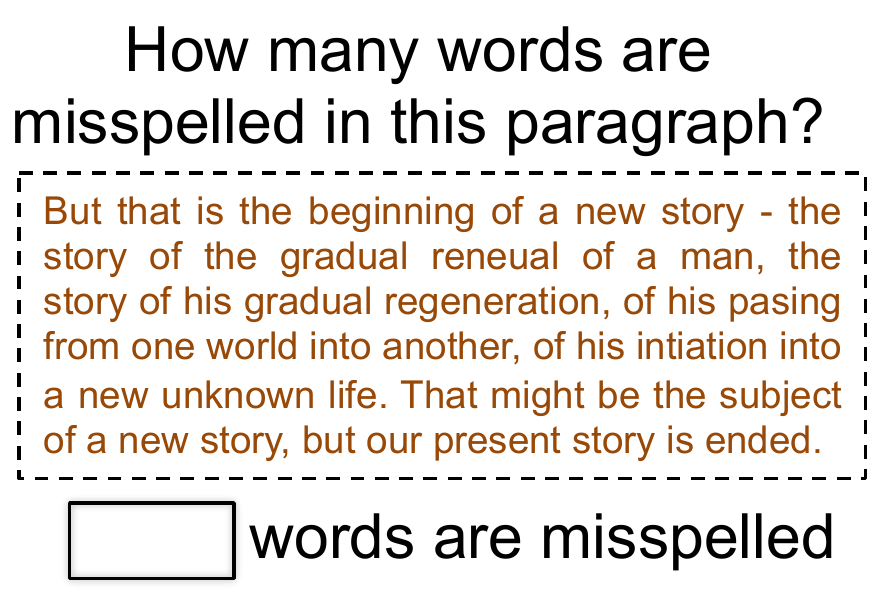}
\caption{}
\label{fig:spelling_cardinal}
\end{subfigure}
\quad
\begin{subfigure}{0.22\textwidth}
\widgraph{\textwidth}{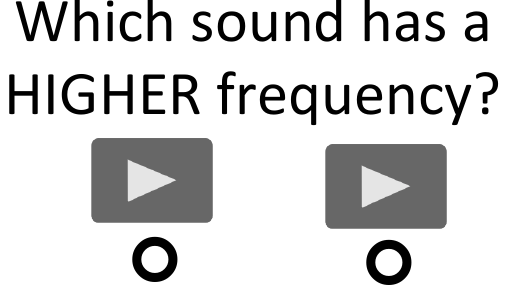}
\caption{}
\label{fig:tones_ordinal}
\end{subfigure}
\quad
\begin{subfigure}{0.27\textwidth}
\widgraph{\textwidth}{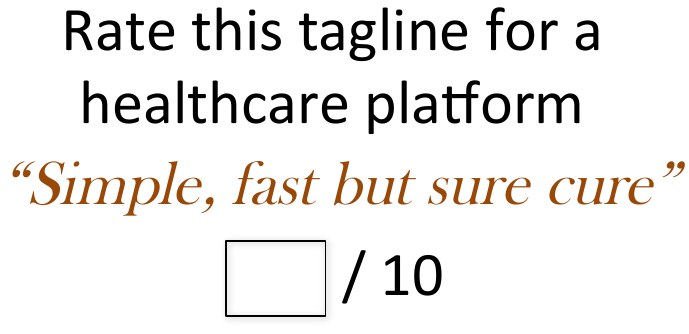}
\caption{}
\label{fig:tagline_cardinal}
\end{subfigure}
\caption{Screenshots of the tasks presented to the subjects. For each
  task, only one version (cardinal or ordinal) is shown here.}
\label{fig:data_processing_experiments}
\end{figure}

\newcommand{\cvoTableRow}[7]{#2 & #4 & #5 & #3 & #6 & #1 & #7}

\begin{table}
\begin{center}
\begin{tabular}{|l|c|c|c|c|c|c|c|}
\hline ~~~~~~Task & \cvoTableRow{Tagline }{ Circle }{Spelling }{Age }{Distance }{Audio }{ Relevance}\\ \hline 
Error in Ordinal & \cvoTableRow{44\% }{ \textbf{6\%} }{ \textbf{40\%} }{ \textbf{13\%} }{ \textbf{17\%} }{ \textbf{20\%} }{ \textbf{31\%} }\\ 
\multicolumn{1}{|r|}{Std. dev.} &  \cvoTableRow{ .47 }{ .23 }{ .49  }{ .33 }{ .38 }{ .40 }{ .44}\\\hline
Error in Cardinal & \cvoTableRow{ \textbf{42\%} }{ 17\% }{ 42\% }{ 17\% }{  20\% }{ 29\% }{ 35\% }\\
\multicolumn{1}{|r|}{Std. dev.} & \cvoTableRow{ .46 }{ .31 }{ .46 }{ .38 }{ .38 }{ .43 }{ .44}\\
\hline 
Time in Ordinal & \cvoTableRow{ \textbf{251s} }{ \textbf{98s}}{ \textbf{316s} }{ \textbf{31s}}{ \textbf{84s} }{ \textbf{66s} }{ \textbf{105s} }\\ 
\multicolumn{1}{|r|}{Std. dev.} & \cvoTableRow{ 28.1 }{ 21.1 }{ 33.2 }{ 14.3 }{ 62.1 }{ 11.1 }{ 13.1}\\\hline
Time in Cardinal & \cvoTableRow{ 342s}{ 181s }{ 525s }{ 70s}{ 144s }{ 134s }{ 185s}\\ 
\multicolumn{1}{|r|}{Std. dev.} & \cvoTableRow{ 44.6 }{ 39.9 }{ 46.0 }{ 33.1 }{ 56.2 }{ 12.4 }{ 28.2}\\
\hline
\end{tabular}
\end{center}
\caption{Comparison of the average amount of error when ordinal data
  is collected directly versus when cardinal data is collected and
  converted to ordinal. Also tabulated is the median time (in seconds)
  taken to complete a task by a subject in either type of task.}
\label{tab:mturk_CVO_raw}
\end{table}

The number of items $\numitems$ in the experiments ranged from $10$ to $25$. For each of the seven experiments, we recruited $100$ workers, and
assigned each worker to either the ordinal or the cardinal
version of the task at random. Upon obtaining the data, we first reduced the
cardinal data obtained from the experiments into ordinal form by
comparing answers given by the subjects to consecutive questions. For
five of the experiments ((a) through (e)), we had access to the
``ground truth'' solutions, using which we computed the fraction of
answers that were incorrect in the ordinal and the
cardinal-converted-to-ordinal data (any tie in the latter case was
counted as half an error). For the two remaining experiments ((f) and
(g)) for which there is no ground truth, we computed the `error' as
the fraction of (ordinal or cardinal-converted-to-ordinal) answers
provided by the subjects that disagreed with each other. It is
important to note that in the experiments in this section, we did
\emph{not} run any estimation procedure on the data: we only measured
the noise in the \emph{raw responses}. The entire data pertaining to
these experiments, including the interface seen by the workers and the data obtained from their work, is available on the first author's website.

The results are summarized in Table~\ref{tab:mturk_CVO_raw}. If the
cardinal measurements could always be converted to ordinal ones with
the same noise level as directly eliciting ordinal responses, then it
would be unlikely for the amount of error in the ordinal setting to be
smaller than that in the cardinal
setting. Table~\ref{tab:mturk_CVO_raw} shows that converting cardinal
data to an ordinal form very often results in a higher (and sometimes
significantly higher) per-sample error in the (raw) responses than
direct elicitation of ordinal evaluations. Such an outcome may be explained by the argument that the
inherent evaluation process in humans is not the same in the cardinal
and ordinal cases: humans do \emph{not} perform an ordinal evaluation
by first performing cardinal evaluations and then comparing
them~\citep{barnett2003modern,stewart2005absolute}. One can also see
from Table~\ref{tab:mturk_CVO_raw} that the amount of time required
for cardinal evaluations was typically (much) higher than for ordinal
evaluations.  One can thus assume that we will typically have the
per-observation error in the ordinal case lower than that in the
cardinal case. In particular, if we consider the \tstone and the \dir
models (introduced in Section~\ref{sec::models}), we can assume that $\noisestd < \noisestdc$.

\subsection{Analytical comparison of Cardinal versus Ordinal}

As discussed earlier, while cardinal measurements allow more
flexibility in the range of responses, ordinal measurements contain a
lower per-sample error. Ordinal measurements have additional benefits
in that they avoid calibration issues that are frequently encountered
in cardinal measurements~\citep{tsukida2011analyze}, such as the
evaluators' inherent (and possibly time-varying) biases, or tendencies
to give inflated or conservative evaluations.  Ordinal measurements
are also recognized to be easier or faster for humans to
make~\citep{barnett2003modern,stewart2005absolute}, allowing for more
evaluations with the same amount of time, effort and cost.

The lack of clarity regarding when to use a cardinal versus an ordinal
approach forms the motivation of this section. Can we make as reliable
estimates from paired comparisons as from numeric scores? How much
lower does the noise have to be for comparative measurements to be
preferred over cardinal measurements? The answers to these questions
will help in determining how responses should be elicited.

In order to compare the cardinal and ordinal methods of data
elicitation, we focus on a setting with evenly budgeted
measurements. In accordance with the fixed-design setup assumed
throughout the paper, we choose the vectors $\diff_i$ a
priori. Suppose that $\numobs$ is large enough, and that in the
ordinal case we compare each pair $\numobs/{\numitems \choose 2}$
times.  In the cardinal case suppose that we evaluate the quality of
each item $\numobs/\numitems$ times.  We consider the Gaussian-noise
models \tstone and \dir  introduced earlier in Section~\ref{sec::models}. In order to capture the fact that the amount
of noise is different in the cardinal and ordinal settings, we will
denote the standard deviation of the noise in the cardinal setting as
$\noisestdc$, and retain our notation of $\noisestd$ for the noise in
the ordinal setting.  In order to bring the two models on the same
footing, we measure the error in terms of the squared $\ell_2$-norm.

Let $\strongcon_G$ and $\cdfparam_G$ denote the parameters
$\strongcon$ and $\cdfparam$ (defined in~\eqref{EqnDefnStrongcon} and~\eqref{EqnDefnCdfparam} respectively) specialized to the Gaussian distribution.
Define $\cvoconst_\ell (\noisestd,\wmax) \defn
\frac{\plaincon_{2\ell}}{\cdfparam_G(\wmax,\noisestd)}$, $\cvoconst_u
(\noisestd,\wmax) \defn \frac{\plaincon_{2u}
  \cdfparam_G(\wmax,\noisestd)}{\strongcon_G(\wmax,\noisestd)}$ and
$\cvoconst (\noisestd,\wmax) \defn \left\lceil \frac{\plaincon_{2}
  \noisestd^2}{\cdfparam_G \wmax^2} \right\rceil$. Observe that
$\cvoconst_\ell$, $\cvoconst_u$ and $\cvoconst$ are independent of the
parameters $\numobs$ and $\numitems$.

With these preliminaries in place, we now compare the minimax error in
the estimation under the cardinal and ordinal settings.

\begin{proposition}
\label{ThmCVO}
Given a sample size $\numobs$ that is a multiple of \mbox{$\numitems
  (\numitems-1) \cvoconst (\noisestd,\wmax) $,} suppose that we
observe each coordinate $\numobs/\numitems$ times under the \dir
model. Then the minimax risk is given by
\begin{subequations}
\begin{align}
\label{EqnDirUniformBounds}
 \inf_{\wthat} \sup_{\wtstar \in \SPECSET} \Exs \Big[ \Lnorm{\wthat - \wtstar}{2}^2
  \Big] = \noisestdc^2 {\frac{\numitems
  }{\numobs}}.
\end{align}    
Similarly, if we observe each pair $\numobs/{\numitems \choose 2}$
times in the \tstone~model, then the minimax risk is sandwiched as
\begin{align}
\label{EqnThurstoneUniformBounds}
 \noisestd^2 {\cvoconst_\ell (\noisestd,\wmax) \frac{\numitems
   }{\numobs}} \leq \inf_{\wthat} \sup_{\wtstar \in \SPECSET} \Exs
 \Big[ \Lnorm{\wthat - \wtstar}{2}^2 \Big] \leq \noisestd^2
     {\cvoconst_u (\noisestd,\wmax) \frac{\numitems}{\numobs}}.
\end{align}
\end{subequations}
\end{proposition}

In the cardinal case, when each coordinate is measured the same number
of times, the \dir~model reduces to the well-studied normal location
model, for which the MLE is known to be the minimax estimator and its
risk is straightforward to characterize (see \citet{tpe} for
instance). In the ordinal case, the result follows from the general
treatment in Section~\ref{SecTheoryComparison}.

Let us now return to the question deciding between the cardinal and
the ordinal methods of data elicitation. Suppose that we believe the
Gaussian-noise models to be reasonably correct, and the
per-observation errors $\noisestd$ and $\noisestdc$ under the two
settings are known or can be separately measured.
Proposition~\ref{ThmCVO} shows that the scaling of the minimax error
in the cardinal and ordinal settings is identical in terms of the
problem parameters $\numobs$ and $\numitems$. As an important
consequence, our result thus allows for the choice to be made based
\emph{only} on the parameters $(\noisestd,\noisestdc,\wmax)$, and
independent of $\numobs$ and $\numitems$: the ordinal approach incurs
a lower minimax error when \mbox{$\cvoconst_u (\noisestd,\wmax)
  \noisestd^2 < \noisestdc^2$} while the cardinal approach is better
off in terms of minimax error whenever \mbox{$\cvoconst_\ell
  (\noisestd,\wmax) \noisestd^2 > \noisestdc^2$}. Establishing the
exact decision boundary would require tightening the constants in the
bounds, a task we leave for future work.

\subsection{Aggregate Estimation Error in Experiments on MTurk}
\label{sec:experiments_mturk_inference}

For the sake of completeness, we also computed the estimation error in
the cardinal and ordinal settings.  We consider data from the three
experiments (c), (d) and (e).\footnote{We restrict attention to these
  three experiments for the following reasons. There is no ground
  truth for experiments (f) and (g). In experiment (a), the size of
  each circle in each question is chosen independently from a
  continuous distribution, making all questions different and
  preventing aggregation. Experiment (b) employs a disconnected
  topology.} We normalize the true vector to have
$\Lnorm{\wtstar}{\infty} = 1$ and set $\wmax=1$.  For each of the
three experiments, we execute $100$ iterations of the following
procedure. Select five workers from the cardinal and five from the
ordinal pool of workers uniformly at random. (The number five is
inspired by practical
systems~\citep{wang2011managing,piech2013tuned}.) We run the
maximum-likelihood estimator of the \dir~model on the data from the
five workers selected from the cardinal pool, and the
maximum-likelihood estimator of the \tstone~model on the data from the
five workers of the ordinal pool.  Note that unlike
Section~\ref{sec:experiments_mturk_dataprocessing}, the cardinal data
here is \emph{not} converted to ordinal.

\begin{table}[h]
\begin{center}
\begin{tabular}{|l|c|c|c|}
\hline Task & Spelling & Distance & Audio \\
 \hline 
$\frac{\Lnormd{\wtstar}{\wthat}{2}^2}{\numitems}$ in Ordinal & {0.358 $\pm$ 0.035} & \textbf{0.168 $\pm$ 0.026} &
\textbf{0.444 $\pm$ 0.055}\\ 
\hline
$\frac{\Lnormd{\wtstar}{\wthat}{2}^2}{\numitems}$ in Cardinal & \textbf{0.350 $\pm$ 0.045} & {0.330 $\pm$ 0.028} &{ 0.508 $\pm$ 0.053} \\
\hline 
Kendall-tau coefficient in Ordinal & \textbf{0.277 $\pm$ 0.049}& \textbf{0.547 $\pm$ 0.034} &
\textbf{0.513 $\pm$ 0.047}\\
\hline
Kendall-tau coefficient in Cardinal & 0.129 $\pm$ 0.046  & 0.085 $\pm$ 0.038 & 0.304 $\pm$ 0.049\\
\hline
\end{tabular}
\end{center}
\caption{Evaluation of the inferred solution from the data received
  from multiple workers.}
\label{tab:mturk_inference}
\end{table}

The results are tabulated in Table~\ref{tab:mturk_inference}. To put
the results in perspective of the rest of the paper, let us also
recall the per-sample errors in these experiments from
Table~\ref{tab:mturk_CVO_raw}. Observe that among these three experiments, the per-sample noise in the cardinal data was closest to that in the ordinal data in the experiment on identifying the number of spelling mistakes. The gap was larger in the two remaining experiments. This fact is reflected in the
results of Table~\ref{tab:mturk_inference} where the estimator on the
cardinal data incurs a lower $\ell_2$-error than the estimator on the ordinal data in the experiment on identifying the number of spelling mistakes, whereas the outcome goes the other way in the two remaining experiments. 
%
%
%
Our theory needs to tighten
the constants in order to address this regime.


\section{Conclusions}

In this paper, we presented topology-aware minimax error bounds under
a broad class of preference-elicitation models. We demonstrated the
utility of these results in guiding the selection of comparisons and
in guiding the choice of the elicitation paradigm (cardinal versus
ordinal) when these options are available.  One potential direction
for future work would be to investigate improved data collection
mechanisms, for instance \emph{adaptive} schemes where we focus our
effort on the most noisy comparisons.  A second direction would be to
characterize the precise thresholds for making the choice between the
cardinal and ordinal approaches.  Finally, the Thurstone and BTL
models are parametric idealizations that have proved useful in a wide
variety of applications. In future work we would like to investigate
more flexible semi-parametric and non-parametric pairwise comparison
models (see, for instance, \citet{chatterjee2014matrix,braverman2008noisy}).


\subsection*{Acknowledgments}  This work was partially supported by
Office of Naval Research MURI grant N00014-11-1-0688, MURI grant 96045-23800, and National
Science Foundation Grants CIF-31712-23800, DMS-1107000 and CIF-81652-23800. The work of N.S. was also partially supported by a Microsoft Research PhD fellowship.


\appendix

\section{Proof of Theorem~\ref{ThmMinimaxL}}
\label{AppThmMinimaxL}

The following two sections prove the lower and upper bounds
(respectively) on the minimax risk of \ord model under the squared
$\Lap$ semi-norm.


\subsection{Lower bound}
\label{AppMinimaxLLower}
Our lower bounds are based on the Fano argument, which is a standard
method in minimax analysis (see for instance \cite{tsybakovbook}). Suppose that our goal is to bound the
minimax risk of estimating a parameter $\genpar$ over an indexed class
of distributions $\Pclass = \{ \mprob_\genpar \, \mid \, \wt \in
\SPECSETgeneral \}$ in the square of a pseudo-metric $\SEMINORM$. Consider a
collection of vectors $\{\packvec{1}, \ldots, \packvec{\packnum} \}$
contained within $\SPECSETgeneral$ such that
\begin{align*}
\min_{\stackrel{j, k \in [\packnum]}{j \neq k}} \SEMINORM
\big(\packvec{j}, \packvec{k} \big) \geq \delta \quad \mbox{and} \quad
\frac{1}{{\packnum \choose 2}} \sum_{\stackrel{j, k \in [\packnum]}{j
    \neq k}} \kl{\mprob_{\packvec{j}}}{\mprob_{\packvec{k}}} \leq
\beta.
\end{align*}
We refer to any such subset as an $(\delta, \beta)$-packing set.

\blem[Pairwise Fano minimax lower bound]
\label{LemGeneralMinimaxLower}
Suppose that we can construct a $(\delta, \beta)$-packing with
cardinality $\packnum$.  Then the minimax risk is lower bounded as
\begin{align}
\label{EqnGeneralMinimaxLower}
\inf_{\wthat} \sup_{\wtstar \in \SPECSETgeneral} \Exs \Big[ \SEMINORM(\wthat, \wtstar)^2
  \Big]
   & \geq
\frac{\delta^2}{2} \Big( 1-\frac{\beta+\log 2}{\log \packnum} \Big).
\end{align}
\elem

In order to apply Lemma~\ref{LemGeneralMinimaxLower}, we need to a
construct a suitable packing set.  Given a scalar $\alpha \in
(0,\frac{1}{4})$  whose value will be specified
later, define the integer
\begin{align}
\label{EqnGVPackSize}
\packnum(\alpha) & \defn \left\lfloor \exp \Big \{ \frac{\numitems}{2}
\big(\log 2 + 2\alpha \log 2\alpha + (1-2\alpha) \log (1-2\alpha)\big)
\Big \} \right \rfloor.
\end{align}
We require the following two auxiliary lemmas:
\blem
\label{LemGilVar}
For any $\alpha \in (0,\frac{1}{4})$, there exists a set of
$\packnum(\alpha)$ binary vectors
$\{\lemmawt^1,\ldots,\lemmawt^{\packnum(\alpha)}\} \subset
\{0,1\}^\numitems$ such that
\begin{subequations}
\begin{align}
\label{EqnPackingDistance}
\alpha \numitems \leq \norm{\lemmawt^j-\lemmawt^k}^2 \leq
\numitems & \qquad \mbox{for all $j \neq k \in [\packnum(\alpha)]$, and}\\
\label{EqnPackingInprod}
\inprod{\unitvec_1}{\lemmawt^j}= 0 & \qquad \mbox{for all $j \in
  [\packnum(\alpha)]$,}
\end{align}
\end{subequations}
where $\unitvec_1$ denotes the first canonical basis vector.
\elem
\noindent This result is a straightforward consequence of the
Gilbert-Varshamov
bound~\citep{gilbert1952comparison,varshamov1957estimate}.


\blem
\label{LemKLUpper}
For any pair of quality score vectors $\packvec{j}$ and $\packvec{k}$,
and for 
\begin{align*}
\cdfparam \defn \frac{\max \limits_{x \in [0,2\wmax/\noisestd]}
  \glmcdf'(x)}{\glmcdf(2\wmax/\noisestd)(1-\glmcdf(2\wmax/\noisestd))},
\end{align*}
we have
\begin{align}
\label{EqnKLUpper}
\kl{\mprob_{\packvec{j}}}{\mprob_{\packvec{k}}} & \leq \frac{\numobs
  \cdfparam }{\noisestd^2} \Lnormsqr{\packvec{j}}{\packvec{k}}{\Lap}.
\end{align}
\elem
\noindent We prove this lemma at the end of this section. \\

\vspace*{.05in}

Taking these two lemmas as given for the moment, consider the set
$\{\lemmawt^1,\ldots,\lemmawt^{\packnum(\alpha)}\}$ of
$\numitems$-dimensional binary vectors given by
Lemma~\ref{LemGilVar}. 
The Laplacian $\Lap$ of the comparison graph is symmetric and
positive-semidefinite, and so has a diagonalization of the form $\Lap = U^T
\Lambda U$ where $U \in \real^{\numitems \times \numitems}$ is an
orthonormal matrix, and $\Lambda$ is a diagonal matrix of nonnegative
eigenvalues.

Letting $\LamTil$ denote the Moore-Penrose pseudo-inverse of $\Lambda$,
consider the collection $\{\wt^1, \ldots, \wt^{\packnum(\alpha)} \}$
of vectors given by $\packvec{j} \defn \frac{\delta}{\sqrt{\numitems}}
U^T \sqrt{\LamTil} \lemmawt^j$ for each $j \in [\packnum(\alpha)]$.
Since $\ones \in \nullspace(\Lap)$, we are guaranteed that
$\inprod{\ones}{\packvec{j}} = \frac{\delta}{\sqrt{\numitems}} \ones^T
U^T \sqrt{\LamTil} \lemmawt^j \; = 0$.  On the other hand,
\begin{align*}
(\packvec{j}-\packvec{k})^T \Lap (\packvec{j} - \packvec{k}) & \leq
  \frac{\delta^2}{\numitems} (\lemmawt^j-\lemmawt^k)^T \sqrt{\LamTil}
  U \Lap U^T \sqrt{\LamTil} (\lemmawt^j - \lemmawt^k) \\
& = \frac{\delta^2}{\numitems} (\lemmawt^j - \lemmawt^k)
  \sqrt{\LamTil} \Lambda \sqrt{\LamTil} (\lemmawt^j - \lemmawt^k)\\
& = \frac{\delta^2}{\numitems} \|\lemmawt^j - \lemmawt^k\|_2^2,
\end{align*}
Here the last step makes use of the fact that the first coordinate of
each vector $\lemmawt^j$ and $\lemmawt^k$ is zero.  It follows that $
\packdmin \delta^2 \leq \LAPNORM{\wt^{j} - \wt^{k}}^2 \leq \delta^2$.

Setting $\delta^2 \defn 0.01 \frac{\noisestd^2 \numitems}{\numobs
  \cdfparam}$, we find that
\begin{align*}
\|\wt^j\|_\infty \; \leq \; \frac{\delta}{\sqrt{\numitems}}
\|\sqrt{\LamTil} \lemmawt^j\|_2 \; \stackrel{(i)}{\leq} \;
\frac{\delta}{\sqrt{\numitems}} \sqrt{\trace{\LamTil}} &
\stackrel{(ii)}{=} \; \frac{\delta}{\sqrt{\numitems}}
\sqrt{\trace{\LapInv}} \; \stackrel{(iii)}{\leq} \; \wmax,
\end{align*}
where inequality (i) follows from the fact that $\lemmawt^j$ has
entries in $\{0, 1\}$; equation (ii) follows since $\LapInv = U^T
\LamTil U$ by definition; and inequality (iii) follows from our choice
of $\delta$ and our assumption \mbox{$\numobs \geq \frac{\plaincon
    \noisestd^2 \trace{\LapInv}}{\cdfparam \wmax^2}$} on the sample
size with $\plaincon = 0.01$. We have thus verified that each vector
$\packvec{j}$ also satisfies the boundedness constraint
$\|\packvec{j}\|_\infty \leq B$ required for membership in
$\Wclass_\wmax$. 
Finally, observe
that
\begin{align*}
\max_{j \neq k} \kl{\mprob_{\packvec{j}}}{\mprob_{\packvec{k}}} \leq
\frac{\numobs \cdfparam \delta^2}{\noisestd^2}, \quad \mbox{and} \quad
\min_{j \neq k} \LAPNORM{\packvec{j} - \packvec{k}}^2 \geq \alpha
\delta^2.
\end{align*}
We have thus constructed a suitable packing set for applying
Lemma~\ref{LemGeneralMinimaxLower}, which yields the lower bound
\begin{align*}
\Exs[\LAPNORM{\wthat - \wtstar}^2 ] & \geq \frac{\alpha}{2} \delta^2
\Big \{ 1 - \frac{\frac{\delta^2 \cdfparam \numobs}{\noisestd^2} + \log 2}{\log \packnum(\alpha)}\Big \}.
\end{align*}
Substituting our choice of $\delta$ and setting $\alpha=0.01$ proves
the claim for $\numitems > 9$.

In order to handle the case $\numitems \leq 9$, we consider the set of
the three $\numitems$-length vectors given by \mbox{$\lemmawt^1 =
  [0~~\cdots~~0~~-1]$,} \mbox{$\lemmawt^2 = [0~~\cdots~~0~~1]$} and
\mbox{$\lemmawt^3 = [0~~\cdots~~0~~0]$.} Construct the packing set
$\{\wt^1,\wt^2,\wt^3\}$ from these three vectors
$\{\lemmawt^1,\lemmawt^2,\lemmawt^3\}$ as done above for the case of
$\numitems>9$. From the calculations made for the general case above,
we have for all pairs $\min_{j \neq k} \Lnorm{\wt^j-\wt^k}{\Lap}^2
\geq \frac{\delta^2}{9}$ and $\max_{j,k} \LAPNORM{\packvec{j} -
  \packvec{k}}^2 \leq 4\delta^2$, and as a result $\max_{j,k}
\kl{\mprob_{\packvec{j}}}{\mprob_{\packvec{k}}} \leq \frac{4\numobs
  \cdfparam \delta^2}{\noisestd^2}$. Choosing $\delta^2 = \frac{
  \noisestd^2 \log2}{8\numobs \cdfparam}$ and applying
Lemma~\ref{LemGeneralMinimaxLower} proves the theorem.\\

\vspace*{.05in}

\noindent The only remaining detail is to prove
Lemma~\ref{LemKLUpper}.

\paragraph{Proof of Lemma~\ref{LemKLUpper}:}
For any pair of quality score vectors $\packvec{j}$ and $\packvec{k}$,
the KL divergence between the distributions $\mprob_{\packvec{j}}$ and
$\mprob_{\packvec{k}}$ is given by
\begin{align*}
\kl{\mprob_{\packvec{j}}}{\mprob_{\packvec{k}}} & =
\sum_{i=1}^{\numobs} \glmcdf(\inprod{\packvec{j}}{\diff_i}/\noisestd)
\log \frac{\glmcdf(\inprod{\packvec{j}}{
    \diff_i}/\noisestd)}{\glmcdf(\inprod{\packvec{k}}{\diff_i}/\noisestd)}
+ (1-\glmcdf(\inprod{\packvec{j}}{\diff_i}/\noisestd)) \log
\frac{1-\glmcdf(\inprod{\packvec{j}}{\diff_i}/\noisestd)}{1-\glmcdf(\inprod{\packvec{k}}
  {\diff_i}/\noisestd)}.
\end{align*}
For any $a,b \in (0,1)$, we have the elementary inequality $a \log
\frac{a}{b} \leq (a-b)\frac{a}{b}$. Applying this inequality to our
expression above gives
\begin{align*}  
\kl{\mprob_{\packvec{j}}}{\mprob_{\packvec{k}}} & \leq
\sum_{i=1}^{\numobs} (\glmcdf(\inprod{\packvec{j}}{\diff_i}/\noisestd)
- \glmcdf(\inprod{\packvec{k}}{\diff_i}/\noisestd))
\frac{\glmcdf(\inprod{\packvec{j}}
  {\diff_i}/\noisestd)}{\glmcdf(\inprod{\packvec{k}}{\diff_i}/\noisestd)}
\\ & \quad \quad \qquad \qquad - \Big \{
\glmcdf(\inprod{\packvec{j}}{\diff_i}/\noisestd)) - \glmcdf(
\inprod{\packvec{k}}{\diff_i}/\noisestd) \Big \} \frac{1-\glmcdf(
  \inprod{\packvec{j}}{\diff_i}/\noisestd)}{1-\glmcdf(\inprod{\packvec{k}}{\diff_i}/\noisestd)}\\ &
\leq \sum_{i=1}^{\numobs}
\frac{(\glmcdf(\inprod{\packvec{j}}{\diff_i}/\noisestd) -
  \glmcdf(\inprod{\packvec{k}}{\diff_i}/\noisestd))^2}{\glmcdf(\inprod{\packvec{k}}{\diff_i}/\noisestd)
  (1-\glmcdf(\inprod{\packvec{k}}{\diff_i}/\noisestd))}.
  \end{align*}
Since $\max \{ \|\wt^j\|_\infty, \, \|\wt^k\|_\infty \} \leq \wmax$,
and since $\glmcdf$ is a non-decreasing function, we have
\begin{align*}
\kl{\mprob_{\packvec{j}}}{\mprob_{\packvec{k}}} & \leq
\sum_{i=1}^{\numobs}
\frac{(\glmcdf(\inprod{\packvec{j}}{\diff_i}/\noisestd) -
  \glmcdf(\inprod{\packvec{k}}{\diff_i}/\noisestd))^2}{\glmcdf(2\wmax/\noisestd)
  (1-\glmcdf(2\wmax/\noisestd))}.
\end{align*}
Finally, applying the mean value theorem and recalling the definition
of $\cdfparam$~(from~\eqref{EqnDefnCdfparam}) yields
\begin{align*}
\kl{\mprob_{\packvec{j}}}{\mprob_{\packvec{k}}} & \leq
\sum_{i=1}^{\numobs} \cdfparam {(
  \inprod{\packvec{j}}{\diff_i}/\noisestd -
  \inprod{\packvec{k}}{\diff_i}/\noisestd)^2} \; = \; \frac{\numobs
  \cdfparam }{\noisestd^2} \Lnormsqr{\packvec{j}}{\packvec{k}}{\Lap},
\end{align*}
as claimed.


\subsection{Upper bound}

For the \ord model, the MLE is given by $\hat{w} \in \arg \min
\limits_{\wt \in \SPECSET} \loss(\wt)$, where
\begin{subequations}
\begin{align}
\label{EqnOrdinalLikelihood}
\loss(\wt) & = -\frac{1}{\numobs} \sum_{i=1}^{\numobs} \Big \{
\Ind[\obs_i=1] \log \glmcdf \big(
\frac{\inprod{\diff_i}{\wt}}{\noisestd} \big) + \Ind[\obs_i=-1] \log
\Big( 1-\glmcdf \big( \frac{\inprod{\diff_i}{\wt}}{\noisestd}
\big)\Big) \Big \}, \qquad \mbox{and} \\
\label{EqnDefnWclassWmaxGLM}
\SPECSET & \defn \big \{ \wt \in \real^\numitems \, \mid \,
\inprod{\ones}{\wt} = 0, \quad \mbox{and} \quad \|\wt\|_\infty \leq
\wmax \big \}.
\end{align}
\end{subequations}
Our goal is to bound the estimation error 
of the MLE in the squared semi-norm
$\LAPNORM{v}^2 = {v^T \Lap v}$. 

For the purposes of this proof (as well as subsequent ones), let us
state and prove an auxiliary lemma that applies more generally to
$M$-estimators that are based on minimizing an arbitrary convex and
differentiable function over some subset $\SPECSETgeneral$ of the set $\SPECSETinf \defn
\{ \wt \in \real^\numitems \, \mid \, \inprod{1}{\wt} = 0 \}$.  The
MLE under consideration here is a special case. This
lemma requires that $\loss$ is differentiable and strongly
convex at $\wtstar$ with respect to the semi-norm $\LAPNORM{\cdot}$,
meaning that there is some constant $\lemstrongcon > 0$ such that
\begin{align}
\label{EqnStrongConvexity}
\loss(\wtstar + \Delta) - \loss(\wtstar) - \inprod{\nabla
  \loss(\wtstar)}{\Delta} & \geq \lemstrongcon \LAPNORM{\Delta}^2
\end{align}
for all perturbations $\Delta \in \real^\numitems$ such that $(\wtstar
+ \Delta) \in \SPECSETgeneral$.
Finally, it is also convenient to introduce
the semi-norm $\LAPNORMINV{u} = \sqrt{u^T \LapInv u}$, where $\LapInv$
is the Moore-Penrose pseudo-inverse of $\Lap$.
\blem[Upper bound for $M$-estimators]
\label{LemMestUpper}
Consider the $M$-estimator
\begin{align}
\label{EqnGeneralMest}
\wthat & \in \arg \min_{\wt \in \SPECSETgeneral} \loss(\wt), \qquad
\mbox{where $\SPECSETgeneral$ is any subset of $\SPECSETinf$,}
\end{align}
and $\loss$ is a differentiable cost function satisfying the
$\lemstrongcon$-strong convexity condition~\eqref{EqnStrongConvexity}
at some $\wtstar \in \SPECSETgeneral$.  Then
\begin{align}
\label{EqnMestUpper}
\LAPNORM{\wthat - \wtstar} & \leq \frac{1}{\lemstrongcon} \LAPNORMINV{\nabla
  \loss(\wtstar)}.
\end{align}
\elem

\begin{proof}
Since $\wthat$ and $\wtstar$ are optimal and feasible, respectively,
for the original optimization problem, we have $\loss(\wthat) \leq
\loss(\wtstar)$.  Defining the error vector $\Delta = \wthat -
\wtstar$, adding and subtracting the quantity $\inprod{\nabla
  \loss(\wtstar)}{\Delta}$ yields the bound
\begin{align*}
\loss(\wtstar + \Delta) - \loss(\wtstar) - \inprod{\nabla
  \loss(\wtstar)}{\Delta} & \leq - \inprod{\nabla
  \loss(\wtstar)}{\Delta}.
\end{align*}
By the $\lemstrongcon$-convexity condition, the left-hand side is lower
bounded by $\lemstrongcon \LAPNORM{\Delta}^2$.  As for the right-hand
side, note that $\Delta$ satisfies the constraint
$\inprod{\ones}{\Delta} = 0$, and thus is orthogonal to the nullspace
of the Laplacian matrix $\Lap$.  Therefore, by
Lemma~\ref{LemPseudoCauchySchwarz} (in Appendix~\ref{AppLaplacian}), we have $|\inprod{\nabla
  \loss(\wtstar)}{\Delta}| \leq \LAPNORMINV{\nabla \loss(\wtstar)} \;
\LAPNORM{\Delta}$.  Combining the pieces yields the claimed
inequality~\eqref{EqnMestUpper}.
\end{proof}

In order to apply Lemma~\ref{LemMestUpper} to the MLE for the \ord
model, we need to verify that the negative log
likelihood~\eqref{EqnOrdinalLikelihood} satisfies the strong convexity
condition, and we need to bound the random variable
$\LAPNORMINV{\nabla \loss(\wtstar)}$ defined in the dual norm
$\LAPNORMINV{\cdot}$.

\paragraph{Verifying strong convexity:}
By chain rule, the Hessian of $\ell$ is given by
\begin{align*}
\nabla^2 \loss(\wt) = \frac{1}{\numobs \noisestd^2}
\sum_{i=1}^{\numobs} \Big \{ \Ind[\obs_i=1] T_{i1} + \Ind[\obs_i=-1]
T_{i2} \Big \} \; \diff_i \diff_i^T,
\end{align*}
where
\begin{align*}
T_{i1} & \defn \frac{\glmcdf'(\frac{\inprod{\wt}{\diff_i}}{\noisestd})^2 -
  \glmcdf( \frac{\inprod{\wt}{\diff_i}}{\noisestd}) \glmcdf''( \frac{\inprod{\wt}{\diff_i}}{\noisestd}) }{\glmcdf( \frac{\inprod{\wt}{\diff_i}}{\noisestd})^2}, \quad \mbox{and} \quad
T_{i2} \defn
\frac{\glmcdf'(\frac{\inprod{\wt}{\diff_i}}{\noisestd})^2+(1-\glmcdf(\frac{\inprod{\wt}{\diff_i}}{\noisestd})) \glmcdf''(
  \frac{\inprod{\wt}{\diff_i}}{\noisestd})}{(1-\glmcdf( \frac{\inprod{\wt}{\diff_i}}{\noisestd}))^2}.
\end{align*}
Observe that the term $T_{i1}$ is simply the second derivative of
$\log \glmcdf$ evaluated at $\frac{\inprod{\wt}{\diff_i}}{\noisestd}$, and
hence the strong log-concavity of $\glmcdf$ implies $T_{i1} \geq
\strongcon$.  On the other hand, the term $T_{i2}$ is the second
derivative of $\log (1-\glmcdf)$. Since $\glmcdf(-x) = 1 - \glmcdf(x)$
for all $x$, it follows that the function $x \mapsto 1-\glmcdf(x)$ is
also strongly log-concave with parameter $\strongcon$ and hence
$T_{i2} \geq \strongcon$.  Putting together the pieces, we conclude
that
\begin{align*}
 v^T \nabla^2 \loss (\wt) v \geq \frac{\strongcon}{\numobs
   \noisestd^2} \|X v\|_2^2 \qquad \mbox{for all $v, \wt \in
   \SPECSET$},
\end{align*}
where $\Xmat \in \real^{\numobs \times \numitems}$ has the
differencing vector $x_i \in \real^\numitems$ as its $i^{th}$ row.

Thus, if we introduce the error vector $\Delta \defn \wthat -
\wtstar$, then we may conclude that
\begin{align*}
\loss(\wtstar + \Delta) - \loss( \wtstar ) - \inprod{\nabla \loss
  (\wtstar)}{ \Delta} & \geq \frac{\strongcon}{\numobs \noisestd^2}
\|X \Delta\|_2^2 \; = \; \frac{\strongcon}{ \noisestd^2}
\LAPNORM{\Delta}^2,
\end{align*}
showing that $\loss$ is strongly convex around $\wtstar$ with
parameter $\lemstrongcon = \frac{\strongcon}{\noisestd^2}$.  An
application of Lemma~\ref{LemMestUpper} then gives $\LAPNORM{\Delta}^2
\leq \frac{\noisestd^4}{\strongcon^2} \Lnorm{\nabla \loss
  (\wtstar)}{\LapInv}^2$.

\paragraph{Bounding the dual norm:}
In order to obtain a concrete bound, it remains to control the
quantity $\nabla \loss (\wtstar)^T \LapInv \nabla\loss (\wtstar)$.
Observe that the gradient takes the form
\begin{align*}
\nabla \loss (\wtstar) & = \frac{-1}{\numobs \noisestd}
\sum_{i=1}^{\numobs} \big[ \Ind[\obs_i = 1] \frac{\myglmpdf
    (\inprod{\wtstar}{
      \diff_i}/\noisestd)}{\glmcdf(\inprod{\wtstar}{\diff_i} /
    \noisestd)} - \Ind[\obs_i=-1] \frac{\myglmpdf(
    \inprod{\wtstar}{\diff_i}/\noisestd)}{1 -
    \glmcdf(\inprod{\wtstar}{\diff_i}/\noisestd)}\big] \diff_i.
\end{align*}
Define a random vector $V \in \real^\numobs$ with independent
components as
\begin{align*}
V_i = \begin{cases} \frac{\myglmpdf( \inprod{\wtstar}{\diff_i} /
    \noisestd)} {\glmcdf( \inprod{\wtstar}{\diff_i}/\noisestd)} &
  \mbox{w.p. \quad}\glmcdf(\inprod{\wtstar}{\diff_i} / \noisestd) \\
\frac{-\myglmpdf(\inprod{\wtstar}{\diff_i} / \noisestd)}{1 -
  \glmcdf(\inprod{\wtstar}{\diff_i}/\noisestd)} & \mbox{w.p. \quad} 1-
\glmcdf ( \inprod{\wtstar}{\diff_i}/\noisestd).
\end{cases}
\end{align*}
With this notation, we have $\nabla \loss(\wtstar) = -
\frac{1}{\numobs \noisestd} \, \Xmat^T V$. One can verify that
$\Exs[V] = 0$ and
\begin{align}
\label{EqnOrdUpperL0}
| V_i| \leq \sup_{z \in [-2\wmax/\noisestd, 2\wmax/\noisestd]} \max
\Big\{ \frac{\myglmpdf(z)}{\glmcdf(z)} ,
\frac{\myglmpdf(z)}{1-\glmcdf(z)} \Big\} & \leq {\sup_{z
    \in [-2\wmax/\noisestd, 2\wmax/\noisestd]}
  \frac{\myglmpdf(z)}{\glmcdf(z) (1-\glmcdf(z))}} \leq {\cdfparam},
\end{align}
where $\cdfparam$ is as defined in~\eqref{EqnDefnCdfparam}.
Defining the $\numobs$-dimensional square matrix $M \defn
\frac{\noisestd^2}{\strongcon^2 \numobs^2} \diffmx \LapInv \diffmx^T$,
our definitions and previous bounds imply that $\LAPNORM{\Delta}^2
\leq V^T M V$.

Consequently, our problem has been reduced to controlling the
fluctuations of the quadratic form $V^T M V$; in order to do so, we
apply the Hanson-Wright inequality (see Lemma~\ref{LemHansonWright} in
Appendix~\ref{AppTailBounds}).  A straightforward calculation yields
\begin{align*}
\fronorm{M}^2 = (\numitems-1) \frac{\noisestd^4}{\strongcon^4
  \numobs^2} \quad \mbox{and} \quad \opnorm{M} =
\frac{\noisestd^2}{\strongcon^2 \numobs},
\end{align*}
where we have used the fact that $\Lap = \frac{1}{\numobs} \diffmx^T
\diffmx$.  Moreover, since the components of $V$ are independent and of zero mean, a straightforward calculation yields that $\Exs[V^T M V] \leq \Exs[\Lnorm{V}{\infty}^2 \trace{M}] \leq \frac{\cdfparam^2 \noisestd^2 \numitems}{\strongcon^2 \numobs}$.
  
Since $|V_i| \leq \cdfparam$, the variables are
$\cdfparam$-sub-Gaussian, and hence the Hanson-Wright inequality
implies that
\begin{align*}
\mprob \Big[ V^T M V - \frac{\cdfparam^2 \noisestd^2
    \numitems}{\strongcon^2 \numobs} > t \Big] \leq 2 \mbox{exp}\big(-c
  \min\{ \frac{t^2 \strongcon^4 \numobs^2}{\cdfparam^4 (\numitems-1)
    \noisestd^4}, \frac{ t \strongcon^2 \numobs}{\cdfparam^2
    \noisestd^2} \} \big) \quad\mbox{for all $t>0$}.
\end{align*} 
Consequently, after some simple algebra, we conclude that
\begin{align*}
\mprob \Big( \Lnorm{\Delta}{\Lap}^2 > t \frac{\plaincon \cdfparam^2
  \noisestd^2 }{\strongcon^2} \frac{\numitems}{ \numobs} \Big) \leq
e^{-t} \quad\mbox{for all $t \geq 1$},
\end{align*} 
for some universal constant $\plaincon$.  Integrating this tail
bound yields the bound on the expectation.


\section{Proof of Theorem~\ref{ThmMinimax2}}
\label{AppThmMinimax2}

The following two sections prove the upper and lower bounds
(respectively) on the minimax risk in the squared Euclidean norm for
\ord model.  We prove the lower bound in two parts corresponding to
the two components of the ``$\max$'' in the statement of the theorem.


\subsection{Upper bound}

The proof of the upper bound under the Euclidean norm follows directly
from the upper bound under the $\Lap$ semi-norm proved in
Theorem~\ref{ThmMinimaxL}. From the setting described in
Section~\ref{SecProblem}, we have that the nullspace of the matrix
$\Lap$ is given by the span of the all ones vector. Furthermore, we
have $\inprod{\wtstar - \wthat}{\ones}=0$, and
$\Lnorm{\wtstar-\wthat}{\Lap}^2 \geq \eigenvalue{2}{\Lap}
\Lnorm{\wtstar-\wthat}{2}^2$.  Substituting this inequality into the
upper bound~\eqref{EqnLUpperBound} gives the desired result.


\subsection{Lower bound: Part I}
\label{AppMinimax2Lower1}

Since the Laplacian $\Lap$ of the comparison graph is symmetric and
positive-semidefinite. By diagonalization, we can
write $\Lap = U^T \Lambda U$ where $U \in \real^{\numitems \times
  \numitems}$ is an orthonormal matrix, and $\Lambda$ is a diagonal
matrix of nonnegative eigenvalues with $\Lambda_{jj} =
\eigenvalue{j}{\Lap}$.

We first use the Fano method (Lemma~\ref{LemGeneralMinimaxLower}) to
prove that the minimax risk is lower bounded as $\plaincon \noisestd^2
\frac{\numitems^2}{\numobs}$.  For scalars $\alpha \in
(0,\frac{1}{4})$ and $\delta > 0$ whose values will be specified
later, recall the set
$\{\lemmawt^1,\ldots,\lemmawt^{\packnum(\alpha)}\}$ of vectors in the
Boolean hypercube $\{0,1\}^\numitems$ given by
Lemma~\ref{LemGilVar}. We then define a second set $\{\packvec{j}, j
\in [M(\alpha)] \}$ via $\packvec{j} \defn
\frac{\delta}{\sqrt{\numitems}} U^T P \lemmawt^j$, where $P$ is a
permutation matrix to be specified momentarily.  At this point, the
only constraint imposed on $P$ is that it keeps the first coordinate
constant.  By construction, for each $j \neq k$, we have
$\Lnormd{\packvec{j}}{\packvec{k}}{2}^2 = \frac{\delta^2}{\numitems}
\Lnormd{\lemmawt^j}{\lemmawt^k}{2}^2 \geq \alpha \delta^2$, where the
final inequality follows from the fact that the set
$\{\lemmawt^1,\ldots,\lemmawt^{\packnum(\alpha)}\}$ comprises binary
vectors with a minimum Hamming distance at least $\alpha \numitems$.

Consider any distinct $j, k \in [\packnum(\alpha)]$. Then, for some
$\{i_1,\ldots,i_r\} \subseteq \{2,\ldots,\numitems\}$ with $\alpha
\numitems \leq r \leq \numitems$, it must be that
\begin{align*}
\Lnormd{\packvec{j}}{\packvec{k}}{\Lap}^2 = \frac{\delta^2}{\numitems}
\Lnormd{U^T P \lemmawt^j}{U^T P \lemmawt^k}{\Lap}^2 =
\frac{\delta^2}{\numitems} \Lnorm{\lemmawt^j - \lemmawt^k}{\Lambda}^2
= \frac{\delta^2}{\numitems} \sum_{m=1}^{r} \eigenvalue{i_m}{\Lap}.
\end{align*}
It follows that for some non-negative numbers $a_2,\ldots,a_\numitems$
such that $\alpha \numitems \leq \sum_{i=2}^{\numitems} a_i \leq
\numitems$,
\begin{align*}
\frac{1}{{\packnum(\alpha) \choose 2}} \sum_{j \neq k}
\Lnormd{\packvec{j}}{\packvec{k}}{\Lap}^2 = \frac{\delta^2}{\numitems}
\sum_{i=2}^{\numitems} a_i \eigenvalue{i}{\Lap}.
\end{align*}
We choose the permutation matrix $P$ such that the last
$(\numitems-1)$ coordinates are permuted to have $a_2 \geq \cdots \geq
a_\numitems$ and the $\numitems^{\rm th}$ coordinate remains
fixed. With this choice, we get
\begin{align*}
\frac{1}{{\packnum(\alpha) \choose 2}} \sum_{j \neq k}
\Lnormd{\wt^j}{\wt^k}{\Lap}^2 \leq \frac{\delta^2}{\numitems}
\frac{\numitems}{\numitems-1} \trace{\Lap} \leq \frac{2
  \delta^2}{\numitems} \trace{\Lap}.
\end{align*}
Lemma~\eqref{LemLapProperties} (Appendix~\ref{AppLaplacian}) gives the trace constraint $\trace{\Lap} = 2$, which in turn guarantees that
$\frac{1}{{\packnum(\alpha) \choose 2}} \sum_{j \neq k}
\Lnormd{\wt^j}{\wt^k}{\Lap}^2 \leq \frac{4 \delta^2}{\numitems}$. For
the choice of $P$ specified above, we have for every $j \in
[\packnum(\alpha)]$,
\begin{align*}
\inprod{\ones}{\packvec{j}} = \frac{\delta}{\sqrt{\numitems}}
\unitvec_{1}^T { P \lemmawt^j} = \unitvec_{1}^T { \lemmawt^j } = 0,
\end{align*}
where the final equation employed the property~\eqref{EqnPackingInprod}.

Setting $\delta^2 = 0.01 \frac{\noisestd^2 \numitems^2}{4 \numobs
  \cdfparam}$, we have $\|\wt^j\|_\infty \; \leq \;
\frac{\delta}{\sqrt{\numitems}} \|\lemmawt^j\|_2 \;
\stackrel{(i)}{\leq} \; \delta \; \stackrel{(ii)}{\leq} \; \wmax$,
where inequality (i) follows from the fact that $\lemmawt^j$ has
entries in $\{0, 1\}$; inequality (ii) follows from our
choice of $\delta$ and our
assumption \mbox{$\numobs \geq \frac{\plaincon \noisestd^2
    \trace{\LapInv}}{\cdfparam \wmax^2}$} on the sample size with
$\plaincon = 0.002$, where Lemma~\ref{LemLapProperties} guarantees
\mbox{$\numobs \geq \frac{\plaincon \noisestd^2 \numitems^2}{4
    \cdfparam \wmax^2}$}. We have thus verified that each vector
$\packvec{j}$ also satisfies the boundedness constraint
$\|\packvec{j}\|_\infty \leq B$ required for membership in
$\Wclass_\wmax$.

From the proof of Theorem~\ref{ThmMinimaxL}, we have that for any
distinct $\kl{\mprob_{\packvec{j}}}{\mprob_{\packvec{k}}} \leq
\frac{\numobs \cdfparam }{\noisestd^2}
\Lnormd{\packvec{j}}{\packvec{k}}{\Lap}^2$, and hence
\begin{align*}
\frac{1}{{\packnum(\alpha) \choose 2}} \sum_{j \neq k}
\kl{\mprob_{\packvec{j}}}{\mprob_{\packvec{k}}} \leq \frac{\numobs
  \cdfparam }{\noisestd^2} \frac{4 \delta^2}{\numitems} = 0.01 \,
\numitems,
\end{align*}
where we have substituted our previous choice of $\delta$.

Applying Lemma~\ref{LemGeneralMinimaxLower} with the packing set
$\{\packvec{1},\ldots,\packvec{\packnum(\alpha)}\}$ gives
\begin{align*}
\MiniMax_\numobs \big(\theta(\Pclass); \SEMINORM \big) & \geq
\frac{\alpha \delta^2}{2} \big( 1- \frac{0.01 \numitems +\log 2}{\log
  \packnum(\alpha)} \big).
\end{align*}
Substituting our choice of$\delta$ and setting $\alpha=0.01$ proves
the claim for $\numitems > 9$.

For the case of $\numitems \leq 9$, consider the set of the three
$\numitems$-length vectors $\lemmawt^1 = [0~~\cdots~~0~~-1]$,
$\lemmawt^2 = [0~~\cdots~~0~~1]$ and $\lemmawt^3 =
[0~~\cdots~~0~~0]$. Construct the packing set $\{\wt^1,\wt^2,\wt^3\}$
from these three vectors $\{\lemmawt^1,\lemmawt^2,\lemmawt^3\}$ as
done above for the case of $\numitems>9$. From the calculations made
for the general case above, we have for all pairs $\min_{j \neq k}
\Lnorm{\wt^j-\wt^k}{2}^2 \geq \frac{\delta^2}{9}$ and $\max_{j,k}
\LAPNORM{\packvec{j} - \packvec{k}}^2 \leq 4\delta^2$, and as a result
$\max_{j,k} \kl{\mprob_{\packvec{j}}}{\mprob_{\packvec{k}}} \leq
\frac{4\numobs \cdfparam \delta^2}{\noisestd^2}$. Choosing $\delta^2 =
\frac{ \noisestd^2 \log2}{8\numobs \cdfparam}$ and applying
Lemma~\ref{LemGeneralMinimaxLower} yields the claim.


\subsection{Lower bound: Part II}

Given an integer $\numitems' \in \{2,\ldots,\numitems\}$, and scalars
$\alpha \in (0,\frac{1}{4})$ and $\delta > 0$, define the integer
\begin{align}
\packnum'(\alpha) & \defn \left\lfloor \exp \Big \{
\frac{\numitems'}{2} \big(\log 2 + 2\alpha \log 2\alpha + (1-2\alpha)
\log (1-2\alpha)\big) \Big \} \right \rfloor.
\end{align}
Applying Lemma~\ref{LemGilVar} with $\numitems'$ as the dimension
yields a subset $\{\lemmawt^1,\ldots,\lemmawt^{\packnum'(\alpha)}\}$
of the Boolean hypercube $\{0,1\}^{\numitems'}$ with the stated
properties.  We then define a set of $\numitems$-length vectors
$\{\wttil^1,\ldots,\wttil^{\packnum'(\alpha)}\}$ via
\begin{align*}
\wttil^j = [0~(\lemmawt^j)^T~0~\cdots~0]^T \quad \mbox{for each $j \in
  [\packnum(\alpha)]$}.
\end{align*}
For each $j \in [\packnum(\alpha)]$, let us define $\packvec{j} \defn
\frac{\delta}{\sqrt{\numitems'}} U^T \sqrt{\LamTil} \wttil^j$.  Now,
letting $\unitvec_1 \in \real^\numitems$ denote the first standard
basis vector, we have $\inprod{\ones}{\packvec{j}} =
\frac{\delta}{\sqrt{\numitems'}} \ones^T U^T \sqrt{\LamTil} \wttil^j
\; = 0$.  where we have used the fact that $\ones
\in \nullspace({\Lap})$.  Furthermore, for any $j \neq k$, we have
\begin{align*}
\Lnorm{\packvec{j}-\packvec{k}}{2}^2 = \frac{\delta^2}{\numitems'}
(\wttil^j - \wttil^k)^T \LambdaInv (\wttil^j - \wttil^k)
\geq \frac{\delta^2}{\numitems'} \sum_{i = \lfloor (1-\alpha)
  \numitems' \rfloor}^{\numitems'} \frac{1}{\lambda_i}.
\end{align*}
Thus, setting $\delta^2 = 0.01 \frac{\noisestd^2 \numitems'}{\numobs
  \cdfparam}$ yields
\begin{align*}
\|\wt^j\|_\infty \; \leq \; \frac{\delta}{\sqrt{\numitems'}} \|
     {\sqrt{\LambdaInv}} \wttil^j\|_2 \; \stackrel{(i)}{\leq} \;
     \frac{\delta}{\sqrt{\numitems'}} \sqrt{\trace{\LambdaInv}} \;
     \stackrel{(ii)}{=} \; \frac{\delta}{\sqrt{\numitems'}}
     \sqrt{\trace{\LapInv}} \; \stackrel{(iii)}{\leq} \; \wmax,
\end{align*}
where inequality (i) follows from the fact that $\lemmawt^j$ has
entries in $\{0, 1\}$; step (ii) follows because the matrices
$\sqrt{\LambdaInv}$ and $\sqrt{\LapInv}$ have the same eigenvalues;
and inequality (iii) follows from our choice of $\delta$ and our
assumption \mbox{$\numobs \geq \frac{\plaincon \noisestd^2
    \trace{\LapInv}}{\cdfparam \wmax^2}$} on the sample size with
$\plaincon = 0.01$. We have thus verified that each vector
$\packvec{j}$ also satisfies the boundedness constraint
$\|\packvec{j}\|_\infty \leq B$ required for membership in
$\Wclass_\wmax$.
Furthermore, for any pair of distinct vectors in this set, we have
\begin{align*}
\Lnorm{\packvec{j}-\packvec{k}}{\Lap}^2 = \frac{\delta^2}{\numitems'}
\|\lemmawt^j - \lemmawt^k\|_2^2 \leq \delta^2.
\end{align*}
From the proof of Theorem~\ref{ThmMinimaxL}, we
$\kl{\mprob_{\packvec{j}}}{\mprob_{\packvec{k}}} \leq \frac{\numobs
  \cdfparam }{\noisestd^2} \Lnormd{\packvec{j}}{\packvec{k}}{\Lap}^2
\leq 0.01 \numitems'$.  Applying Lemma~\ref{LemGeneralMinimaxLower}
with the packing set
$\{\packvec{1},\ldots,\packvec{\packnum(\alpha)}\}$ gives
\begin{align*}
\MiniMax_\numobs \big(\wt(\Pclass); \Lnorm{\cdot}{2}^2 \big) & \geq
\frac{\alpha \delta^2}{2} \big( 1- \frac{0.01 +\log 2}{\log
  \packnum'(\alpha)} \big).
\end{align*}
Substituting our choice of $\delta$ and setting $\alpha=0.01$ proves
the claim for $\numitems ' > 9$.

For the case of $\numitems' \leq 9$, we will show a lower bound of
$\frac{\plaincon \noisestd^2}{\numobs} \frac{9}{
  \eigenvalue{2}{\Lap}}$ for a universal constant $\plaincon>0$. This
quantity is at least as large as the claimed lower bound. Consider the
packing set of three $\numitems$-length vectors $\wt^1 = \delta U
\sqrt{\LambdaInv} [0~~1~~0~~\cdots~~0]^T$, $\wt^2 = - \wt^1$ and
$\wt^3 = [0~~\cdots~~0]^T$ for some $\delta>0$. Then for every $j \neq
k$, one can verify that $\Lnorm{\wt^j-\wt^k}{\Lap}^2 \leq 4 \delta^2$,
$\Lnorm{\wt^j - \wt^k}{2}^2 \geq
\frac{\delta^2}{\eigenvalue{2}{\Lap}}$. Choosing $\delta^2 = \frac{
  \noisestd^2 \log2}{8 \numobs \cdfparam}$ and applying
Lemma~\ref{LemGeneralMinimaxLower} proves the claim for $\numitems'
\leq 9$.

Finally, taking the maximum over all values of $\numitems' \in
\{2,\ldots,\numitems\}$ gives the claimed lower bound.


\section{Proof of Theorem~\ref{ThmMinimaxPair}}
\label{AppThmMinimaxPairedLinear}

We now turn to the proof of Theorem~\ref{ThmMinimaxPair} on the
minimax rate for the \pair model.  Recall that this observation model
takes the standard linear model, $y = \diffmx \wtstar + \epsilon$,
where $y \in \reals^\numobs, \wt \in \reals^\numitems$ and $\epsilon
\sim N(0,\sigma^2 I)$.

\subsection{Upper bound under the squared $\Lap$ semi-norm}

The maximum likelihood estimate in the \pair model is a special case
of the general \mbox{$M$-estimator~\eqref{EqnGeneralMest}} with
$\loss(\wt) \defn \frac{1}{2 \numobs} \sum_{i=1}^{\numobs} \big(\obs_i
- \inprod{\diff_i}{\wt} \big)^2$.  For this quadratic objective
function, it is easy to verify that the $\strongcon$-convexity
condition holds with $\strongcon = 1$.  (In particular, note that the
Hessian of $\loss$ is given by $\Lap = \Xmat^T \Xmat/\numobs$.)

Given the result of Lemma~\ref{LemMestUpper}, it remains to upper
bound $\LAPNORMINV{\nabla \loss(\wtstar)}$.  A straightforward
computation yields $\LAPNORMINV{\nabla \loss(\wtstar)}^2 =
\frac{\regnoise}{\sigma}^T Q \frac{\regnoise}{\sigma}$ where $Q \defn
\frac{\sigma^2}{\numobs^2} \Xmat \LapInv \Xmat^T$.  Consequently, the
random variable $\LAPNORMINV{\nabla \loss(\wtstar)}^2$is quadratic
form in the standard Gaussian random vector
$\frac{\regnoise}{\sigma}$.  An application of Lemma~\ref{LemLapX} (Appendix~\ref{AppLaplacian})
gives $\trace{Q} = \frac{\sigma^2}{\numobs} \big(\numitems-1 \big)$
and $\opnorm{Q} = \frac{\sigma^2}{\numobs}$, and then applying a known
tail bound on Gaussian quadratic forms (see Lemma~\ref{LemQuadForm} in
Appendix~\ref{AppTailBounds}) yields
\begin{align*}
 \mprob \left[ \frac{\LAPNORMINV{\nabla \loss(\wtstar)}^2}{\sigma^2}
   \geq \Big( \sqrt{\frac{\numitems}{\numobs}} +
   \frac{\delta}{\sqrt{\numobs}} \Big)^2 \right] & \leq e^{ -
   \frac{\delta^2}{2}}\qquad \mbox{for all }\delta \geq 0.
\end{align*}
Since $\numitems \geq 2$, we have $\big( \sigma
\sqrt{\frac{\numitems}{\numobs}} + \frac{\sigma}{\sqrt{\numobs}}
\delta \big)^2 \leq \frac{2\noisestd^2 \numitems \delta^2}{\numobs}$
for all $\delta \geq 4$, which yields
\begin{align*}
 \mprob \Big[\LAPNORMINV{\nabla \loss(\wtstar)}^2 \geq t \:
   \frac{4\noisestd^2 \numitems}{\numobs} \Big] & \leq e^{-t} \qquad
 \mbox{for all $t \geq 8$.}
\end{align*}
Integrating this tail bound yields that $\Exs\Big[ \LAPNORMINV{\nabla
    \loss(\wtstar)}^2 \Big] \leq \plaincon \sigma^2
\frac{\numitems}{\numobs}$, from which the claim follows.


\subsection{Lower bound under the squared $\Lap$ semi-norm}

Based on the pairwise Fano lower bound previously stated in
Lemma~\ref{LemGeneralMinimaxLower}, we need to construct a suitable
$(\delta, \beta)$-packing, where the semi-norm $\rho(\packvec{j},
\packvec{k}) = \LAPNORM{\packvec{j} - \packvec{k}}$ is defined by the
Laplacian.  Given the additive Gaussian noise observation model, we
also have
\begin{align}
\kull{\mprob_{\packvec{j}}}{\mprob_{\packvec{k}}} & = \frac{\numobs}{2
  \sigma^2} \LAPNORM{\packvec{j} - \packvec{k}}^2,
\end{align}
The construction of the packing and the remainder of the proof proceeds in a manner identical to the proof of the lower bound in Theorem~\ref{ThmMinimaxL}, except for the absence of the requirement of $\Lnorm{\packvec{j}}{\infty} \leq \wmax$ on the elements $\{\packvec{j}\}$ of the packing set.


\subsection{Upper bound under the squared Euclidean norm}

The upper bound follows by direct analysis of the (unconstrained)
least-squares estimate, which has the explicit form $\wthat =
\frac{1}{\numobs} \LapInv \diffmx^T y$, and thus
\begin{align*}
\Exs \| \wthat - \wtstar\|_2^2 = \Exs \| \frac{1}{\numobs} \LapInv
\diffmx^T \epsilon \|_2^2 \; = \; \sigma^2 \trace{\frac{1}{\numobs^2}
  \LapInv X^T X \LapInv}
\end{align*}
where we have used the fact that $\epsilon \sim N(0, \sigma^2
I_\numobs)$.  Since $\Lap = X^T X/\numobs$ by definition, we conclude
that $\Exs \| \wthat - \wtstar\|_2^2 = \frac{\sigma^2
  \trace{\LapInv}}{\numobs}$ as claimed.


\subsection{Lower bound under the squared Euclidean norm}

We obtain the lower bound by computing the Bayes risk with respect to
a suitably defined (proper) prior distribution over the weight vector
$\wtstar$.  In particular, if we impose the prior $\wtstar \sim N(0,
\frac{\sigma^2}{\numobs} \LapInv )$, Bayes' rule then leads to the
posterior distribution
\begin{align*}
\mprob \big (\wt \mid  \obs ; \diffmx \big) & \propto \exp \left(
\frac{-1}{2\sigma^2} \| \obs - \diffmx \wt \|_2^2 \right) \exp \left(
\frac{-\numobs}{2\sigma^2} \wt^T \Lap \wt \right)
\Ind\{\inprod{\wt}{\ones}=0\}.
\end{align*}
Thus conditioned on $\obs$, $\wt$ is distributed as $N \left(
(\diffmx^T \diffmx + \numobs \Lap)^{-1} \diffmx^T \obs, \frac{\sigma^2}{2}
\LapInv \right).$ By applying iterated expectations, the Bayes risk is
given by $\Exs \|\wt - \frac{1}{2} \LapInv \diffmx^T \obs \|_2^2 =
\frac{\sigma^2}{2} \trace{\LapInv}$, which completes the proof.


\section{Proof of Theorem~\ref{thm:mwise}}
\label{AppThmMwise}

This section presents the proof of Theorem~\ref{thm:mwise} for the
setting of $\numchoices$-wise comparisons. We first state some simple
properties of the model introduced in Section~\ref{SecMWise}, which we
will use subsequently in the proofs of the results.

\begin{lemma}
\label{lem:auxLaplacian}
The Laplacian of the underlying pairwise-comparison graph satisfies
the trace constraints $\nullspace(\genlapmx) = \ones$,
$\lambda_2(\genlapmx) > 0$ and $\trace{\genlapmx} =
\numchoices(\numchoices-1)$.
\end{lemma}

\begin{lemma}
\label{lem:auxHessian}
For any $j \in [\numchoices]$, $i \in [\numobs]$ and any vector $v \in
\reals^\numchoices$, we have
\begin{align*}
\frac{\eigenvalue{2}{\hessgencdf}}{\numchoices} v^T (\numchoices
\identity - \ones \ones^T) v \leq v^T \revmx_j \hessgencdf \revmx_j^T
v \leq \frac{\eigenvalue{\max}{\hessgencdf}}{\numchoices} v^T
(\numchoices \identity - \ones \ones^T) v.
\end{align*}
\end{lemma}

\noindent See Section~\ref{AppAuxiliaryMwise} for the proof of these
auxiliary lemmas.


\subsection{Upper bound under the squared $\Lap$ semi-norm}

We prove this upper bound by applying Lemma~\ref{LemMestUpper}.
In this case, the rescaled negative log likelihood takes the form
\begin{align*}
\loss(\wt) & = -\frac{1}{\numobs} \sum_{i=1}^{\numobs} \sum_{j =
  1}^{\numchoices} \Ind[\obs_i=j] \log \glmcdf \big( \wt^T \gendiff_i
\revmx_j \big),
\end{align*}
and the MLE is obtained by constrained minimization over the set
$\SPECSET \defn \big \{ \wt \in \real^\numitems \, \mid \,
\inprod{\ones}{\wt} = 0, \quad \mbox{and} \quad \|\wt\|_\infty \leq
\wmax \big \}$.  As in our proof of the upper bound in Theorem~\ref{ThmMinimaxL},
we need to verify the $\lemstrongcon$-strong convexity condition, and
to control the dual norm $\LAPNORMINV{\nabla \loss(\wtstar)}$.


\paragraph{Verifying strong convexity:}

The gradient of the negative log likelihood is
\begin{align*}
\nabla \loss(\wt) & = -\frac{1}{\numobs} \sum_{i=1}^{\numobs}
\sum_{j=1}^{\numchoices} \Ind[\obs_i=j] \gendiff_i \revmx_j \nabla
\log \glmcdf (v) \big|_{ v = \wt^T \gendiff_i \revmx_j } .
\end{align*}
The Hessian of the negative log likelihood can be written as
\begin{align*}
\nabla^2 \loss(\wt) & = \frac{1}{\numobs} \sum_{i=1}^{\numobs}
\sum_{j=1}^{\numchoices} \Ind[\obs_i=j] \gendiff_i \revmx_j \nabla^2
\log \glmcdf (v) \big|_{ v = \wt^T \gendiff_i \revmx_j } \revmx_j^T
\gendiff_i^T .
\end{align*}
Using our strongly log-concave assumption on $\glmcdf$, we have that
for any vector $z \in \reals^\numitems$,
\begin{align*}
z^T \nabla^2 \loss(\wt) z & = -\frac{1}{\numobs} \sum_{i=1}^{\numobs}
\sum_{j=1}^{\numchoices} \Ind[\obs_i=j] z^T \gendiff_i \revmx_j
\nabla^2 \log \glmcdf (v) \big|_{ v = \wt^T \gendiff_i \revmx_j }
\revmx_j^T \gendiff_i^T z \\ & \geq \frac{1}{\numobs}
\sum_{i=1}^{\numobs} \sum_{j=1}^{\numchoices} \Ind[\obs_i=j] z^T
\gendiff_i \revmx_j \hessgencdf \revmx_j^T \gendiff_i^T z\\ & \geq
\frac{\eigenvalue{2}{\hessgencdf}}{\numchoices} \frac{1}{\numobs}
\sum_{i=1}^{\numobs} \sum_{j=1}^{\numchoices} \Ind[\obs_i=j] z^T
\gendiff_i (\numchoices \identity - \ones \ones^T) \gendiff_i^T z,
\end{align*}
where the last step follows from Lemma~\ref{lem:auxHessian}.  The
definition~\eqref{eq:defn_genlapmx} of $\genlapmx$ implies that
\begin{align*}
z^T \nabla^2 \loss(\wt) z & \geq
\frac{\eigenvalue{2}{\hessgencdf}}{\numchoices} z^T \genlapmx z =
\frac{\eigenvalue{2}{\hessgencdf}}{\numchoices}
\Lnorm{z}{\genlapmx}^2.
\end{align*}
Consequently, the $\lemstrongcon$-convexity condition holds around
$\wtstar$ with $\lemstrongcon =
\frac{\eigenvalue{2}{\hessgencdf}}{\numchoices}$.
An application of Lemma~\ref{LemMestUpper} then yields
\begin{align}
\label{EqnMWiseUpper1}
\Lnorm{\wtMLE - \wtstar}{\genlapmx}^2 & \leq
\frac{\numchoices^2}{\eigenvalue{2}{\hessgencdf}^2} \LAPNORMINV{\nabla
  \loss(\wtstar)}^2 \; = \;
\frac{\numchoices^2}{\eigenvalue{2}{\hessgencdf}^2} \nabla \loss
(\wtstar)^T \genlapmxinv \nabla\loss (\wtstar).
\end{align}

\paragraph{Controlling the dual norm:}

The gradient of the negative log likelihood can then be rewritten as
$\nabla \loss(\wtstar) = -\frac{1}{\numobs} \sum_{i=1}^{\numobs}
\gendiff_i V_i$, where each index $i \in [\numobs]$, the random vector
vector $V_i \in \real^\numchoices$ is given by $V_i \defn
\sum_{j=1}^{\numchoices} \Ind[\obs_i=j] \; \revmx_j \: \nabla \log
\glmcdf (\inprod{\wtstar}{\gendiff_i} \revmx_j)$.  Now observe that
the matrix $\rankminusonemx \defn \identity - \frac{1}{\numchoices}
\ones \ones^T$ is symmetric and positive semi-definite with rank
$(\numchoices-1)$, eigenvalues $\{1,\ldots,1,0\}$, its nullspace
equals the span of the all-ones vector, and that
$\rankminusonemx^\dagger = \rankminusonemx$.  Using this matrix, we
define the transformed vector $\Vtil_i \defn
(\rankminusonemx^\dagger)^\frac{1}{2} \Vvar_i$ for each $i \in
[\numobs]$.

Consider a vector $x$ and its shifted version $x + t \ones$, where $t
\in \real$ and $\ones$ denotes the vector of all ones.  By the shift
invariance property, the function $g(t) = F(x + t \ones) - F(x)$ is
constant, and hence 
\begin{align}
\label{EqnAuxCDF}
g'(0) = \inprod{\nabla F(x)}{\ones} = 0, \quad \mbox{and} \quad g''(0)
= \inprod{\ones}{\big(\nabla^2 F(x) \big) \ones} = 0,
\end{align}
which implies that $1 \in \nullspace(\nabla^2 F(x))$. Furthermore, we
have $\inprod{\nabla \log F(x)}{\ones} = \frac{1}{F(x)} \inprod{\nabla
  F(x)}{\ones} = 0$.  Consequently, $\inprod{V_i}{\ones} = 0 =
\inprod{V_i}{\nullspace(\rankminusonemx)}$. This allows us to write
\begin{align*}
\nabla \loss(\wtstar) & = -\frac{1}{\numobs} \sum_{i=1}^{\numobs}
\gendiff_i \rankminusonemx^\frac{1}{2} \Vtil_i, \quad \mbox{and} \quad
\nabla \loss (\wtstar)^T \genlapmxinv \nabla\loss (\wtstar) =
\frac{1}{\numobs^2} \sum_{i=1}^{\numobs} \sum_{\ell=1}^{\numobs}
\Vtil_i^T \rankminusonemx^\frac{1}{2} \gendiff_i^T \LapInv
\gendiff_\ell \rankminusonemx^\frac{1}{2} \Vtil_\ell.
\end{align*}
By definition, for every pair $i \neq \ell \in [\numobs]$, $\Vtil_i$
is independent of $\Vtil_\ell$. Moreover, for every $i \in
[\numobs]$,
\begin{align*}
\Exs[ \Vtil_i ] &= \Exs
    [(\rankminusonemx^\dagger)^{\frac{1}{2}} \sum_{j=1}^{\numchoices}
      \Ind[\obs_i=j] \revmx_j \nabla \log \glmcdf (v) \big|_{ v =
        (\wtstar)^T \gendiff_i \revmx_j}]\\ &=
    (\rankminusonemx^\dagger)^{\frac{1}{2}} \sum_{j=1}^{\numchoices}
    \glmcdf((\wtstar)^T \gendiff_i \revmx_j) \revmx_j \nabla \log
    \glmcdf (v) \big|_{ v = (\wtstar)^T \gendiff_i \revmx_j}\\ &=
    (\rankminusonemx^\dagger)^{\frac{1}{2}} \sum_{j=1}^{\numchoices}
    \revmx_j \nabla \glmcdf (v) \big|_{ v = (\wtstar)^T \gendiff_i
      \revmx_j}.
\end{align*}
In order to further evaluate this expression, define a function
$g:\reals^m \rightarrow \reals$ as $g(z) = \sum_{j=1}^{\numchoices}
\glmcdf(z^T \revmx_j)$. Then by definition we have $g(z)=1$. Taking
derivatives, we get $0 = \nabla g(z) = \sum_{j=1}^{\numchoices}
\revmx_j \nabla \glmcdf(z^T \revmx_j)$. It follows that $\Exs[
  \Vtil_i ] = 0$, and hence that
\begin{align*}
\Exs[ \nabla \loss (\wtstar)^T \genlapmxinv \nabla\loss (\wtstar) ] &
= \frac{1}{\numobs^2} \Exs[ \sum_{i=1}^{\numobs}
  \sum_{\ell=1}^{\numobs} \Vtil_i^T
  \rankminusonemx^\frac{1}{2} \gendiff_i^T \LapInv \gendiff_\ell
  \rankminusonemx^\frac{1}{2} \Vtil_\ell] \\
& =
\frac{1}{\numobs^2} \Exs[ \sum_{i=1}^{\numobs} \Vtil_i^T
  \rankminusonemx^\frac{1}{2} \gendiff_i^T \LapInv \gendiff_i
  \rankminusonemx^\frac{1}{2} \Vtil_i] \\
& \leq \frac{1}{\numobs} \Exs[ \sup_{\ell \in [\numobs]}
  \Lnorm{\Vtil_\ell}{2}^2] \trace{\frac{1}{\numobs}
  \sum_{i=1}^{\numobs} \rankminusonemx^\frac{1}{2} \gendiff_i^T
  \LapInv \gendiff_i \rankminusonemx^\frac{1}{2} } .
\end{align*}
Since $\Lap = \frac{\numchoices}{\numobs} \sum_{i=1}^{\numobs}
\gendiff_i \rankminusonemx \gendiff_i^T$, we have
$\trace{\frac{1}{\numobs} \sum_{i=1}^{\numobs}
  \rankminusonemx^\frac{1}{2} \gendiff_i^T \LapInv \gendiff_i
  \rankminusonemx^\frac{1}{2} } = \frac{\numitems-1}{\numchoices}$, as
well as
\begin{align*}
\Lnorm{\Vtil_\ell}{2}^2 = \sum_{j=1}^{\numchoices}
\Ind[\obs_i=j] ( \nabla \log \glmcdf (v) \big|_{ v = (\wtstar)^T
  \gendiff_i \revmx_j})^T \revmx_j^T \rankminusonemx \revmx_j \nabla
\log \glmcdf (v) \big|_{ v = (\wtstar)^T \gendiff_i \revmx_j}.
\end{align*}
Recalling the previously defined matrix $\rankminusonemx$, observe
that since $\revmx_j$ is simply a permutation matrix, we have
$\revmx_j^T \rankminusonemx \revmx_j = \rankminusonemx$ for every $j
\in [\numchoices]$.  By chain rule, we have $\inprod{\nabla \log
  \glmcdf (v)}{\ones} = \frac{1}{\glmcdf(v)} \inprod{\nabla
  \glmcdf(v)}{\ones} = 0$, where the last step follows from our
previous calculation.  It follows that
\begin{align*}
\Exs \big[ \inprod{\nabla \loss (\wtstar)}{\genlapmxinv \nabla \loss
    (\wtstar)} \big] & \leq \frac{\numitems}{\numobs} \sup_{v \in
  [-\wmax,\wmax]^\numchoices} \Lnorm{\nabla \log \glmcdf (v)}{2}^2 .
\end{align*}
Substituting this bound into equation~\eqref{EqnMWiseUpper1} yields
the claim.


\subsubsection{Lower bound under the squared $\Lap$ semi-norm}
\label{AppMWiseLLower}

For any pair of quality score vectors $\packvec{j}$ and $\packvec{k}$,
the KL divergence between the distributions $\mprob_{\packvec{j}}$ and
$\mprob_{\packvec{k}}$ is given by
\begin{align*}
\kl{\mprob_{\packvec{j}}}{\mprob_{\packvec{k}}} & =
\sum_{i=1}^{\numobs} \sum_{l=1}^{\numchoices} \glmcdf(\packvec{j}^T
\gendiff_i \revmx_l) \log \frac{\glmcdf(\packvec{j}^T \gendiff_i
  \revmx_l)}{\glmcdf(\packvec{k}^T \gendiff_i \revmx_l)} .
\end{align*}
Applying the inequality $\log x \leq x-1$, valid for $x>0$, we find
that
\begin{align*}  
\kl{ \mprob_{\packvec{j}}}{\mprob_{\packvec{k}}} & \leq
\sum_{i=1}^{\numobs} \sum_{l=1}^{\numchoices} \glmcdf(\packvec{j}^T
\gendiff_i \revmx_l) \Big( \frac{\glmcdf(\packvec{j}^T \gendiff_i
  \revmx_l)}{\glmcdf(\packvec{k}^T \gendiff_i \revmx_l)} - 1 \Big) .
\end{align*}
Now employing the fact that $\sum_{l=1}^{\numchoices}
\glmcdf(\packvec{j}^T \gendiff_i \revmx_l) = \sum_{l=1}^{\numchoices}
\glmcdf(\packvec{k}^T \gendiff_i \revmx_l) = 1$ gives
\begin{align*}
\kl{\mprob_{\packvec{j}}}{\mprob_{\packvec{k}}} & \leq
\sum_{i=1}^{\numobs} \sum_{l=1}^{\numchoices} \Big(
\frac{\glmcdf(\packvec{j}^T \gendiff_i
  \revmx_l)^2}{\glmcdf(\packvec{k}^T \gendiff_i \revmx_l)} -
2\glmcdf(\packvec{j}^T \gendiff_i \revmx_l) + \glmcdf(\packvec{k}^T
\gendiff_i \revmx_l) \Big).\\
& = \sum_{i=1}^{\numobs} \sum_{l=1}^{\numchoices}
\frac{(\glmcdf(\packvec{j}^T \gendiff_i \revmx_l) -
  \glmcdf(\packvec{k}^T \gendiff_i \revmx_l))^2}{\glmcdf(\packvec{k}^T
  \gendiff_i \revmx_l)}\\
& \leq \frac{1}{\glmcdf(-\wmax,\wmax,\ldots,\wmax)}
\sum_{i=1}^{\numobs} \sum_{l=1}^{\numchoices} (\glmcdf(\packvec{j}^T
\gendiff_i \revmx_l) - \glmcdf(\packvec{k}^T \gendiff_i \revmx_l))^2\\
& \leq \frac{1}{\glmcdf(-\wmax,\wmax,\ldots,\wmax)}
\sum_{i=1}^{\numobs} \sum_{l=1}^{\numchoices} (\inprod{\nabla
  \glmcdf(z_{il})}{\packvec{j}^T \gendiff_i \revmx_l - \packvec{k}^T
  \gendiff_i \revmx_l})^2,
\end{align*}
for some $z_{il} \in [-\wmax,\wmax]^\numchoices$. Letting $\cdfparam =
\frac{\sup_{z \in [-\wmax,\wmax]^\numchoices} \Lnorm{\nabla
    \glmcdf(z)}{\hessgencdf^\dagger}^2}{\glmcdf(-\wmax,\wmax,\ldots,\wmax)}$
and applying Lemma~\ref{LemPseudoCauchySchwarz} (noting that
$\inprod{\packvec{j}^T \gendiff_i
  \revmx_l}{\nullspace(\hessgencdf)}=0$ for all $i,j,l$) gives
\begin{align}
\kl{\mprob_{\packvec{j}}}{\mprob_{\packvec{k}}} & \leq
\sum_{i=1}^{\numobs} \sum_{l=1}^{\numchoices} \cdfparam
\Lnorm{\packvec{j}^T \gendiff_i \revmx_l - \packvec{k}^T \gendiff_i
  \revmx_l}{\hessgencdf}^2 \nonumber \\
& \leq \cdfparam (\packvec{j} - \packvec{k})^T \Big(
\sum_{i=1}^{\numobs} \sum_{l=1}^{\numchoices} \gendiff_i \revmx_l
\hessgencdf \revmx_l^T \gendiff_i^T \Big) (\packvec{j} - \packvec{k})
\nonumber\\
& \leq \cdfparam \eigenvalue{\numchoices}{\hessgencdf} \numobs
\Lnorm{\packvec{j} - \packvec{k}}{\genlapmx}^2,
\label{eq:KL_gen}
\end{align}
where the final step is a result of Lemma~\ref{lem:auxHessian}.

Consider the pair of scalars $\alpha \in (0,\frac{1}{4})$ and $\delta
> 0$ whose values will be specified later. Let $\packnum(\alpha)$ be
as defined in~\eqref{EqnGVPackSize}. Consider the packing set
$\{\wt^1, \ldots, \wt^{\packnum(\alpha)} \}$ constructed in
Appendix~\ref{AppMinimaxLLower}. Each of these vectors is of length
$\numitems$, satisfies $\inprod{\wt^j}{\ones} = 0$, and furthermore,
each pair from this set satisfies $ \packdmin \delta^2 \leq
\LAPNORM{\wt^{j} - \wt^{k}}^2 \leq \delta^2$. Setting $\delta^2 = 0.01
\frac{\numitems}{\numobs \cdfparam
  \eigenvalue{\numchoices}{\hessgencdf}}$ yields
\begin{align*}
\kl{\mprob_{\packvec{j}}}{\mprob_{\packvec{k}}} \leq 0.01 \numitems.
\end{align*}
Every element from the packing set also satisfies $\|\wt^j\|_\infty
\leq \wmax$ when \mbox{$\numobs \geq \frac{0.01 \noisestd^2
    \trace{\LapInv}}{\cdfparam \wmax^2
    \eigenvalue{\numchoices}{\hessgencdf}}$}, and thus belongs to the
class $\Wclass_\wmax$.

Applying Lemma~\ref{LemGeneralMinimaxLower} yields the lower bound
\begin{align*}
\LAPNORM{\wthat - \wtstar}^2 & \geq \frac{\alpha}{2} 0.01
\frac{\numitems}{\numobs \cdfparam
  \eigenvalue{\numchoices}{\hessgencdf}} \Big \{ 1 - \frac{0.01
  \numitems + \log 2}{\log \packnum(\alpha)}\Big \}.
\end{align*}
Setting $\alpha=0.01$ proves the claim for $\numitems>9$.

For the case of $\numitems \leq 9$, consider the set of the three
$\numitems$-length vectors $\lemmawt^1 = [0~~\cdots~~0~~-1]$,
$\lemmawt^2 = [0~~\cdots~~0~~1]$ and $\lemmawt^3 =
[0~~\cdots~~0~~0]$. Construct the packing set $\{\wt^1,\wt^2,\wt^3\}$
from these three vectors $\{\lemmawt^1,\lemmawt^2,\lemmawt^3\}$ as
done above for the case of $\numitems>9$. From the calculations made
for the general case above, we have for all pairs $\min_{j \neq k}
\Lnorm{\wt^j-\wt^k}{\Lap}^2 \geq \frac{\delta^2}{9}$ and $\max_{j,k}
\LAPNORM{\packvec{j} - \packvec{k}}^2 \leq 4\delta^2$, and as a result
$\max_{j,k} \kl{\mprob_{\packvec{j}}}{\mprob_{\packvec{k}}} \leq
4\numobs \cdfparam \eigenvalue{\numchoices}{\hessgencdf}
\delta^2$. Choosing $\delta^2 = \frac{ \log2}{8\numobs \cdfparam
  \eigenvalue{\numchoices}{\hessgencdf}}$ and applying
Lemma~\ref{LemGeneralMinimaxLower} proves the claim.


\subsubsection{Upper bound under the squared Euclidean norm}

The upper bound under the squared $\ell_2$-norm follows directly from
the upper bound under the squared $\Lap$ semi-norm in
Theorem~\ref{thm:mwise}: noting that $(\wtstar - \wthat)
\perp \nullspace(\genlapmx)$, we get that
\begin{align*}
(\wtstar - \wthat)^T \genlapmx (\wtstar - \wthat) \geq
  \eigenvalue{2}{\genlapmx} \Lnorm{\wtstar - \wthat}{2}^2.
\end{align*}
Substituting this inequality in the upper bound on the minimax risk
under the squared $\Lap$ semi-norm in Theorem~\ref{thm:mwise} gives the
desired result.


\subsubsection{Lower bound under the squared Euclidean norm}

Define $\cdfparam = \frac{\sup_{z \in [-\wmax,\wmax]^\numchoices}
  \Lnorm{\nabla
    \glmcdf(z)}{\hessgencdf^\dagger}^2}{\glmcdf(-\wmax,\wmax,\ldots,\wmax)}$.
Equation~\eqref{eq:KL_gen} in Appendix~\ref{AppMWiseLLower} shows that
for any vectors $\wt^j, \wt^k \in \Wclass_\wmax$,
\begin{align*}
\kl{\mprob_{\packvec{j}}}{\mprob_{\packvec{k}}} & \leq \cdfparam
\eigenvalue{\numchoices}{\hessgencdf} \numobs \Lnorm{\packvec{j} -
  \packvec{k}}{\genlapmx}^2,
\end{align*}

Consider the pair of scalars $\alpha \in (0,\frac{1}{4})$ and $\delta
> 0$ whose values will be specified later. Let $\packnum(\alpha)$ be
as defined in~\eqref{EqnGVPackSize}. In
Appendix~\ref{AppMinimax2Lower1} we constructed a set $\{\packvec{1},
\ldots, \packvec{\packnum(\alpha)} \}$ of vectors of length
$\numitems$ that satisfy $\inprod{\wt^j}{\ones}=0$ for every $j \in
      [\packnum(\alpha)]$, and for every pair of vectors in this set,
      $\Lnormd{\packvec{j}}{\packvec{k}}{2}^2 \geq \alpha \delta^2$
      and $\frac{1}{{\packnum(\alpha) \choose 2}} \sum_{j \neq k}
      \Lnormd{\wttil^j}{\wttil^k}{\Lap}^2 \leq \frac{2
        \delta^2}{\numitems} \trace{\genlapmx}$. Applying
      Lemma~\ref{lem:auxLaplacian} gives
\begin{align*}
\frac{1}{{\packnum(\alpha) \choose 2}} \sum_{j \neq k}
\Lnormd{\wttil^j}{\wttil^k}{\Lap}^2 \leq \frac{2 \delta^2}{\numitems}
\numchoices (\numchoices-1).
\end{align*}
Setting $\delta^2 = 0.005 \frac{\numitems^2}{\numobs \cdfparam
  \eigenvalue{\numchoices}{\hessgencdf} \numchoices (\numchoices-1)}$
yields
\begin{align*}
\kl{\mprob_{\packvec{j}}}{\mprob_{\packvec{k}}} \leq 0.01 \numitems.
\end{align*}
In a manner similar to Lemma~\ref{LemLapProperties} in the pairwise
comparison case, one can show that in the general setting of this
section, $\trace{\LapInv} \geq \frac{\numitems^2}{4 \numchoices
  (\numchoices-1)}$. Then, every element from the packing set also
satisfies $\|\wt^j\|_\infty \leq \wmax$ when $\delta \leq \wmax$,
which holds true under our assumption of \mbox{$\numobs \geq
  \frac{\plaincon \noisestd^2 \trace{\LapInv}}{\cdfparam \wmax^2
    \eigenvalue{\numchoices}{\hessgencdf}} \geq \frac{\plaincon
    \noisestd^2 \numitems^2 }{4 \numchoices (\numchoices-1) \cdfparam
    \wmax^2 \eigenvalue{\numchoices}{\hessgencdf}}$} with $\plaincon =
0.01$. Each element of our packing set thus belongs to the class
$\Wclass_\wmax$.
Applying Lemma~\ref{LemGeneralMinimaxLower} yields the lower bound
\begin{align*}
\LAPNORM{\wthat - \wtstar}^2 & \geq \frac{\alpha}{2} 0.01
\frac{\numitems^2}{\numobs \cdfparam
  \eigenvalue{\numchoices}{\hessgencdf} \numchoices (\numchoices-1)}
\Big \{ 1 - \frac{0.01 \numitems + \log 2}{\log \packnum(\alpha)}\Big
\}.
\end{align*}
Setting $\alpha=0.01$ proves the claim for $\numitems>9$.


For the case of $\numitems \leq 9$, consider the set of the three $\numitems$-length
vectors $\lemmawt^1 = [0~~\cdots~~0~~-1]$, $\lemmawt^2 =
[0~~\cdots~~0~~1]$ and $\lemmawt^3 = [0~~\cdots~~0~~0]$. Construct the
packing set $\{\wt^1,\wt^2,\wt^3\}$ from these three vectors
$\{\lemmawt^1,\lemmawt^2,\lemmawt^3\}$ as done above for the case of
$\numitems>9$. From the calculations made for the general case above,
we have for all pairs $\min_{j \neq k} \Lnorm{\wt^j-\wt^k}{2}^2 \geq
\frac{\delta^2}{9}$ and $\max_{j,k} \LAPNORM{\packvec{j} -
  \packvec{k}}^2 \leq 4\delta^2$, and as a result $\max_{j,k}
\kl{\mprob_{\packvec{j}}}{\mprob_{\packvec{k}}} \leq 4\numobs
\cdfparam \eigenvalue{\numchoices}{\hessgencdf} \delta^2$. Choosing
$\delta^2 = \frac{\log2}{8\numobs \cdfparam
  \eigenvalue{\numchoices}{\hessgencdf}}$ and applying
Lemma~\ref{LemGeneralMinimaxLower} proves the claim.


\subsection{Some implied properties of the model}
\label{AppAuxiliaryMwise}

In this section, we prove the two auxiliary lemmas stated at the
start of this appendix.


\subsubsection{Proof of Lemma~\ref{lem:auxLaplacian}}

From the definition~\eqref{eq:defn_genlapmx} of $\genlapmx$, have
\begin{align*}
\genlapmx \ones &= \frac{1}{\numobs} \sum_{i=1}^{\numobs} \gendiff_i
(\numchoices \identity - \ones \ones^T) \gendiff_i^T \ones \; = \;
\frac{1}{\numobs} \sum_{i=1}^{\numobs} \gendiff_i (\numchoices
\identity - \ones \ones^T) \ones \; = 0,
\end{align*}
showing that $\ones \in \nullspace(\genlapmx)$.

Now consider any non-zero vector $v \defn [v_1,\ldots,v_\numitems]^T
\in \reals^\numitems$ such that $v \notin \spanvectors{\ones}$. Then
there must exist some $i,j \in [\numitems]$ such that $v_i \neq
v_j$. We know that there exists some path from item $i$ to $j$ in the
comparison hyper-graph. Thus there must exist some hyper-edge in this
path with two items, say $i',j'$, such that $v_{i'} \neq
v_{j'}$. Suppose that hyper-edge corresponds to sample $\ell \in
[\numobs]$. Let $v' \defn \gendiff_\ell^T v$. Then $v' \notin \spanvectors{\ones}$. The Cauchy-Schwarz inequality $\inprod{v'}{v'}
\inprod{1}{1} > (\inprod{v'}{1})^2$ thus implies
\begin{align*}
v^T \gendiff_\ell (\numchoices \identity - \ones \ones^T)
\gendiff_\ell^T v > 0.
\end{align*}
Furthermore, for any $v'' \in \reals^\numchoices$, the Cauchy-Schwarz
inequality $\inprod{v''}{v''} \inprod{1}{1} > (\inprod{v''}{1})^2$
implies that for any $i \in [\numobs]$, we have $v^T \gendiff_i
(\numchoices \identity - \ones \ones^T) \gendiff_i^T v \geq 0$.
Overall we conclude that have $v^T \genlapmx v > 0$ for every $v
\notin \spanvectors{\ones}$, and hence, $\nullspace(\genlapmx) =
\ones$ and $\eigenvalue{2}{\genlapmx} > 0$.

Finally, we have
\begin{align}
\label{EqnMwiseLemma3}
\trace{\genlapmx} & = \frac{1}{\numobs} \sum_{i=1}^{\numobs} \trace{
  \gendiff_i (\numchoices \identity - \ones \ones^T) \gendiff_i^T} \;
= \; \frac{1}{\numobs} \sum_{i=1}^{\numobs} \Big( \numchoices \trace{
  \gendiff_i \gendiff_i^T } - \trace{ \gendiff_i \ones \ones^T
  \gendiff_i^T } \Big).
\end{align}
By the definition of the matrices $\{\gendiff_i\}_{i \in [\numobs]}$,
$\trace{\gendiff_i \gendiff_i^T} = \numchoices$ and $\trace{
  \gendiff_i \ones \ones^T \gendiff_i^T } = \numchoices$. Substituting
these values in~\eqref{EqnMwiseLemma3} gives the desired result
$\trace{\genlapmx} = \numchoices (\numchoices-1)$.\qed


\subsubsection{Proof of Lemma~\ref{lem:auxHessian}}

Let $h_1,\ldots,h_\numchoices$ denote the $\numchoices$ eigenvectors
of $\hessgencdf$, with $h_1 = \frac{1}{\sqrt{\numchoices}}\ones$. Then
for any vector $v' \in \reals^\numchoices$,
\begin{align*}
v'^T \hessgencdf v' = \sum_{i=2}^{\numchoices}
\eigenvalue{i}{\hessgencdf} \inprod{v'}{h_i}^2 \geq
\eigenvalue{2}{\hessgencdf} \sum_{i=2}^{\numchoices}
\inprod{v'}{h_i}^2 &= \eigenvalue{2}{\hessgencdf} \Big(
\sum_{i=1}^{\numchoices} \inprod{v'}{h_i}^2 -\frac{1}{\numchoices}
\inprod{v'}{\ones}^2 \Big) \\ & = \eigenvalue{2}{\hessgencdf} v'^T (
\identity -\frac{1}{\numchoices} \ones \ones^T) v',
\end{align*}
where the final step employed the property $\sum_{i=1}^{\numchoices}
h_i h_i^T = \identity$ of the eigenvectors $h_1,\ldots,h_\numchoices$
of $\hessgencdf$.  A similar argument gives
\begin{align*}
v'^T \hessgencdf v' = \sum_{i=2}^{\numchoices}
\eigenvalue{i}{\hessgencdf} \inprod{v'}{h_i}^2 \leq
\eigenvalue{\max}{\hessgencdf} \sum_{i=2}^{\numchoices}
\inprod{v'}{h_i}^2 & = \eigenvalue{\max}{\hessgencdf} \Big(
\sum_{i=1}^{\numchoices} \inprod{v'}{h_i}^2 -\frac{1}{\numchoices}
\inprod{v'}{\ones}^2 \Big) \\ & = \eigenvalue{\max}{\hessgencdf} v'^T
( \identity -\frac{1}{\numchoices} \ones \ones^T) v'.
\end{align*}
Setting $v' = \revmx_j^T v$ gives
\begin{align*}
\eigenvalue{2}{\hessgencdf} v^T \revmx_j ( \identity
-\frac{1}{\numchoices} \ones \ones^T) \revmx_j^T v \leq v^T \revmx_j
\hessgencdf \revmx_j^T v \leq \eigenvalue{\max}{\hessgencdf} v^T
\revmx_j ( \identity -\frac{1}{\numchoices} \ones \ones^T) \revmx_j^T
v.
\end{align*}
Observe that the matrix $\identity -\frac{1}{\numchoices} \ones
\ones^T$ is invariant to permutation of the coordinates, and hence
$\revmx_j ( \identity -\frac{1}{\numchoices} \ones \ones^T) \revmx_j^T
= \identity -\frac{1}{\numchoices} \ones \ones^T$. This gives
\begin{align*}
\frac{\eigenvalue{2}{\hessgencdf}}{\numchoices} v^T ( \numchoices
\identity - \ones \ones^T) v \leq v^T \revmx_j \hessgencdf \revmx_j^T
v \leq \frac{\eigenvalue{\max}{\hessgencdf}}{\numchoices} v^T (
\numchoices \identity - \ones \ones^T) v.
\end{align*}
\qed


\section{Some useful tail bounds}
\label{AppTailBounds}

In this appendix, we collect a few useful tail bounds for quadratic
forms in Gaussian and sub-Gaussian random variables.

\blem[Tail bound for Gaussian quadratic form]
\label{LemQuadForm}
For any positive semidefinite matrix $Q$ and standard Gaussian vector
$g \sim N(0, I_\numitems)$, we have
\begin{align}
\label{EqnQuadFormBound}
\mprob \big[ g^T Q g \geq \big(\sqrt{\trace{Q}} + \sqrt{\opnorm{Q}}
\;  \delta \big)^2 \big] & \leq e^{-\delta/2}.
\end{align}
valid for all $\delta \geq 0$.
\elem
\begin{proof}
Note that the function $g \mapsto \|\sqrt{Q} g\|_2$ is Lipschitz with
constant $\opnorm{\sqrt{Q}}$.  Consequently, by concentration for
Lipschitz functions of Gaussian vectors~\citep{Ledoux01}, the random
variable $Z = \|\sqrt{Q} g\|_2$ satisfies the upper bound
\begin{align*}
\mprob \big[ Z \geq \Exs[Z] + t \big] \leq \exp \big( - \frac{t^2}{2
  \opnorm{\sqrt{Q}}^2}\big) \; = \; \exp \big(-\frac{t^2}{2
  \opnorm{Q}} \big).
\end{align*}
By Jensen's inequality, we have $\Exs[Z] = \Exs[ \|\sqrt{Q}g\|_2] \;
\leq \sqrt{\Exs[g^T Q g]} \; = \; \sqrt{\trace{Q}}$. Setting \mbox{$t
  = \sqrt{\opnorm{Q}} \, \delta$} completes the proof.
\end{proof}

\blem[\citep{hanson1971bound,rudelson2013hanson}]
\label{LemHansonWright}
Let $V \in \real^\numitems$ be a random vector with independent
zero-mean components that are sub-Gaussian with parameter $K$, and let
$M \in \real^{\numitems \times \numitems}$ be an arbitrary matrix.
Then there is a universal constant $c > 0$ such that
\begin{align}
\label{EqnHansonWright}
\mprob \left[ \big |V^T M V - \Exs[V^T M V] \big| > t \right) \leq 2
  \exp \, \left(-c\min\left\{\frac{t^2}{K^4 \fronorm{M}^2},
  \frac{t}{K^2 \opnorm{M}} \right\} \right) \qquad \mbox{for all $t >
    0$.}
\end{align}
\elem
%

\section{Properties of Laplacian matrices}
\label{AppLaplacian}

By construction, the Laplacian $\Lap$ of the comparison graph is
symmetric and positive-semidefinite. By the singular value
decomposition, we can write $\Lap = U^T \Lambda U$ where $U \in
\real^{\numitems \times \numitems}$ is an orthonormal matrix, and
$\Lambda$ is a diagonal matrix of nonnegative eigenvalues with
$\Lambda_{jj} = \eigenvalue{j}{\Lap}$ for every $j \in
     [\numitems]$. 
Given our assumption of $\eigenvalue{1}{\Lap} \leq
     \cdots \leq \eigenvalue{\numitems}{\Lap}$, we also have
     $\Lambda_{11} \leq \cdots \leq \Lambda_{\numitems
       \numitems}$. Also recall that $\LapInv$ denotes the
     Moore-Penrose pseudo-inverse of $\Lap$. In terms of the notation
     introduced, the Moore-Penrose pseudo-inverse is then given by
     $\LapInv = U^T \LambdaInv U$, where $\LambdaInv$ is a diagonal
     matrix with entries
\begin{align*}
\LambdaInv_{jj} & =
\begin{cases}
(\Lambda_{jj}^{-1}) & \mbox{if $\Lambda_{jj} > 0$} \\ 0 &
  \mbox{otherwise.}
\end{cases}
\end{align*}

The following pair of lemmas establish some useful properties about
$\Lap$.

\begin{lemma}\label{LemLapProperties}
The Laplacian matrix~\eqref{EqnDefnLap} satisfies the trace constraints
\begin{align*}
\trace{\Lap} = 2, \quad \mbox{and} \quad \trace{\LapInv} \geq
\frac{\numitems^2}{4}.
\end{align*} 
\end{lemma}
\begin{proof}
From the definition~\eqref{EqnDefnLap} of the matrix $\Lap$, we have
$\trace{\Lap} = \frac{1}{\numobs} \sum_{i=1}^{\numobs} \trace{ \diff_i
  \diff_i^T} = 2$.  We also know that $\eigenvalue{1}{\Lap} = 0$, and
hence $\sum_{j=2}^{\numitems} \eigenvalue{j}{\Lap} = 2$. Given the
latter constraint, the sum $\sum_{j=2}^{\numitems}
\frac{1}{\eigenvalue{j}{\Lap}}$ is minimized when
$\eigenvalue{2}{\Lap} = \cdots = \eigenvalue{\numitems}{\Lap}$. Some
simple algebra now gives the claimed result.
\end{proof}

\begin{lemma}
\label{LemLapX}
For the matrix $\Lap$ defined in~\eqref{EqnDefnLap}, and for a
$(\numobs \times \numitems)$ matrix $\diffmx$ with $\diff_i^T$ as its
$i^{\rm th}$ row,
\begin{align*}
\trace{\frac{1}{\numobs} \diff^T \LapInv \diff } = \numitems-1, \quad
\fronorm{\frac{1}{\numobs} \diff^T \LapInv \diff} = \numitems - 1,
\quad \mbox{and} \quad \opnorm{\frac{1}{\numobs} \diff^T \LapInv
  \diff} = 1.
\end{align*}
\end{lemma}
\begin{proof}
Let $Q = \frac{1}{\numobs} \diff^T \LapInv \diff$. Since $\Lap =
\frac{1}{\numobs} \Xmat^T \Xmat = U^T \Lambda U$, the diagonal entries
of $\Lambda$ are the squared singular values of
$\Xmat/\sqrt{\numobs}$.  Consequently, there must exist an orthonormal
matrix $V$ such that $\Xmat/\sqrt{\numobs} = V \sqrt{\Lambda} U^T$,
and thus we can write $Q = V \sqrt{\Lambda} \: \LambdaInv \:
\sqrt{\Lambda} \, V^T$.  By definition of the Moore-Penrose
pseudo-inverse, the matrix $\sqrt{\Lambda} \, \LambdaInv \,
\sqrt{\Lambda}$ is a diagonal matrix; since the Laplacian graph is
connected, its diagonal contains $(\numitems-1)$ ones and a single zero.  Noting that $V$ is an orthonormal matrix gives the claimed
result.
\end{proof}

For future reference, we state and prove a lemma showing that these
two semi-norms satisfy a restricted form of the Cauchy-Schwarz
inequality:
\blem
\label{LemPseudoCauchySchwarz}
For any two vectors $u$ and $v$ such that $u \perp \nullspace(\Lap)$
or/and $v \perp \nullspace(\Lap)$, we have
\begin{align}
\label{EqnToast}
|\inprod{u}{v}| & \leq \LAPNORMINV{u} \; \LAPNORM{v}.
\end{align}
\elem
\begin{proof}
Since $\Lap = U^T \Lambda U$ and $\LapInv = U^T \LambdaInv U$, we have
\begin{align*}
\sqrt{v^T \Lap v} \sqrt{u^T \LapInv u} & = \sqrt{v^T U^T \Lambda U v}
\; \sqrt{u^T U^T \LambdaInv U u} = \norm{\vtil} \norm{\util} \geq
|\inprod{\vtil}{\util}|,
\end{align*}
where we have defined $\vtil \defn \sqrt{\Lambda} U v$ and $\util
\defn \sqrt{\LambdaInv} U u$.  Continuing on,
\begin{align*}
\inprod{\vtil}{\util} & = v^T U^T \sqrt{\Lambda} \sqrt{\LamTil} U u \;
= \; v^T U U^T u,
\end{align*}
where we have used the fact that $u$ or/and $v$ are orthogonal to the
null space of $\Lap$. Since $U$ is orthonormal, we conclude that
$\inprod{\vtil}{\util} = \inprod{v}{u}$, which completes the proof.
\end{proof}


\section{Minimax risk without assumptions on quality scores}\label{AppUnboundedWeights}
The setting considered throughout the paper imposes two restrictions~\eqref{EqnDefnSpecset} on the quality score vector $\wtstar$. The first condition is that of shift invariance, that is, $\inprod{\wtstar}{\ones}=0$. The necessity of this condition for identifiability under the~\ord model is easy to verify. The second condition is that the quality score vectors are $\wmax$-bounded, that is, $\Lnorm{\wtstar}{\infty} \leq \wmax$ for some finite $\wmax$. In this section, for the sake of completeness, we show that the minimax risk is infinite in the absence of this condition.
\begin{proposition}\label{prop:unbounded_weights}
Any estimator $\wttil$ based on $\numobs$ samples from the~\ord model (with unbounded quality score vectors) has error lower bounded as
\begin{align*}
\sup_{\wtstar \in \SPECSETinf} \Exs \Big[ \Lnorm{\wttil - \wtstar}{2}^2
  \Big] = \sup_{\wtstar \in \SPECSETinf} \Exs \Big[ \Lnorm{\wttil - \wtstar}{\Lap}^2
  \Big]  = \infty.
\end{align*}
\end{proposition}

The remainder of this section is devoted to the formal proof of Proposition~\ref{prop:unbounded_weights}. Consider the event where for every comparison, the item with the higher quality score in $\wtstar$ wins. For any $\wtstar \in \SPECSETinf \backslash \{\zeros\}$, this event occurs with a probability at least $\frac{1}{2^\numobs}$. Under this event, the true $\wtstar$ is indistinguishable from the quality score vector $\plaincon \wtstar \in \SPECSETinf$ for every $\plaincon \geq 0$, and the error is also unbounded. Since the probability of this event is strictly bounded away from zero, the expected error is also unbounded.


\bibliographystyle{siva}
\bibliography{bibtex}

\end{document}